\documentclass[11pt]{article}

\usepackage{etoolbox}

\usepackage{times}

\usepackage{etoolbox} 
\newtoggle{coltformat} 
\togglefalse{coltformat} 
\newcommand{\colt}[1]{\iftoggle{coltformat}{#1}{}} 
\newcommand{\arxiv}[1]{\iftoggle{coltformat}{}{#1}}  

\arxiv{

 \usepackage[letterpaper, left=1in, right=1in, top=1in, bottom=1in]{geometry}

\usepackage[colorlinks=true, linkcolor=blue!70!black, citecolor=blue!70!black]{hyperref}
\usepackage{microtype}

\usepackage{natbib}
\bibliographystyle{plainnat}
\bibpunct{(}{)}{;}{a}{,}{,}

\usepackage{amsthm}
\usepackage{mathtools}
\usepackage{amsmath}
\usepackage{bbm}
\usepackage{amsfonts}
\usepackage{amssymb}
\let\vec\undefined
\usepackage{xpatch}
\usepackage{pifont}
\usepackage{array}
\usepackage{booktabs}
\usepackage{floatrow}
\newfloatcommand{capbtabbox}{table}[][\FBwidth]
\usepackage{blindtext}
\usepackage{caption} 
\usepackage{amsmath} 
\usepackage{amsthm} 
}

\usepackage[suppress]{color-edits}   
\addauthor{as}{purple}
\addauthor{zj}{red} 
\addauthor{cw}{brown} 
\addauthor{sr}{orange} 

\usepackage{amsmath,amsfonts,bm}

\newcommand{\Phistar}{\Phi^\star}
\newcommand{\aggCstar}{\aggC^\star}

\def\eqref#1{equation~\ref{#1}}

\def\1{\bm{1}}

\DeclareMathAlphabet{\mathsfit}{\encodingdefault}{\sfdefault}{m}{sl}
\SetMathAlphabet{\mathsfit}{bold}{\encodingdefault}{\sfdefault}{bx}{n}

\def\gA{{\mathcal{A}}}

\def\gD{{\mathcal{D}}}

\def\gF{{\mathcal{F}}}
\def\gG{{\mathcal{G}}}

\def\gI{{\mathcal{I}}}
\def\gJ{{\mathcal{J}}}

\def\gT{{\mathcal{T}}}

\def\gX{{\mathcal{X}}}

\def\gZ{{\mathcal{Z}}}

\newcommand{\E}{\mathbb{E}}

\DeclareMathOperator*{\argmax}{arg\,max}
\DeclareMathOperator*{\argmin}{arg\,min}

\usepackage{thmtools}
\usepackage{thm-restate}
\arxiv{\declaretheorem[name=Theorem,parent=section]{theorem}}
\arxiv{\declaretheorem[name=Definition,parent=section]{definition}}
\arxiv{\declaretheorem[name=Assumption,parent=section]{assumption}}

\usepackage{mathtools}
\arxiv{\usepackage{amsthm}}
\usepackage{amsmath}
\usepackage{amsxtra}
\usepackage{enumitem}      
\usepackage{mathtools}
\usepackage{bbm}
\usepackage{amsfonts}
\usepackage{amssymb}
\let\vec\undefined
\usepackage{MnSymbol} 

\usepackage{xpatch}

\arxiv{\theoremstyle{definition}  
\theoremstyle{plain}
}
\colt{\newtheorem{assumption}{Assumption}}

\arxiv{} 
\arxiv{\newtheorem{proposition}{Proposition}}

\newtheorem{property}{Property}

\arxiv{\newtheorem{corollary}{Corollary}}
\arxiv{\newtheorem{lemma}{Lemma}}
\arxiv{\newtheorem{remark}{Remark}}
\arxiv{\newtheorem{example}{Example}}

\newtheorem*{theorem*}{Theorem}

\xpatchcmd{\proof}{\itshape}{\normalfont\proofnameformat}{}{}
\newcommand{\proofnameformat}{\bfseries}

\usepackage{prettyref}
\newcommand{\pref}[1]{\prettyref{#1}}

\newcommand{\savehyperref}[2]{\texorpdfstring{\hyperref[#1]{#2}}{#2}}
\newrefformat{eq}{\savehyperref{#1}{\textup{(\ref*{#1})}}}
\newrefformat{eqn}{\savehyperref{#1}{Equation~\ref*{#1}}}
\newrefformat{con}{\savehyperref{#1}{Conjecture~\ref*{#1}}}
\newrefformat{lem}{\savehyperref{#1}{Lemma~\ref*{#1}}}
\newrefformat{def}{\savehyperref{#1}{Definition~\ref*{#1}}}
\newrefformat{line}{\savehyperref{#1}{line~\ref*{#1}}}
\newrefformat{thm}{\savehyperref{#1}{Theorem~\ref*{#1}}}
\newrefformat{corr}{\savehyperref{#1}{Corollary~\ref*{#1}}}
\newrefformat{sec}{\savehyperref{#1}{Section~\ref*{#1}}}
\newrefformat{app}{\savehyperref{#1}{Appendix~\ref*{#1}}}
\newrefformat{ass}{\savehyperref{#1}{Assumption~\ref*{#1}}}
\newrefformat{eg}{\savehyperref{#1}{Example~\ref*{#1}}} 
\newrefformat{fig}{\savehyperref{#1}{Figure~\ref*{#1}}}
\newrefformat{alg}{\savehyperref{#1}{Algorithm~\ref*{#1}}}
\newrefformat{rem}{\savehyperref{#1}{Remark~\ref*{#1}}}
\newrefformat{conj}{\savehyperref{#1}{Conjecture~\ref*{#1}}}
\newrefformat{prop}{\savehyperref{#1}{Proposition~\ref*{#1}}}
\newrefformat{proto}{\savehyperref{#1}{Protocol~\ref*{#1}}}
\newrefformat{prob}{\savehyperref{#1}{Problem~\ref*{#1}}}
\newrefformat{claim}{\savehyperref{#1}{Claim~\ref*{#1}}}
\newrefformat{propr}{\savehyperref{#1}{Property~\ref*{#1}}}

\newcommand\numberthis{\addtocounter{equation}{1}\tag{\theequation}}
\allowdisplaybreaks

\DeclarePairedDelimiter{\abs}{\lvert}{\rvert} 
\DeclarePairedDelimiter{\brk}{[}{]}
\DeclarePairedDelimiter{\crl}{\{}{\}}
\DeclarePairedDelimiter{\prn}{(}{)}
\DeclarePairedDelimiter{\nrm}{\|}{\|}

\let\Pr\undefined

\DeclareMathOperator{\En}{\mathbb{E}}

\DeclareMathOperator{\Pr}{Pr}

\newcommand{\ldef}{\vcentcolon=}

\newcommand{\wt}[1]{\widetilde{#1}}
\newcommand{\wh}[1]{\widehat{#1}}
\newcommand{\mb}[1]{\boldsymbol{#1}}

\def\ddefloop#1{\ifx\ddefloop#1\else\ddef{#1}\expandafter\ddefloop\fi}
\def\ddef#1{\expandafter\def\csname bb#1\endcsname{\ensuremath{\mathbb{#1}}}}
\ddefloop ABCDEFGHIJKLMNOPQRSTUVWXYZ\ddefloop
\def\ddefloop#1{\ifx\ddefloop#1\else\ddef{#1}\expandafter\ddefloop\fi}
\def\ddef#1{\expandafter\def\csname b#1\endcsname{\ensuremath{\mathbf{#1}}}}
\ddefloop ABCDEFGHIJKLMNOPQRSTUVWXYZ\ddefloop
\def\ddef#1{\expandafter\def\csname c#1\endcsname{\ensuremath{\mathcal{#1}}}}
\ddefloop ABCDEFGHIJKLMNOPQRSTUVWXYZ\ddefloop
\def\ddef#1{\expandafter\def\csname h#1\endcsname{\ensuremath{\widehat{#1}}}}
\ddefloop ABCDEFGHIJKLMNOPQRSTUVWXYZabcdefghijklmnopqrsuvwxyz\ddefloop    
\def\ddef#1{\expandafter\def\csname hc#1\endcsname{\ensuremath{\widehat{\mathcal{#1}}}}}
\ddefloop ABCDEFGHIJKLMNOPQRSTUVWXYZ\ddefloop
\def\ddef#1{\expandafter\def\csname t#1\endcsname{\ensuremath{\widetilde{#1}}}}
\ddefloop ABCDEFGHIJKLMNOPQRSTUVWXYZ\ddefloop
\def\ddef#1{\expandafter\def\csname tc#1\endcsname{\ensuremath{\widetilde{\mathcal{#1}}}}}
\ddefloop ABCDEFGHIJKLMNOPQRSTUVWXYZ\ddefloop

\newcommand{\sr}[1][n]{\sigma_{1:#1}}

\usepackage{tikz}

\newcommand{\proman}[1]{\prn*{\romannumeral #1}}

\newcommand{\overleq}[1]{\overset{ #1}{\leq{}}}
\newcommand{\overgeq}[1]{\overset{#1}{\geq{}}}
\newcommand{\overeq}[1]{\overset{#1}{=}}

\newcommand{\algcomment}[1]{\textcolor{blue!70!black}{\footnotesize{\texttt{\textbf{//
          #1}}}}}
\newcommand{\algcommentbig}[1]{\textcolor{blue!70!black}{\footnotesize{\texttt{\textbf{/*
          #1~*/}}}}}

\newcommand{\CWremove}[1]{}

 \newcommand{\B}{\mathfrak{B}}

\newcommand{\pieval}{{\pi_{\mathrm{e}}}} 
\newcommand{\ripieval}{{\rich{\pi}_{\mathrm{e}}}}
\newcommand{\ripiexp}{{\rich{\pi}_{\mathrm{b}}}}
\newcommand{\aggpieval}{{\agg{\pi}_{\mathrm{e}}}} 
\newcommand{\pioff}{{\pi_{b}}} 
\newcommand{\piexp}{{\pi_{b}}}  
\newcommand{\piother}{{\pi_{c}}}  
\newcommand{\wtpieval}{{{\wt\pi}_{\mathrm{e}}}} 

\newcommand{\wtpiexp}{{{\wt \pi}_{b}}}

\newcommand{\Yhat}{\widehat{Y}} 
\newcommand{\Zhat}{\widehat{Z}} 
\newcommand{\That}{\widehat{\gT}}

\newcommand{\indic}{\mathbb{I}}
 
\newcommand{\calX}{\mathcal{X}}

\newcommand{\calA}{\mathcal{A}}
\newcommand{\calZ}{\mathcal{Z}}

\newcommand{\fout}{\widehat{f}}
\newcommand{\fstar}{f^{\star}}

\newcommand{\agg}[1]{\bar{#1}} 
\newcommand{\ophatC}[4]{\agg{\mathsf{C}}_{#4}(#1, #2, #3)}

\newcommand{\aggC}{\agg{\mathsf{C}}}

\newcommand{\rich}[1]{\check{#1}}
\newcommand{\Cpush}{\mathsf{C}_{\textup{pf}}}
\newcommand{\CpushX}{\mathsf{C}_{\mathcal{X}}}
\newcommand{\CpushA}{\mathsf{C}_{\mathcal{A}}}
\newcommand{\Cupper}{\mathsf{C}_{\textup{upper}}}
\newcommand{\Clower}{\mathsf{C}_{\textup{lower}}}
\newcommand{\order}{O} 
\newcommand{\tabC}{\mathsf{C}}

\newcommand{\TV}{D_{\mathrm{TV}}}

\newcommand{\su}{{u}}
\newcommand{\sv}{{v}}
\newcommand{\sw}{{w}} 
\newcommand{\sz}{{z}}
\newcommand{\sx}{{x}}
\newcommand{\sy}{{y}}
\newcommand{\sa}{{a}}
\newcommand{\ssp}{{p}}
\newcommand{\sq}{{q}}
\renewcommand{\sr}{{r}}

\newcommand{\estat}{\epsilon_{\textup{stat}}}
\newcommand{\eappr}{\epsilon_{\textup{appr}}}
\newcommand{\Phimax}{\Phi_{\max}}
\newcommand{\poly}{\textup{poly}}

\newcommand{\Phiopt}{\gI}  
\newcommand{\up}[1]{^{(#1)}}

\newcommand{\shp}{\bar{\sx}} 
\newcommand{\Block}{Block}
\newcommand{\converter}{{\textsc{admissible-to-trajectory}}\xspace} 

\newcommand{\At}{\textsc{Alg}_{\textsc{traj}}}   
 
\newcommand{\Dt}{\mathcal{D}^{\textsc{traj}}}
\newcommand{\Da}{\mathcal{D}^{\textsc{adm}}} 

\renewcommand{\epsilon}{\varepsilon}

\newcommand{\ind}[1]{^{\scriptscriptstyle(#1)}} 
\newcommand{\inds}[1]{^{\scriptscriptstyle[#1]}}

\newcommand{\Dti}[1]{\brk{\Delta T\ind{#1}}} 

\newcommand{\MTM}{\textsc{MTM}\xspace}  
\newcommand{\MDP}{\textsc{MDP}\xspace} 
  
\newcommand{\BVFT}{\textsc{BVFT}\xspace} 
\newcommand{\Replicator}{\textsc{Replicator}\xspace}  

\newcommand{\opeg}{\mathfrak{g}}

\newcommand{\famGg}{\mathfrak{G}}  
\newcommand{\famGa}{\mathfrak{G}_{\textsc{adm}}} 
\newcommand{\famGt}{\mathfrak{G}_{\textsc{traj}}} 
\newcommand{\OPE}{\textsc{OPE}\xspace}

\newcommand{\aone}{\mathfrak{a}_1}
\newcommand{\atwo}{\mathfrak{a}_2} 
\newcommand{\Term}{\mathfrak{R}}  

\usepackage{hyperref}
\usepackage{url}

\usepackage{amsfonts}\usepackage{amssymb} 
\usepackage{xcolor} 
\usepackage[utf8]{inputenc} 
\usepackage[T1]{fontenc}    
\usepackage{hyperref}       
\usepackage{url} 
\usepackage{booktabs}       
\usepackage{amsfonts}       
\usepackage{nicefrac}       
\usepackage{microtype}      
\usepackage{xcolor}         

\usepackage{amsfonts} 
\usepackage{algorithm} 
\usepackage[noend]{algpseudocode}

\usepackage{amssymb, bbold} 
\usepackage{nicefrac} 
\usepackage{bm}

\hypersetup{colorlinks=true,allcolors=blue}

\arxiv{\title{Offline Reinforcement Learning: \\ Role of State Aggregation and Trajectory Data}
\usepackage{times} 

\author{
 Zeyu Jia$^{1}$\footnote{Authors are listed in alphabetical order by last name.}\\[-0pt] 
 \scalebox{0.79}{\texttt{zyjia@mit.edu}}
 \and 
 Alexander Rakhlin$^{1}$ \\[-0pt]
 \scalebox{0.79}{\texttt{rakhlin@mit.edu}}
 \and 
 Ayush Sekhari$^{1}$ \\[-0pt]
 \scalebox{0.79}{\texttt{sekhari@mit.edu}}
 \and  
 Chen-Yu Wei$^{2}$ \\[-0pt]
 \scalebox{0.79}{\texttt{chenyu.wei@virginia.edu}}
  \and 
  \\
  $^{1}$Massachusetts Institute of Technology, \hspace{2mm} $^{2}$University of Virginia   
} 
} 

\colt{
\title[Offline RL: Role of State Aggregation and Trajectory Data]{Offline Reinforcement Learning: \\ Role of State Aggregation and Trajectory Data}  
\coltauthor{%
 \Name{Author Name1} \Email{abc@sample.com}\\ 
 \addr Address 1
 \AND
 \Name{Author Name2} \Email{xyz@sample.com}\\
 \addr Address 2%
}
} 

\usepackage{longtable}
\usepackage{import} 
\usepackage{enumitem} 
\usepackage{svg} 

\usepackage{makecell}
\usepackage{xspace} 
\usepackage[scaled=.9]{helvet}
\usepackage{booktabs} 
\usepackage[export]{adjustbox}
\usepackage{comment}
\usepackage{bbm} 
\usepackage{bbold}
\usepackage{wrapfig}

\usepackage{upgreek} 

\usepackage{graphicx}

\arxiv{\usepackage{subfigure}} 

\usepackage{enumitem}

\date{}
\usepackage{yfonts} 

\colt{
	\newenvironment{proofof}[1][\unskip]{%
		\par\medskip\noindent{\bfseries\upshape Proof of #1. \/ }%
	}{\jmlrQED}
} 
\arxiv{
	\newenvironment{proofof}[1][\unskip]{%
		\par\medskip\noindent{\bfseries\upshape Proof of #1. \/ }%
	}{\hfill $\blacksquare$}   
} 
\begin{document} 

\maketitle  

\begin{abstract} 
   We revisit the problem of offline reinforcement learning with value function realizability but without Bellman completeness. Previous work by \cite{xie2021batch} and \cite{foster2022offline} left open the question whether a bounded concentrability coefficient along with trajectory-based offline data admits a polynomial sample complexity. In this work, we provide a negative answer to this question for the task of offline policy evaluation. In addition to addressing this question, we provide a rather complete picture for offline policy evaluation with only value function realizability. Our primary findings are threefold: 1) The sample complexity of offline policy evaluation is governed by the concentrability coefficient in an aggregated Markov Transition Model jointly determined by the function class and the offline data distribution, rather than that in the original MDP. This unifies and generalizes the ideas of \cite{xie2021batch} and \cite{foster2022offline},  2) The concentrability coefficient in the aggregated Markov Transition Model may grow exponentially with the horizon length, even when the concentrability coefficient in the original \MDP is small and the offline data is \emph{admissible} (i.e., the data distribution equals the occupancy measure of some policy), 3) Under value function realizability, there is a generic reduction that can convert any hard instance with admissible data to a hard instance with trajectory data, implying that trajectory data offers no extra benefits over admissible data. These three pieces jointly resolve the open problem, though each of them could be of independent interest. 
\end{abstract}  
\section{Introduction}

In offline Reinforcement Learning (RL), the learner aims to either find the optimal policy (policy optimization) or evaluate a given policy (policy evaluation) based on pre-collected (offline) data, i.e.~without any live interaction with the environment. This approach is relevant in situations where real-time engagement is either impractical or costly, such as in safety-critical applications like healthcare and autonomous driving. 
The main hurdle in offline RL is the discrepancy between the state-action distribution induced by the given offline data and the target policy, which can significantly compromise our ability to evaluate policies in the environment. Consequently, much of the offline RL research focuses on addressing this data-policy mismatch.  

To deal with large state spaces in RL, function approximation is used to generalize the acquired knowledge across states. A key approach in this direction is \textit{model-based offline RL}, which involves constructing explicit models for transition and reward functions. 
Theoretical research studies the sample complexity under \emph{model realizability}, where we assume access to a class that contains the underlying model \citep{uehara2021pessimistic}. 
Thus, the sample complexity of model-based offline RL depends on the concentrability coefficient and the model class complexity, where the former quantifies the data-policy distribution mismatch (formally defined in \pref{sec: concen coeff}) and the latter quantifies the diversity within the model class. Unfortunately, in many tasks, model-based offline RL is not applicable due to the prohibitively large model complexity needed to represent the ground truth model.

Alternatively, \textit{value-based offline RL} simplifies the process by approximating only the target policy's value function, typically resulting in lower complexity. Despite being a simpler and more direct approach, achieving polynomial sample complexity under this approach has historically required additional assumptions beyond a bounded concentrability coefficient and value function realizability, such as Bellman completeness \citep{chen2019information}, $\beta$-incompleteness with $\beta<1$ \citep{zanette2023realizability}, density-ratio realizability \citep{xie2020q}, pushforward concentrability \citep{xie2021batch}, etc. The feasibility of attaining polynomial sample complexity for offline RL with solely bounded concentrability coefficient and realizability remained open for a long time.

The recent work of \cite{foster2022offline} provided a negative answer to this question, showing that realizability and concentrability alone cannot ensure polynomial sample complexity in offline RL. Their lower bound construction, however, relies on the offline data being generated somewhat unnaturally---e.g. the offline data includes samples from trajectories that can never be visited by any policy or only includes samples from just the first two steps in an episode, with the later steps hidden from the learner. Thus, the following question remained open: 

\begin{center}
\begin{minipage}{0.95\textwidth}  
\centering  
\textit{Is offline RL statistically efficient under value function realizability, bounded concentrability, and natural offline data distribution such as trajectories generated by a single behavior policy?} 
\end{minipage}  
\end{center} 

Our work answers this question in the negative for \emph{offline policy evaluation} (\OPE). Specifically, we show that: There exist \MDP instances where even though the offline data consists of \emph{trajectories} generated by a single behavior policy (used to collect offline data), the value function is realized by a small value function class, action space is finite (binary), and the concentrability coefficient is polynomial in the horizon length, the sample complexity for offline policy evaluation scales exponentially with the horizon. 

En route to establishing the above result, we provide a comprehensive characterization for offline policy evaluation with value function approximation and show that to achieve sample efficiency, the concept of ``concentrability'' needs to be strengthened. In particular, we show that the worst-case sample complexity for offline policy evaluation is both upper and lower bounded via the ``aggregated concentrability coefficient'' that denotes the data-policy distribution mismatch in an aggregated transition model obtained by clubbing together transitions from the states that have indistinguishable value functions under the given value function class (formally defined in \pref{sec: general lower bound}). On the lower bound side, our results generalize the lower bounds of \cite{foster2022offline}  since we can show that the aggregated concentrability coefficient is exponentially large in their construction. On the upper bound side, our work recovers the results of  \cite{xie2021batch} since we can show that aggregated concentrability coefficient is upper bounded by the pushforward concentrability coefficient, a notion of data-policy distribution mismatch used in their upper bound. 



Having identified that the aggregated concentrability coefficient is the key quantity that governs the sample complexity of offline policy evaluation, we provide a simple example where the offline data distribution is admissible (i.e., the offline data distribution equals the occupancy measure of an offline policy), action space is binary, the standard concentrability coefficient is $\order(H^3)$, but the aggregated concentrability is $2^{\Omega(H)}$, where $H$ is the horizon length. Since  aggregated concentrability governs the sample complexity of offline policy evaluation, 
this implies that realizability, concentrability, and admissibility are not sufficient for sample-efficient OPE. 

Finally, to obtain the lower bound for trajectory data, we provide a generic reduction that converts the aforementioned hard instance for admissible data into a hard instance for trajectory data, thus showing that in the worst case, \OPE is no easier for trajectory data than for admissible data. 




\section{Preliminaries}
\subsection{Markov Decision Process} 
We consider the finite horizon episodic setting. A \textbf{Markov Transition Model} (\MTM), denoted by $M = \MTM(\cX, \cA, T, H, \rho)$, is parameterized by a state space $\gX$, an action space $\gA = \crl{\aone, \atwo, \dots}$, a transition kernel $T: \cX \times \cA \mapsto \Delta(\cX)$, horizon length $H \in \bbN$, and initial distribution $\rho \in \Delta(\cX)$. We assume that the state space \(\cX\) is layered across time, i.e., $\gX=\gX_1\cup \gX_2 \cup \dots \cup\gX_H$ with $\gX_i\cap \gX_j=\emptyset$ for any $i\neq j$. The initial distribution $\rho\in\Delta(\gX_1)$ specifies the state distribution that every episode starts with. The transition kernel $T(\sx' \mid{} \sx, \sa)$ for $\sx\in\gX\setminus \gX_H$, $\sx'\in\gX$ and $\sa\in\gA$ specifies the probability of transitioning to state $\sx'$ if the learner takes action $\sa$ on state $\sx$. Due to the layering structure, $T(\cdot \mid{} \sx, \sa)$ is supported on $\gX_h$ if $\sx\in\gX_{h-1}$. 

For a \(\MTM\) $M$ and policy $\pi: \gX\rightarrow \triangle(\gA)$, we let $\mathbb{E}^{M,\pi}[\cdot]$  denote the expectation under the  process: $\sx_1\sim \rho$, and  for $h=1,\ldots, H$, action $\sa_h\sim \pi(\cdot \mid{} \sx_h)$, and next state $\sx_{h+1}\sim T(\cdot \mid{} \sx_h, \sa_h)$. Additionally, the state occupancy measure for a particular layer $h$ is defined as $d_h^{\pi}(\sx; M) \ldef{}  \mathbb{E}^{M,\pi}[ \indic\{\sx_h=\sx\}]$ 
, and the state-action occupancy measure is defined as $d_h^\pi(\sx, \sa; M) \ldef{} d_h^\pi(\sx; M)\pi(\sa\mid \sx)$. 

 A \textbf{Markov Decision Process} (\MDP), denoted by $M=\MDP(\gX, \gA, T, r, H, \rho)$, is a Markov Transition Model augmented with a reward function $r: \calX\times \calA\to \Delta([-1, 1])$. For \(h \in [H]\) and policy \(\pi\), the state value function \(V_h^\pi(\cdot; M): \cX_h \mapsto \bbR\) is defined such that  for any \(x_h \in \cX_h\), $V_h^\pi(x_h; M)$ is the total cumulative reward obtained by starting at state \(x_h\) at timestep \(h\), and acting according to the policy \(\pi\) till the end of the episode, i.e. $V_h^\pi(\sx; M) = \mathbb{E}^{M,\pi}\big[\sum_{k=h}^H r(\sx_{k}, a_{k})\mid \sx_h=\sx\big]$. We similarly define the state-action value function \(Q_h^\pi: \cX_h \times \cA \mapsto \bbR\) such that for any \(x_h \in \cX_h\), $Q_h^\pi(\sx, \sa)=\E^{M,\pi}[\sum_{h'=h}^H  r(\sx_{h'}, \sa_{h'})\mid(\sx_h, \sa_h)=(\sx,\sa)]$. We let \(\Pi\) denote the set of all non-stationary Markovian policies in the underlying MDP. 
 

For the ease of notation, for any policy \(\pi\), we define $T(\sx'\mid\sx, \pi)=\E_{\sa\sim \pi(\cdot\mid\sx)}[T(\sx'\mid\sx, \sa)]$ and $r(\sx, \pi)=\E_{\sa\sim \pi(\cdot\mid\sx)}[r(\sx, \sa)]$, and if \(\pi\) is deterministic (i.e., $ \pi(\cdot\mid\sx)$ is always supported on a single action for every \(x \in \cX\)), we use $\pi(\sx)\in\calA$ to denote the action it chooses on the state $\sx$. Furthermore, whenever clear from the context, we overload the notation and use $r(\sx, \sa)$ to denote the expected value of reward distribution $r(\sx, \sa)$. Furthermore, whenever apparent, we omit the dependency on $M$ and simply write $d_h^\pi(\sx), V_h^\pi(\sx)$ and $V^\pi(\rho)$, etc. to denote the corresponding occupancy measures and value functions in the \MDP \(M\).


\subsection{Offline RL Preliminaries} \label{sec:offline_RL_prelims} 
Throughout the paper, we consider the offline RL setting. In this setting, the learner is equipped with an offline data distribution\footnote{Various works in the offline RL literature assume that instead of direct sampling access to the offline distribution \(\mu\), the learner is given an offline dataset \(D\) of \(n\) samples drawn from \(\mu\). These two settings are equivalent upto sampling.} \(\mu\) and can only gather data about the MDP by i.i.d.~ sampling from this offline distribution (however, the learner does not know the density function of \(\mu\)). We consider three types of offline data models: 
\begin{enumerate}[label=\(\bullet\)] 
    \item \textbf{General Data:} The offline dataset is characterized by an offline distribution \(\mu = \prn{\mu_1,  \dots, \mu_{H-1}}\) where \(\mu_h \in \Delta(\cX_h \times \cA)\) for \(h \leq H - 1\). An offline sample comprises of \(H-1\) many tuples $(\sx_h, \sa_h, \sr_h, \sx_{h+1}')$, where for each \(h \in [H-1]\), $(\sx_h, \sa_h)\in\calX_h\times\calA$ is drawn from $\mu_h$, $\sr_h\in [-1, 1]$ is drawn from $r(\sx_h, \sa_h)$, and $\sx_{h+1}'\in\calX_{h+1}$ is drawn from $T(\cdot \mid{} \sx_h, \sa_h)$. 
    
    \item  \textbf{Admissible  Data:} Similar to the General Data setting, but with the addition requirement that there exists an offline policy \(\pioff \in \Pi\) such that 
    for all  \(\sx \in \cX\) and \(\sa \in \cA\), \(\mu_h(\sx, \sa) = d_h^\pioff(\sx, \sa)\). 
    
    
    \item \textbf{Trajectory  Data:} Each offline sample is a complete trajectory $(\sx_1, \sa_1, \sr_1, \sx_2, \sa_2, \sr_2,\ldots,$ $\sx_H)$ sampled in the underlying \MDP using an offline policy \(\pioff\). 
\end{enumerate} 

The goal of the learner is to estimate the value of a given evaluation policy \(\pieval\) by collecting samples from the offline data distribution. In particular, given access to the offline distribution \(\mu\), the learner would like to estimate  
$V^{\pieval}(\rho)$ up to an accuracy of~$\epsilon$, i.e. return a \(\wh V\)  such that 
\begin{align} 
\abs*{\wh V - V^{\pieval}(\rho)} \leq \epsilon, \label{eq:goal}  
\end{align}
where \(V^{\pieval}(\rho) = \En_{x_1 \sim \rho} \brk*{V^{\pieval}(x_1)}\). We are interested in quantifying the amount of data required to achieve \pref{eq:goal} with high probability.  It is not hard to see that learning with general or admissible data is more challenging than learning with trajectory data, because from a trajectory dataset one can generate a $(\sx, \sa,\sr,\sx')$ dataset with $\mu_h(\sx, \sa)=d_h^{\pioff}(\sx, \sa)$, but not vice versa. 

\paragraph{Comparison to the Definition of Admissible Data in \cite{foster2022offline}.}  A result in \cite{foster2022offline} also claims to provide an exponential lower bound for admissible data but in the discounted MDP setting. Their data distribution $\mu(\sx,\sa) = \frac{1}{2}(d_1^\pioff(\sx,\sa) + d_2^\pioff(\sx,\sa))$ is not considered as admissible in our definition above\footnote{Since \cite{foster2022offline} consider the discounted setting, data from the first two steps is enough to provide bounded concentrability coefficient. This is not true in our finite-horizon case. }. We note that their lower bound construction heavily relies on samples being drawn only from the first two steps, and no information from the third step or later should be revealed. On the other hand, our admissible data definition forces the information to be revealed on all steps. We also remark that our lower bound for our notion of admissible data serves as an important step towards the lower bound for trajectory data (see \pref{sec: trajectory lower bound}). 



\subsubsection{Concentrability Coefficient} \label{sec: concen coeff} 
Throughout the paper, for simplicity, we consider a fixed and deterministic evaluation policy $\pieval$ that takes the action \(\aone\) on all the states, i.e. \(\pieval(x) = \aone\) for all \(x \in \cX\). The \emph{concentrability coefficient} (\cite{munos2003error, chen2019information}) of policy $\pieval$ in an \MDP $M=\MDP(\gX, \gA, T, r, H, \rho)$ with respect to offline data distribution $\mu$ is defined as 
\begin{align}
    \tabC(M, \mu, \pieval) = \max_h \max_{\sx\in\calX_h, \sa\in\calA} \frac{d_h^\pieval(\sx, \sa)}{\mu_h(\sx, \sa)}.   \label{eq: concentrability coeff} 
\end{align}
Whenever clear  from the context, we skip \(\pieval\) and use the notation $\tabC(M, \mu)$ to denote $\tabC(M, \mu, \pieval)$. 


\subsubsection{Function Approximation}  
The learner is given a function set $\gF$ that consists of functions of the form $\gX\times \gA\rightarrow [-1, 1]$. We make the following realizability assumption that the Q-function belongs to the function class $\cF$.
 
\begin{assumption}[Realizability] 
\label{ass:Q_realizability} 
We have \(Q^{\pieval} \in \cF\).  
\end{assumption}

Given an offline data distribution $\mu$ and a value function class that satisfies realizability, a key question in offline policy evaluation is whether we can design an algorithm that achieves \pref{eq:goal}, with probability at least $1-\delta$, by only using $\poly(\log |\gF|, H,  \nicefrac{1}{\delta}, \nicefrac{1}{\epsilon}, \tabC(M,\mu))$ offline samples. In the following sections, we argue that this is unfortunately impossible, unless we replace $\tabC(M,\mu)$ with a different notion of concentrability (see \pref{sec: general lower bound} for details). 

\subsubsection{Offline Policy Evaluation Problem} 
An Offline Policy Evaluation (\OPE) problem \(\opeg\) is given by a tuple  $(M, \pieval, \mu, \cF)$ where \(M\) denotes the underlying \MDP, \(\mu\) denotes the offline data distribution, \(\pieval\) denotes the evaluation policy, and \(\cF\) denotes a state-action value function class. Given an \(\OPE\) instance \(\opeg\) and a parameter \(\epsilon > 0\), the goal of the learner is to estimate the value of the policy \(\pieval\) in the MDP \(M\) upto precision \(\epsilon\) in expectation, by only relying on samples drawn from the offline distribution \(\mu\). 

We say that the \OPE problem \(\opeg\) is \textit{realizable} if \(Q^\pieval(\cdot; M) \in \cF\). Furthermore, whenever \(\mu\) is admissible and there exists a policy \(\pioff\) such that \(\mu_h = d_h^\pioff\) for all \(h \leq H\), we often denote the \OPE problem as \(\opeg = (M, \pieval, \pioff, \cF)\). Finally, in the case of trajectory data, we still use the notation \(\opeg = (M, \pieval, \pioff, \cF)\) to denote the \(\OPE\) problem  but explicitly clarify, whenever invoked, that the learner now has access to complete trajectories sampled using \(\pioff\).

\section{State Aggregation in Offline RL} \label{sec: general lower bound}




We start by considering the offline policy evaluation problem with general offline data, and introduce useful tools and notation for our main lower bounds for admissible and trajectory data, and our upper bound, in the following sections. 

There is a rich literature on understanding the right structural assumptions for offline RL with general offline data. For a warm-up, when the underlying \MDP is tabular, i.e. has a small number of states and actions, it is well-known that the concentrability coefficient governs the statistical complexity of offline policy evaluation. To give some intuition for this claim, and to set the foundation for what follows, let \(x^\star\) denote the state that maximizes the right-hand side in the definition of the concentrability coefficient in \pref{eq: concentrability coeff}, and for simplicity, suppose that \(d^\pieval(x^\star) \geq \epsilon\). Now, consider two scenarios, the first where the \MDP has a reward of \(+1\) by taking any action in the state \(x^\star\), and the second where the \MDP has a reward of \(-1\) by taking any action in \(x^\star\); in both cases, we assume zero reward on all other states. Thus, to estimate the value of \(\pieval\) up to precision \(\epsilon\), the learner needs to distinguish between the two scenarios, and the only way to do so is to observe a transition from \(x^\star\) in the given offline dataset, which requires at least $\tfrac{1}{\mu(\sx^\star, \pieval(\sx^\star))} \geq \tabC$ offline samples from \(\mu\), in expectation. To conclude, in tabular MDPs, the learner can explicitly keep track of different states in the MDP, and use the corresponding transition and reward behavior on these states to evaluate \(\pieval\), and thus the worst case scenarios for offline policy evaluation is when the offline data does not provide enough information about the parts of the \MDP where \(\pieval\) has high visitation probability, and thus concentrability coefficient governs the statistical complexity. 

The offline policy evaluation problem unfortunately becomes more challenging when the \MDP has a large state space and the learner has to rely on function approximation. For this regime, previous works by \cite{xie2021batch} and \cite{foster2022offline} hint that the difficulty of offline policy evaluation comes from the hardness of distinguishing states that have different transition behaviors but the same values. Recall that every piece of data in the offline dataset is of the form $(\sx, \sa, \sr, \sx')$. If $\sx_1$ and $\sx_2$ are two states appearing in the dataset such that $\gF$ does not provide any information to distinguish them, i.e., $f(\sx_1, \cdot)=f(\sx_2, \cdot)$ for all $f\in\gF$, then the learner has no guidance from $\gF$ whether they are essentially the same state or not in terms of their rewards or dynamics behavior. There are also no clues from other parts of the dataset, since with high probability, every state only appears at most once in the dataset due to the large state space. Under such a challenging scenario, intuitively, the best the learner can do is to aggregate these two states together and treat them as the same item, to get the most out of the offline dataset and the given value function class. 
This algorithmic idea of ``aggregation'' is precisely what is used in the \BVFT algorithm of \cite{xie2021batch}. In this following section, we formalize the argument that aggregating indistinguishable states is indeed \emph{the best the learner can do} by showing a general lower bound in terms of aggregated concentrability coefficient. To establish our lower bound, in the next section, we formally define the notion of state aggregation and aggregated concentrability coefficient. 

\subsection{Aggregated Concentrability Coefficient}\label{sec: state aggregation} 
Over a given state space $\gX=\gX_1 \cup \dots \cup  \gX_H$, we can define a \emph{state aggregation scheme} $\Phi=\Phi_1\cup\dots\cup\Phi_H$ as below. For any $h$, $\Phi_h$ defines a partition of $\gX_h$ so that the following hold:
\begin{enumerate}[label=\(\arabic*)\)]  
    \item Every element $\phi\in\Phi_h$ is a subset of $\gX_h$; 
    \item  The subsets are disjoint, i.e., $\phi\cap \phi'=\emptyset$ for all  $\phi, \phi'\in\Phi_h$;  
    \item The subsets cover $\gX_h$, i.e., $\bigcup_{\phi\in\Phi_h}\phi=\gX_h$.  
\end{enumerate}


An aggregated Markov Transition Model $\agg{M}$ is defined via a underlying Markov Transition Model $M=(\gX, \gA, T, H, \rho)$, state aggregation schemes $\Phi_h$, and offline data distributions $\mu_h: \gX_h\times \gA\rightarrow \mathbb{R}_{\geq 0}$ for $1\le h\le H-1$. We write $\agg{M}=(M, \Phi, \mu)$. The aggregated transition dynamics for a  policy $\pi$ are defined by 

\begin{equation}\label{eq:agg-transition-model}
     \agg{T}(\phi'\mid{}\phi, \pi; \agg{M}) = \frac{\sum_{\sx\in\phi}\sum_{\sx'\in\phi'} \sum_{\sa\in\gA} \pi(\sa\mid \sx) \mu_h(\sx,\sa)T(\sx' \mid{} \sx, \sa)}{\sum_{\sx\in\phi} \sum_{\sa\in\gA}  \pi(\sa\mid\sx)\mu_h(\sx, \sa) } 
\end{equation}
for $\phi\in\Phi_h$, $\phi'\in\Phi_{h+1}$. 
The aggregated occupancy measure for a policy $\pi$ is defined as 
\begin{align*}
    \agg{d}_h^\pi(\phi; \agg{M}) \ldef{}  \mathbb{E}\left[ \indic\{\phi_h=\phi\}~\bigg|~ \phi_1\sim \agg{\rho}(\cdot), \phi_{i+1}\sim \agg{T}(\cdot \mid{} \phi_i,\pi; \agg{M}), \forall 1\le i\le h-1\right], 
\end{align*}
where the initial distribution $\agg{\rho}$ is defined as $\agg{\rho}(\phi)\ldef{}  \sum_{\sx\in \phi} \rho(\sx)$. 

Note that in general, it may not be meaningful to define aggregated transitions with respect to \emph{actions}, i.e., $\agg{T}(\phi'\mid \phi, \sa; \agg{M})$. This is because states in the same aggregation may not even share the same action space. However, in the special case where states within the same aggregation share the same action space, the quantity $\agg{T}(\phi'\mid \phi, \sa; \agg{M})$ can be defined, which could be useful in simplifying the notation. We use this notation in our lower bound proof (\pref{app:proof1}). 


\begin{definition}[Aggregated Concentrability Coefficient]\label{def:agg-concentrability}
    For an aggregated \MDP $\agg{M}=(M,\Phi, \mu)$ with underlying \MDP $M$, aggregation scheme $\Phi$, and offline distribution $\mu$, we define the aggregated concentrability coefficient $\aggC_\epsilon(M, \Phi, \mu)$ as
\begin{align*}
    \aggC_\epsilon(M, \Phi, \mu) = \max_h \max_{\cI} \left\{ \frac{\sum_{\phi\in \cI} \agg{d}_h^\pieval(\phi)}{\sum_{\phi\in \cI}\sum_{\sx\in\phi} \mu_h(\sx, \pieval(\sx))}  ~~\Bigg|~~ \gI \subseteq  \Phi_h, \ \ \sum_{\phi\in \gI} \agg{d}_h^\pieval(\phi)\ge \epsilon \right\}. 
\end{align*}
\end{definition}

The aggregated concentrability coefficient is analogous to the standard concentrability coefficient defined in \pref{eq: concentrability coeff}, but now under the aggregated transition model. 
The reason why the sum of aggregated occupancy measure is restricted to be at least $\epsilon$ above is because those $\phi$ with extremely small occupancy can be fully ignored during the policy evaluation process, while making no impact in the estimation error even if the above ratio is large.







\subsection{A General Lower Bound in Terms of Aggregated Concentrability Coefficient}

We now have all the necessary tools to state our first lower bound. The following theorem provides a general reduction that lifts any given instance of a Markov Transition Model, evaluation policy, offline data distribution, and aggregation scheme into a class of offline policy evaluation problems, and provides a statistical lower bound for offline policy evaluation for this class in terms of the aggregated concentrability coefficient. 

\begin{theorem} 
\label{thm: informal general data}
Let \(\epsilon \in (0, 1)\), $M$ be a Markov Transition Model, \(\Phi\) be an aggregation scheme over the states of \(M\), \(\pieval\) be a deterministic evaluation policy in \(M\) such that for any aggregation \(\phi \in \Phi\) and states \(x, x' \in \phi\) it holds that \(\pieval(x) = \pieval(x')\),  and \(\mu\) be a general offline data distribution with standard concentrability coefficient $\tabC(M, \mu)$ and aggregated concentrability coefficient $\aggC_\epsilon(M,\Phi, \mu)$. Then, there exists a class \(\famGg\) of realizable \OPE problems such that for every \OPE problem \(\opeg = \OPE\prn*{M\ind{\opeg}, \pieval\ind{\opeg}, \mu\ind{\opeg}, \cF\ind{\opeg}}\) in \(\famGg\), 
\begin{enumerate}[label=\((\alph*)\)]  
    \item  The function class \(\cF\ind{\opeg}\) satisfies \(Q^{\pieval}(\cdot ; M\ind{\opeg}) \in \cF\ind{\opeg}\) (\pref{ass:Q_realizability}), and \(\abs{\cF\ind{\opeg}} = 2\). 
    \item Any pair of states \(x, x'\) that belong to the same aggregation \(\phi \in \Phi\) satisfy \(f(x, \cdot) = f(x', \cdot)\) for all \(f \in \cF\ind{g}\). 
    \item The  concentrability coefficient $\tabC(M\ind{\opeg}, \mu\ind{\opeg})=\Theta(\tabC(M, \mu))$. 
\end{enumerate}
Furthermore, any offline policy evaluation algorithm that guarantees to estimate the value of \(\pieval\ind{\opeg}\) in the \MDP \(M\ind{\opeg}\) 
up to precision \(\epsilon\), in expectation, for every \OPE problem \(\opeg \in \famGg\) must use  
\begin{align*}
\wt \Omega\prn*{\frac{\aggC_\epsilon(M,\Phi, \mu)}{\epsilon}}
\end{align*}
offline samples from \(\mu\ind{\opeg}\) in some \OPE problem \(\opeg \in \famGg\). 
\end{theorem} 



The proof of \pref{thm: informal general data} is deferred to \pref{app:proof1}. In the proof, instead of directly using the given \MTM $M$ to construct the class \(\famGg\), we construct Block MDPs \(\crl{M\ind{\opeg}}_{\opeg\in \famGg}\) with latent state dynamics given by \(M\) (with three additional new latent states per layer). As shown in the appendix, this reduction ensures that the standard concentrability remains unchanged. Furthermore, we note that the function class \(\cF\ind{\opeg}\) and the evaluation policy \(\pieval\ind{\opeg}\) are the same for all instances \(\opeg \in \cG\), and that 
the aggregated concentrability coefficient in \(M\ind{g}\) is $\Theta(\aggC_\epsilon(M, \Phi, \mu))$ (see \pref{prop: rich aggregated concentrability}). We also have the following property.  

\CWremove{the aggregated concentrability coefficient in \(M\ind{g}\) w.r.t.~an appropriately defined lifting of \(\Phi\) to rich observations is $\Theta(\tabC(M, \mu))$ (\pref{prop: rich aggregated concentrability}).}  

\begin{property} 
\label{propr:admissibility} 
In the construction in \pref{thm: informal general data}, if the offline distribution \(\mu\) is admissible for the Markov Transition Model \(M\), then  for every \(\OPE\) problem \(\opeg \in \famGg\), the offline distribution \(\mu\ind{\opeg}\) is also admissible for the corresponding \MDP \(M\ind{\opeg}\). 
\end{property} 



\newcolumntype{C}[1]{>{\centering\arraybackslash}m{#1}}
\subsection{Can Aggregated Concentrability be Larger than Standard Concentrability?} \label{sec:concentrability_larger}  

\begin{figure}[h!]
    \centering
    \includegraphics[width=0.85\textwidth]{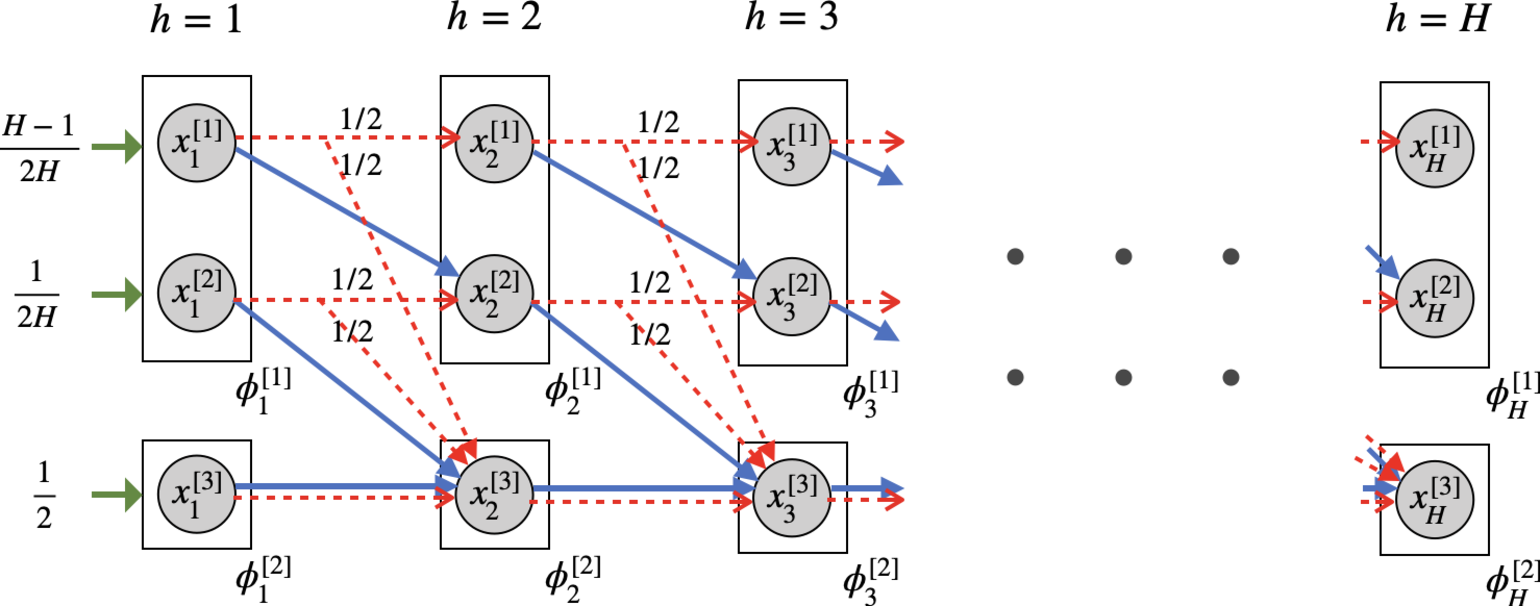} 
    \caption{Markov Transition Model and aggregation scheme in  \pref{eg:arbitrary_behavior}. The {blue} arrows represent the transitions under action $\aone$, and the {red} arrows represent the transitions under $\atwo$. The green arrows denote the initial distribution \(\rho\).} 
    \label{fig:example_MTM}
\end{figure}

\begin{figure}[h!] 
    \centering
\includegraphics[width=0.85\textwidth, trim={0cm 0cm 0cm 0cm},clip]{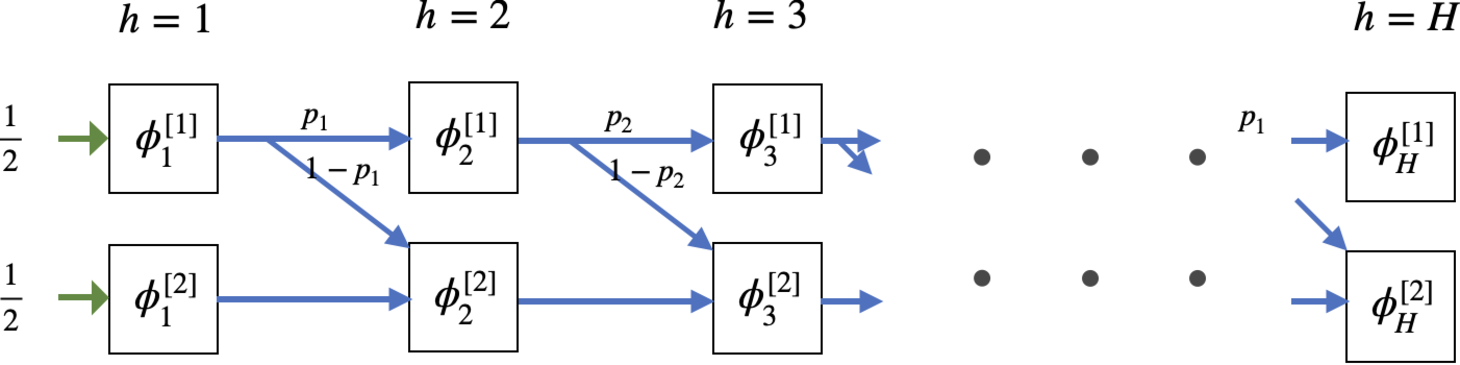} 
    \caption{Dynamics for policy \(\pieval\) in the aggregated \MDP \(\agg{M}\) in \pref{eg:arbitrary_behavior}, where $p_h\ldef{} \agg{T}\prn*{\phi_{h+1}\inds{1}\mid \phi_{h}\inds{1}, \pieval; \agg{M}}$. As shown in \pref{lem:sec-admissible-mu-2}, $p_h\ge \frac{H-1}{H+2}$ for all $1\le h\le H-1$. The green arrows denote the initial distribution \(\rho\) in the aggregated MDP.} 
    \label{fig:example_aggT} 
\end{figure}



\noindent
\pref{thm: informal general data} indicates that the sample complexity of offline policy evaluation (with general data) grows with the aggregated concentrability coefficient $\aggC_\epsilon(M,\Phi,\mu)$ instead of the standard concentrability coefficient $\tabC(M,\mu)$. Given this lower bound, one may wonder how large can the aggregated concentrability be in comparison to the standard concentrability. In this section, we will demonstrate via an example that the gap could indeed be exponential.  

\begin{example}
\label{eg:arbitrary_behavior} Consider the example represented in \pref{fig:example_MTM},  where 
\begin{enumerate}[label=\(\bullet\), itemsep=1mm] 
\item \textbf{Markov Transition Model} \(M\) consists of three states \(\crl{x_h\inds{1}, x_h\inds{2}, x_h\inds{3}}\) in each layer \(h \in [H]\), and two actions \(\cA = \crl{\aone, \atwo}\). The initial distribution \(\rho\) (denoted by the solid green arrows) is defined such that   
\(\rho\prn*{x_1\inds{1}} = \nicefrac{(H - 1)}{2H}\), \(\rho\prn*{x_1\inds{2}} = \nicefrac{1}{2H}\) and \(\rho\prn*{x_1\inds{3}} = \nicefrac{1}{2}\). 

The transition dynamics are identical for every layer  \(h \in [H-1]\), and is defined such that on taking action \(\aone\) the agent transitions deterministically according to the solid blue arrows, and  on taking action \(\atwo\), the agent transitions stochastically according to the dotted red arrows.


\item For every \(h \in [H]\), the \textbf{aggregation scheme} \(\Phi_h = \crl{\phi_h\inds{1}, \phi_h\inds{2}}\) with the aggregation \(\Phi_h\inds{1} = \crl{x_h\inds{1}, x_h\inds{2}}\) and \(\phi_h\inds{2} = \crl{x_h\inds{3}}\).  The aggregated states are denoted by rectangular blocks. 

\item The \textbf{offline distribution} \(\mu\) is the occupancy measure of an offline policy \(\pioff\) defined such that \(\pioff(x) = \frac{1}{H^2} \delta_{\aone}(\cdot) + \frac{H^2 - 1}{H^2} \delta_{\atwo}(\cdot)\) for all \(x \in \cX\)
\item The \textbf{evaluation policy} \(\pieval\) such that  \(\pieval(x) = \delta(\aone)\) for all \(x \in \cX\). 
\end{enumerate}

\end{example}

\begin{proposition} 
\label{prop:toy_example} 
For $\epsilon\le \nicefrac{1}{15}$, the Markov Transition Model \(M\), aggregation scheme \(\Phi\), evaluation policy \(\pieval\) and offline distribution \(\mu\) given in \pref{eg:arbitrary_behavior},   the standard concentrability coefficient \(\tabC(M,\mu)=O(H^3)\), whereas the aggregated concentrability coefficient \(\aggC_{\epsilon}(M, \Phi, \mu)= \wt \Omega(2^H)\). 
\end{proposition}

We now give a sketch for the proof of \pref{prop:toy_example}, with the full details deferred to \pref{app: admissible lb}. We first calculate the upper bound on the standard concentrability coefficient. 
First, we argue that for $h\geq 3$, $d^{\pieval}_h(\sx_h)=1$ if $\sx_h=\sx_h\inds{3}$ and $0$ otherwise. This can be easily observed from the transition of $\pieval$ (blue arrow) in \pref{fig:example_MTM}---following the blue arrow, the policy must stay in $\sx_h\inds{3}$ for $h\geq 3$.
\CWremove{Note that, for \(h = 1\), 
$$d_1^\pieval(x_1) = d_1^\pioff(x_1) =  \rho(x_1) = \crl*{\tfrac{H-1}{2H} ~\text{for}~x_1 = x_1\inds{1},\tfrac{1}{2H} ~\text{for}~x_1 = x_1\inds{2}, \text{and}~  \tfrac{1}{2} ~\text{for}~x_1 = x_1\inds{3}}.$$  
Now, since \(\pieval\) always takes action \(\aone\) which follows the transitions denoted by the solid blue arrow, \(d^\pieval\) will evolve such that  }
\CWremove{\begin{align*} 
d_2^\pieval(x_2) = \frac{H-1}{2H} \cdot \indic\crl{x_2 = x_2\inds{2}} +  \frac{H+1}{2H} \cdot \indic\crl{x_2 = x_2\inds{3}}, 
\end{align*}
and for \(h \geq 3\), $d_h^\pieval(x_h)  = 1$ if \(x_h = x_h\inds{3}\)  and \(0\) otherwise. 
}
Next, we lower bound the state occupancy under $\pioff$. We claim that 
\begin{align*} 
d_h^{\pioff}\prn*{\sx_h\inds{1}} \geq \Omega \left(\frac{1}{2^h}\right), \quad d_h^{\pioff}\prn*{\sx_h\inds{2}} \geq \Omega \left(\frac{1}{H2^h}\right),~\quad \text{\ and\ } \quad d_h^{\pioff}\prn*{\sx_h\inds{3}} \geq  \frac{1}{2}.  \numberthis \label{eq:toy_pioff_occupancy} 
\end{align*}
   The third inequality in \pref{eq:toy_pioff_occupancy} is easy to see since the occupancy on 
$\sx_h\inds{3}$ is non-decreasing w.r.t.~$h$ under any policy (\pref{fig:example_MTM}).  To see the first two inequalities in \pref{eq:toy_pioff_occupancy}, notice that since $\pioff$ chooses $\atwo$ with probability $1-\frac{1}{H^2}$, and $\atwo$ carries $\frac{1}{2}$ of the weights from $\sx_h\inds{1}$ to $\sx_{h+1}\inds{1}$ (depicted by the red arrow in \pref{fig:example_MTM}), we have $d_h^{\pioff}(\sx_h^{\inds{1}})\geq \rho(\sx_1^{\inds{1}})\prn*{\frac{1}{2}(1-\frac{1}{H^2})}^{h-1}= \Omega\prn*{\frac{1}{2^h}}$. Similarly, $d_h^{\pioff}(\sx_h^{\inds{2}})\geq \rho(\sx_1^{\inds{2}})\prn*{\frac{1}{2}(1-\frac{1}{H^2})}^{h-1}= \Omega\prn*{\frac{1}{H2^h}}$. With the above calculation  and that $\tfrac{\pieval(\sa|\sx)}{\pioff(\sa|\sx)}\leq H^2$, we conclude that 
$
\tabC \leq H^2 \max_h \max_{\sx} \tfrac{d_h^{\pieval}(\sx)}{d_h^{\pioff}(\sx)} \leq \order(H^3)
$.  

We now proceed to the lower bound on aggregated concentrability coefficient. From \pref{eq:agg-transition-model}, we know that the aggregated dynamic for \(\pieval\), shown in \pref{fig:example_aggT}, is constructed by reweighting the transition in the original transition model using \(\mu_h(\cdot)=d_h^\piexp(\cdot)\). 
Since $\pioff$ takes action $\atwo$ with large probability, and $\atwo$ does not change the relative weight $\tfrac{d_h^\pioff(x_h\inds{1})}{d_h^\pioff(x_h\inds{2})}$ (see the red arrows in \pref{fig:example_MTM}), it can be shown that $\tfrac{d_h^\pioff(x_h\inds{1})}{d_h^\pioff(x_h\inds{2})} \geq  \tfrac{\rho(x_1\inds{1})}{3 \rho(x_1\inds{2})}= \tfrac{H-1}{3}$ for all $h$. 
This gives 
\begin{align*}
    p_h\ldef{} \agg{T}(\phi_{h+1}\inds{1}\mid \phi_h\inds{1}, \pieval)= \frac{d_h^\pioff(x_h\inds{1})\cdot 1 + d_h^\pioff(x_h\inds{2}) \cdot (\nicefrac{1}{2})}{d_h^\pioff(x_h\inds{1}) + d_h^\pioff(x_h\inds{2})} \geq \frac{H-1}{H+2}, 
\end{align*}
where the factors of $1$ and $\nicefrac{1}{2}$ are the probability of transitioning to $\phi_{h+1}\inds{1}$ from $\sx_h\inds{1}$ and $\sx_h\inds{2}$ following $\pieval$.  This further implies 
$
\agg{d}_h^\pieval(\phi_h\inds{1}) \geq \frac{1}{2}\prn*{\frac{H-1}{H+2}}^h$. 
Using a similar argument as for \pref{eq:toy_pioff_occupancy}, we have \(\agg{d}_h^\pioff(\phi_h\inds{1}) \leq \tfrac{1}{2^h}\). These two bounds together imply that the aggregated concentrability is \(2^{\Omega(H)}\).

 \section{Main Lower Bounds for Offline Policy Evaluation} \label{sec: lower bound} 

\subsection{Admissible Data}\label{sec: admissible data}

\pref{eg:arbitrary_behavior} provides an instance of a Markov Transition Model, aggregation scheme, evaluation policy and offline distribution for which the standard concentrability is \(O(H^3)\) whereas the aggregated concentrated is \(2^{\Omega(H)}\). Since the offline distribution \(\mu\) in \pref{eg:arbitrary_behavior} is the occupancy measure \(d^\pioff\) for the policy \(\pioff\), plugging \pref{eg:arbitrary_behavior}  in \pref{thm: informal general data} implies the following lower bound for offline policy evaluation with admissible offline data.

\begin{theorem}\label{thm:admissible}
Let   $\epsilon\le \nicefrac{1}{15}$, and horizon \(H \geq 1\). Then, there exits a class \(\famGa\) of realizable \OPE problems, such that for every \OPE problem \(\opeg = \prn*{M\ind{\opeg}, \pieval\ind{\opeg}, \mu\ind{\opeg}, \cF\ind{\opeg}} \in \famGa\),  the concentrability coefficient of \(\pieval\ind{\opeg}\) w.r.t.~ \(\mu\ind{\opeg}\) is \(O(H^3)\), the offline distribution \(\mu\ind{\opeg}\) is admissible for the \MDP \(M\ind{\opeg}\), and \(\abs{\cF\ind{\opeg}} = 2\). 

Furthermore, any offline policy evaluation algorithm that guarantees to estimate the value of \(\pieval\ind{\opeg}\) in the \MDP \(M\ind{\opeg}\) 
up to precision \(\epsilon\), in expectation, for every \OPE problem \(\opeg \in \famGa\) must use $ 2^{\Omega(H)}$ offline samples in some \(\opeg \in \famGa\). 
\end{theorem}

The construction of the class \(\famGa\), and the proof of \pref{thm:admissible}, are deferred to \pref{app: admissible lb}. We remark that in all the \OPE problem instances \(\opeg \in \famGa\), the corresponding MDPs \(M\ind{\opeg}\) share the same action space \(\cA = \crl{a_1, a_2}\) (binary actions), state space \(\cX\) and horizon \(H\),  however, the transition dynamics, reward function and initial distribution could change with the instance. Furthermore, the policy \(\pieval\ind{\opeg}\) and the state-action value function class \(\cF\ind{\opeg}\) are also same across all instances \(\opeg \in \famGa\). 

Our lower bound in \pref{thm:admissible} considers admissible offline data distributions, where for any instance \(\opeg \in \famGa\) and \(h \leq H\), the offline distribution \(\mu_h\ind{\opeg} = d_h^{\pioff\ind{\opeg}}(\cdot; M\ind{\opeg})\), and the offline algorithm can draw samples of the form \((x_h, a_h, r_h, x_{h+1})\) from the process \((x_h, a_h) \sim \mu_h\ind{\opeg}\), \(r_h \sim r\ind{\opeg}(x_h, a_h)\) and \(x_{h+1} \sim T\ind{\opeg}(\cdot \mid x_h, a_h)\). Thus, \pref{thm:admissible} strengthens over the results of \cite{foster2022offline}, in which the offline data distribution is not equal to the occupancy measure of a single policy. 

Having shown that bounded concentrability coefficient and realizability alone are not sufficient for statistically efficient offline policy evaluation, even if the offline distribution \(\mu_h = d^{\pioff}_h\) is admissible, we now ask what happens if the learner has access to complete offline trajectories \(\prn{x_1, a_1, r_1, s_2, \dots, r_{H}, s_{H}}\) sampled using \(\pioff\). Unfortunately, for this scenario, the result of \pref{thm:admissible} no longer holds. This is because the reduction in \pref{thm: informal general data}, which is a key tool in the proof of \pref{thm:admissible}, does not prevent from leaking additional information when the learner has access to trajectories of length more than \(2\). In particular, by looking at the conditional distributions of \(x_3\) after fixing actions \(a_1\) and \(a_2\) for the first two timesteps in that construction (which can be computed when given trajectory data that covers the first two timesteps), the learner can infer the value of \(\pieval\) in the underlying MDP. In the next section, we develop additional tools to handle trajectory data.

\subsection{Trajectory Data} \label{sec: trajectory lower bound} 
  

In many real world  applications, the offline dataset is collected by sampling trajectories of the form \((\sx_1, \sa_1, \sr_1, \sx_2,  \dots, \sx_{H}, \sa_{H}, \sr_{H})\) and it remains to address whether access to the entire \(H\)-length trajectory instead of just the tuples \((\sx, \sa, \sr, \sx')\) can allow the learner to circumvent the challenges introduced in previous subsections. In fact, \citet{foster2022offline} left it as an open problem whether access to trajectory data can make offline RL statistically tractable. In this section, we answer this in the negative and show that in the worst case, access to trajectory data does not overcome the statistical inefficiencies of offline RL with just bounded concentrability coefficient and realizability.

\begin{theorem} 
\label{thm:trajectory_lower_bound} Let $\epsilon\le \nicefrac{1}{15}$, and horizon \(H \geq 1\). Then, there exits a class \(\famGt\) of realizable \OPE problems, such that for every \OPE problem \(\opeg = \prn*{M\ind{\opeg}, \pieval\ind{\opeg}, \pioff\ind{\opeg}, \cF\ind{\opeg}} \in \famGt\),  the learner has access to offline trajectories sampled using \(\pioff\ind{\opeg}\), the concentrability coefficient of \(\pieval\ind{\opeg}\) w.r.t.~ \(\pioff\ind{\opeg}\) is $O(H^3)$, and \(\abs{\cF\ind{\opeg}} = 2\). 

Furthermore, any offline policy evaluation algorithm that estimates the value of \(\pieval\ind{\opeg}\) in the \MDP \(M\ind{\opeg}\) 
up to precision \(\epsilon/(16H)\), in expectation, for every \OPE problem \(\opeg \in \famGt\) must use $2^{\Omega(H)}$ offline trajectories in some \(\opeg \in \famGt\). 
\end{theorem} 

\begin{figure}[h!]
    \centering
    \includegraphics[width=0.85\textwidth]{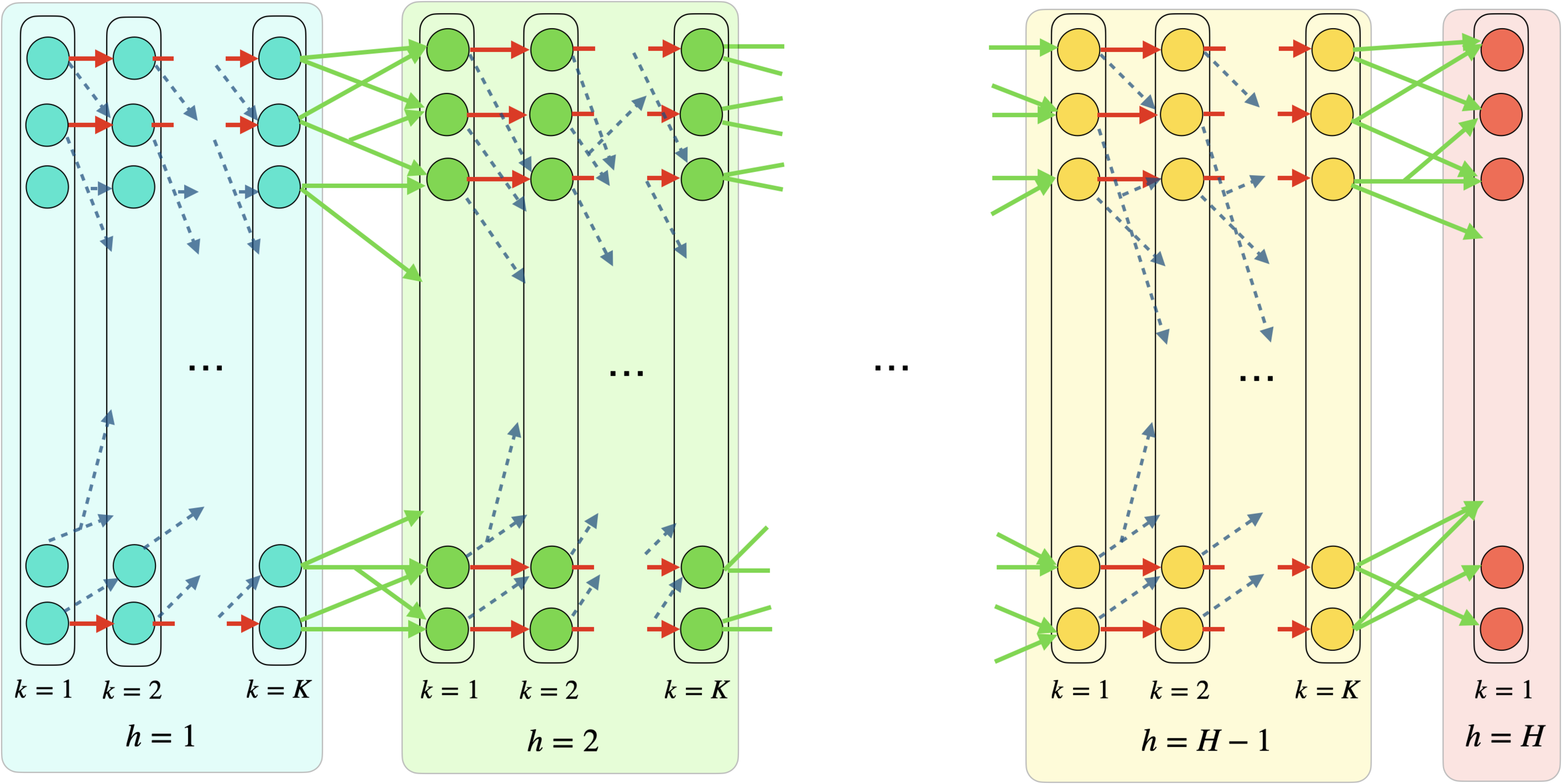} 
    \caption{Lower bound construction for the proof of \pref{thm:realizable-H}. For each \(h \in [H]\), the corresponding block denotes the \(K\) layers that are obtained using \Replicator by replicating the \(h\)-th layer in the given \MDP $M$ for \(K\) many times. The solid {red} arrows represent the transitions under the action $\aone$,  the dotted {blue} arrows represent the transitions under the $\atwo$ (under which we resample from the admissible distribution \(\mu_h\)), the solid {green} arrows denote the transitions according to the original \MDP \(M\).} 
    \label{fig: trajectory}
\end{figure} 

While the full proof is deferred to \pref{app: proof of reduction}, we present the main ideas and the key tools below. The primary reason why the lower bound from \pref{thm:admissible} does not hold under trajectory data is that access to trajectories spanning more than two timesteps in the underlying \MDP{}s in that construction leaks additional information, which can be exploited by the learner to evaluate \(\pieval\). In particular, given trajectory data, the learner can compute the marginal distribution over \(x_{h+2}\) given actions \(a_{h}\) and \(a_{h+1}\), for \(h \leq H - 2\), which can be used to identify the underlying instance in the class \(\famGa\) and thus compute \(\pieval\). Our key insight in the proof is to fix this problem of information leakage by introducing a general-purpose reduction from offline RL with admissible data to offline RL with trajectory data, which may be of independent interest. This reduction is obtained by using two new protocols called (a) the \Replicator~protocol, and (b) the \converter~protocol, which we describe below. 


\paragraph{\Replicator:} Given in \pref{alg:replicator} in the appendix, the \Replicator~protocol  takes as input a realizable \OPE problem \(\opeg = (M, \pioff, \pieval, \cF)\) where the \MDP \(M\) has horizon \(H\), and a parameter \(K\), and converts it into another realizable \OPE problem \(\wt \opeg = (\wt M, \wt \pioff, \wtpieval, \cF)\) where the new \MDP $\wt M$ has horizon \(\wt H = (H-1)K + 1\). We require that \Replicator satisfies the following property. 


\begin{property}[Informal; Formal version in \pref{lem:sec-trajectory-concentrability}]
\label{propr:claim1} 
The realizable \OPE problem \(\wt \opeg \leftarrow \Replicator(\opeg, K)\) satisfies the following: 
\begin{enumerate}[label=\((\alph*)\)] 
	\item Concentrability coefficient $\sup_{\wt h \leq \wt H} \nrm*{\frac{d_h^{\wtpieval}}{d_h^{\wt \piexp}}}_\infty \leq  2 \sup_{h \leq H} \nrm*{\frac{d_h^{\pieval}}{d_h^\piexp}}_\infty.$ 
\item The value of the policy \(\wtpieval\) in \(\wt M\) is equal to the value of \(\pieval\) in \(M\). 
\end{enumerate}
\end{property}

Our construction of \Replicator essentially replicates each layer in the given MDP $M$ for $K$ times (except for the last layer); see \pref{fig: trajectory} for illustration. In the following, we call these replicated layers as \emph{sub-layers}. We first define the transition function. For the last sub-layer (i.e., for $k=K$) of each layer $h \leq H-1$, the transition is exactly the same as that in the MDP from layer \(h\) to \(h +1\) (denoted by the green arrows in \pref{fig: trajectory}). For other sublayers with \(k < K\), the transitions are designed such that: if the action $\aone$ is taken, then the state transitions to the same state in the next sub-layer (red arrows in \pref{fig: trajectory}); and if the action $\atwo$ is taken, the next state is sampled independently from the offline data distribution $\mu_h = d^{\pi_b}_h$ (blue arrows in \pref{fig: trajectory}). Furthermore,  the evaluation policy $\wtpieval$ in the new MDP is the same as $\pieval$, and takes action \(\aone\) on all states. The offline policy $\wt\pioff$ is set as $\pioff$ for the last sub-layer (i.e. $k=K$), and is set as $\frac{1}{2} \prn*{\delta_{\aone} + \delta_{\atwo}}$ for the intermediate sub-layers with $k \leq K-1$.  

The rationale behind this design is that since $\wt\pioff=\frac{1}{2} \prn*{\delta_{\aone}+\delta_{\atwo}}$, for each $h$, with probability $1-2^{-K+1}$ the offline policy will choose $\atwo$ at least once in sub-layers $k=1, \ldots, K-1$. If $\atwo$ is chosen at least once, then the state distribution at $k=K$ is equal to $\mu_h=d^{\pioff}$ and independent from all previous layers $1, \ldots, h-1$.  As long as $K$ is large enough, this happens with very high probability, which makes the offline data distribution at sub-layer $k=K$ resemble admissible data with distribution $\mu_h=d^{\pioff}_h$, even when the data is actually a complete trajectory. It can be shown that this conversion preserves the concentrability coefficient up to a constant factor. 



\paragraph{\converter:} Given in \pref{alg:convert} in the appendix, this protocol takes as input \(K\) tuples of the form \((\sx_h, \sa_h, \sr_h, \sx_{h+1})\) sampled from an admissible offline distribution \(\mu_h = d_h^{\pioff}\), for every \(h \in [H]\), and returns a trajectory \(\wt \tau\) of length \(\wt H\) in \(\wt M\). We require that \converter satisfies the following property.

\begin{property}[Informal; Formal version in \pref{lem:sec-trajectory-convert}] 
\label{propr:claim2}
For a large class of offline policies \(\pioff\), 
the distribution of trajectory \(\wt \tau\), constructed by \(\converter\) using offline data tuples from \(d^\pioff\), is close to the distribution of trajectories \(\wt \tau'\) obtained using \(\wt \pioff\) in \(\wt M\). 



\end{property} 

The idea of \converter is straightforward: We already argue that \Replicator can simulate admissible data using trajectory data. Hence, with a reverse process, given an admissible dataset, we can create a new trajectory dataset in a new MDP that simulates the original admissible dataset. 

With the above two protocols, the reduction of offline RL with admissible data to trajectory data is straightforward and is stated in \pref{alg:convert-admissible} in the appendix. At a high level, given a realizable \OPE problem \(g\) with admissible offline distribution \(d^{\pioff}\), for some large enough \(K\), we use the \Replicator protocol to create a realizable \OPE problem \(\wt g\) and use the \converter protocol to generate trajectory data corresponding to \(\wt \pioff\) in \(\wt M\). Since, \pref{propr:claim1}-(a) implies that the concentrability coefficient stays bounded and \pref{propr:claim1}-(b) implies that the value to be evaluated remains unchanged, the above reduction provides a way to solve offline RL with admissible data by invoking an offline algorithm that requires trajectory data. Thus, if trajectory data with bounded concentrability coefficient is tractable, then so is admissible data by leveraging \pref{alg:convert-admissible}, which contradicts \pref{thm:admissible}. This implies that offline RL with trajectory data must also be statistically inefficient. The formal proof is deferred to  \pref{app: proof of reduction}.

\section{Upper Bound} 
\label{sec: upper bound} 
Our lower bounds show that the worst-case sample complexity of offline policy evaluation grows with the aggregated concentrability. In this section, we complement our lower bounds with an upper bound of the form $\poly(\aggC, H, \epsilon^{-1}, \log|\gF|)$. Taken together, the lower and upper bound suggest that aggregated concentrability, but not the standard concentrability, characterize the worst-case sample complexity of offline policy evaluation with value function approximation. 


\begin{theorem} 
\label{thm:main_BVFT_upper_bound}Let \(\epsilon > 0\), \(\cF\)  be a state-action value function class that satisfies \pref{ass:Q_realizability}, \(\pieval\) be an evaluation policy and \(\mu\) be an offline (general) data distribution. Then, \pref{alg: BVFT} (a adaptation of the \BVFT algorithm of \cite{xie2021batch}) returns a \(\wh V\) such that  $\abs*{\E_{\sx\sim \rho} \brk{V^\pieval(\sx) - \wh V}} = \order\left( \epsilon\right)$ after collecting 
\begin{align*}
n = \cO \prn*{\frac{\aggC\cdot H^6\log(|\gF|/\delta)}{\epsilon^4}} 
\end{align*} 
many (offline) samples from \(\mu\), where $\aggC:=\max_{f,f'\in\gF}\aggC_{\epsilon^2/H^2}(M, \Phi(f,f'), \mu)$ and \(\Phi(f,f')\) is the state aggregation scheme determined by  $f,f'\in\gF$ (see \pref{def: partition induced by two f} for the precise definition). 
\end{theorem} 

The proof of \pref{thm:main_BVFT_upper_bound} is deferred to \pref{app: upper bound}, wherein we also provide a generalization of this result that accounts for misspecification error in \pref{ass:Q_realizability}, and an upper bound for a slightly more challenging scenario where the learner only has access to the state value function class (instead of the state-action value function class). Note that the upper bound depends on the aggregations scheme \(\Phi(f, f')\). The appearance of such an aggregation scheme in the upper bound is not surprising. In our lower bound in \pref{thm: informal general data}, while \(\Phi\) is given as an input, the corresponding value function class \(\cF\) constructed for the class \(\famGg\) satisfies that \(\Phi(f, f') = \Phi\) (see item-(b) in \pref{thm: informal general data}).  

\pref{alg: BVFT} is an adaptation of the \BVFT algorithm \citep{xie2021batch} for offline policy evaluation. At a high-level, the key idea in the algorithm is to solve a minimax problem (with the objective determined by Bellman error) over pairs \((f, f') \in \cF\times \cF\), where for each pair, the algorithm creates a ``tabular problem'' by aggregating states with the same $(f(\sx), f'(\sx))$ value, and estimates the Bellman error for this tabular problem. Intuitively, this is probably the best the learner can do, since besides the value of $(f(\sx))_{f\in\gF}$, the learner has no other ways to distinguish states in the large state space. Thus, due to aggregation, the upper bounds depends on aggregated concentrability coefficeint rather than the standard concentrability coefficient. 

  
 We remark that while \cite{xie2021batch} do not present their upper bound in terms of the aggregated concentrability, this quantity already appears in their analysis (see Appendix C in \cite{xie2021batch}). However, their final bound is represented with a stronger version of concentrability coefficient $\Cpush$ (\emph{pushforward concentrability coefficient}, formally defined in \pref{def:strong-coverability} in the appendix). It is  straightforward to show 
$\aggC\leq \Cpush$ (\pref{lem: C and Cpush}). Our analysis follows theirs, but along the way does not relax $\aggC$ to $\Cpush$.

\section{Conclusion} \arxiv{\label{app:discussions}} 


Our paper considers the problem of offline policy evaluation with value function approximation,  where the function class does not satisfy Bellman completeness, and shows that its  sample complexity is characterized by the aggregated concentrability coefficient---a notion of distribution mismatch in an aggregated \MDP obtained by clubbing together transitions from the states that have indistinguishable value functions under the given value function class (formal details in \pref{sec: general lower bound}). We provide an example of an MDP where the aggregated concentrability coefficient could be exponentially larger than the concentrability coefficient, using which we conclude that statistically efficient offline policy evaluation is not possible with bounded concentrability coefficient even if we assume access to trajectory data. This result thus highlights the necessity for further research into designing more effective strategies for dealing with the complexities inherent in offline reinforcement learning environments. 
\arxiv{We conclude with some discussion:} 
\colt{Further discussion of our results and related works are given in \pref{app:discussions} and \pref{app:related_works}  respectively.}
\arxiv{\colt{\section{Discussion} \label{app:discussions}}

\paragraph{Could the Aggregated Concentrability be Smaller Than the Standard Concentrability?} In \pref{sec:concentrability_larger}, we demonstrated via an example that the aggregated concentrability can be exponentially larger than the standard concentrability. However, it is also quite easy to come up with situations where the aggregated concentrability is actually smaller than the standard one. For example, suppose the aggregations scheme  \(\Phi_h = \crl{\phi_h}\) with \(\phi_h = \cX_h\), i.e. all states belong to a single aggregation. Here, the aggregated concentrability coefficient is exactly \(1\) since each layer has only one aggregation, whereas the standard concentrability coefficient could be arbitrary. 


\paragraph{Gap Between Upper and Lower Bounds in terms of Dependence on \(\epsilon\).}  
The sample complexity in our upper bound (\pref{thm:main_BVFT_upper_bound}) scales with $\frac{1}{\epsilon^4}$ instead of the more common $\frac{1}{\epsilon^2}$. This is similar to \cite{xie2021batch} and is because we divide the state space into $\order(\frac{1}{\epsilon^2})$ aggregations, each of which consists of states having the same value functions up to an accuracy  $\epsilon$. On the other hand, 
our lower bound has a \(\Omega(\frac{1}{\epsilon})\) dependence instead of the more common  $\Omega(\frac{1}{\epsilon^2})$. 
Improving the dependence on \(\epsilon\) in either the upper or lower bounds is an interesting future research direction.

\paragraph{Connections to Other Notations of Concentrability.} Various other notions of concentrability like pushforward concentrability  $\Cpush$, and action concentrability \(\CpushA\) \pref{def:strong-coverability}  are considered in the literature \citep{xie2021batch}.  We show that $\aggC\leq \Cpush$ (\pref{lem: C and Cpush}) and $\aggC\leq (\CpushA)^H$ (\pref{lem: C and BH}). Note that the sample complexity bound of $\order\prn*{(\CpushA)^H}$ is also what we get by using \emph{importance sampling} to perform offline policy evaluation. 

\paragraph{Single Policy vs.~All Policy Concentrability.} The notions of \emph{realizability} and \emph{concentrability coefficient} adopted in our paper are only with respect to the given \emph{evaluation  policy}. This is also called \textit{single policy concentrability} and is the standard assumption in offline policy evaluation literature. An alternative assumption that is used in the offline policy optimization literature is that of \textit{all policy concentrability}, which states that the concentrability coefficient for all possible policies (in the \MDP) w.r.t.~the given offline data is bounded. While we restrict our discussions to the former, our construction in \pref{eg:arbitrary_behavior} has bounded concentrability coefficient for all policies (\pref{app: admissible lb}). An interesting future research direction is to extend our lower bounds to the policy optimization setting. Perhaps, in the policy optimization setting, one can also ask whether generative or online access helps in addition to offline data and concentrability. 

\paragraph{Role of Realizable Value Function Class in Offline RL.} In this paper, we considered the realizable setting (\pref{ass:Q_realizability}) where the learner has access to a value function class that contains \(Q^{\pieval}\), 
and showed that the statistical complexity of offline policy evaluation is governed by the aggregated concentrability coefficient for the aggregation scheme induced by the given function class. However, how important is this access to the value function class? In particular, is statistically efficient offline RL feasible in the agnostic setting where the learner does not have any value function class? Unfortunately, as we show in \pref{app:realizable_function_class}, agnostic offline policy evaluation is not statistically tractable in the worst case even when the learner is given trajectory offline data that has bounded pushforward concentrability coefficient (From \pref{lem: C and Cpush} recall that this implies bounded aggregated concentrability coefficient for any aggregation scheme). Hence, further structural assumptions on the underlying MDP or the policies are needed for tractable learning. \citet{sekhari2021agnostic, jia2024agnostic} explored some structural assumptions that enable agnostic learning in the online RL setting, and extending their work to the offline setting is an interesting future research direction. 

\paragraph{How to Benefit from Trajectory Offline Data?} Our work indicates that in the worst-case, trajectory offline data provides no additional statistical benefit over General or Admissible offline data in the standard offline RL setting with value function approximation and bounded concentrability coefficient. But not all MDPs are the worst-case. Can we expect some instance-dependent benefit from access to trajectory data in offline RL? Alternately, can we make further assumptions on the underlying MDP or the value function classes, that are benign enough to capture real-world scenarios, but allow the learner to better exploit trajectory data. Furthermore, it is also interesting to study whether we can get statistical or computational improvements under trajectory data when the Bellman Completeness property holds. 

 


\arxiv{\section*{Acknowledgement} 
AS thanks Dylan Foster and Wen Sun for useful discussions. We acknowledge support from the NSF through award DMS-2031883, DOE through award DE-SC0022199, and ARO through award W911NF-21-1-0328. }  
}

\newpage  
\bibliography{refs, reference} 

\clearpage 

\setlength{\parindent}{0in} 
\setlength{\parskip}{3pt}  

\appendix 

\renewcommand{\contentsname}{Contents of Appendix} 
\tableofcontents 
\addtocontents{toc}{\protect\setcounter{tocdepth}{3}} 

\clearpage 

\colt{}
\section{Related Works}  \label{app:related_works}

Offline RL is challenging due to lack of direct interaction with the environment. Existing theoretical works that provide polynomial sample complexity guarantees often rely on multiple assumptions to be satisfied simultaneously. Specifically, in the realm of value function approximation, three pivotal assumptions stand out: (value function) realizability, concentrability, and Bellman completeness (i.e. $\mathcal{T}_h f_{h+1}\subset \mathcal{F}_h$ for all \(f_{h+1} \in \cF_{h+1}\)). The first two assumptions can be further categorized into single-policy concentrability (i.e., only the  target policy has bounded concentrability) and all-policy concentrability (all policies in the \MDP have bounded concentrability). 

\paragraph{Bellman Completeness.} If Bellman completeness holds, either all-policy realizability with single-policy concentrability \citep{xie2021bellman} or single-policy realizability with all-policy concentrability \citep{chen2019information} can guarantee polynomial sample complexity for policy optimization. Furthermore, other classical algorithms like Fitted Q-Iteration (FQI) \citep{munos2003error, munos2008finite, antos2008learning} are proved to have finite sample guarantee in terms of concentrability. 
The Bellman completeness assumption, however, is deemed rather undesirable because it is non-monotone in the function class and thus may be severely violated when a rich function class is used. Several efforts have been made to remove this assumption, though all requiring new assumptions: \cite{xie2021batch} showed that if a stronger version of concentrability, called \emph{pushforward} concentrability, holds, then with only single-policy realizability, polynomial sample complexity can be achieved without Bellman completeness. \cite{xie2020q}, \cite{zhan2022offline}, and \cite{ozdaglar2023revisiting} introduced the notion of \emph{density-ratio realizability} (different from value function realizability), and showed that this along with single-policy realizability and single-policy concentrability ensures polynomial sample complexity. \cite{zanette2023realizability} relaxed Bellman completeness to the notion of \emph{$\beta$-incompleteness} where Bellman completeness corresponds to $\beta=0$. He proved that $\beta<1$ along with realizability and concentrability admits polynomial sample complexity for policy evaluation.  

The question of whether just realizability and concentrability alone are sufficient for sample efficient offline RL remained open until the work of \cite{foster2022offline}, who answered this in the negative. They gave two examples where polynomial samples is insufficient even with all-policy realizability and all-policy concentrability. However, their lower bounds heavily rely on the offline data distribution being non-admissible, leaving the admissible and the trajectory cases open (see definitions and comparison in \pref{sec:offline_RL_prelims}).  

Further works on offline RL include \cite{liu2018breaking, uehara2020minimax, uehara2021finite} that focus on policy evaluation, and \cite{zhan2022offline, huang2022convergence, chen2022offline, rashidinejad2022optimal, ozdaglar2023revisiting} that focus on policy optimization. 



\paragraph{Other Lower Bounds in Offline RL.} There is another line of works showing exponential lower bound / impossibility results for offline policy evaluation with linear function approximation, but with concentrability replaced by other weaker notions of coverage \citep{wang2020statistical, amortila2020variant, zanette2021exponential}, e.g. the linear coverability assumption that $\lambda_{\min}(\mathbb{E}_{(s, a)\sim \mu} \phi(s, a)\phi(s, a)^T) \ge 1/d$. However, their alternate assumptions do not imply concentrability; Furthermore, these prior works also do not consider trajectory data, as in our results. More positive results can be found in the literature of \emph{model-based} approaches, for which we refer the reader to \cite{uehara2021pessimistic} and the related works therin. 

\paragraph{Online RL.} In online RL, while value function realizability and Bellman completeness is still a common assumption, the bounded concentrability coefficient assumption can be replaced by some low rank  structure on the Bellman error or its estimator \citep{jiang2017contextual, zanette2020learning, du2021bilinear, jin2021bellman}, which allow for efficient exploration. Recently, \cite{xie2022role} identified a new structural assumption called coverability which resembles all-policy concentrability and ensures polynomial sample complexity when combined with Bellman completeness. There have been various works in online RL that attempt to relax the Bellman completeness assumption by instead assuming density ratio realizability \cite{amortila2024harnessing}, occupancy realizability \cite{huang2023reinforcement}. Additionally, \cite{krishnamurthy2016pac, du2019provably, misra2020kinematic, zhang2022efficient, mhammedi2023representation}, focus on the simpler setting of block MDPs (which is a special case of density ratio realizability). It is an interesting direction to further unify the common notions used in online and offline RL.

\section{Additional Definitions and Notation}
In this section, we provide additional definitions and notations used in the appendix. 


\begin{definition}[Pushforward Concentrability Coefficient and Action Concentrability Coefficient; {\citet[Assumption 1]{xie2021batch}}]  \label{def:strong-coverability} For a distribution $\mu\in\Delta(\calX\times\calA)$, if we further assume that  
    \begin{enumerate}[label=\((\alph*)\)]
        \item There exists some $\CpushA > 0$ such that   $\mu(\sa\mid\sx)\ge 1/{\CpushA}$ 
        for any $\sx\in\calX, \sa\in\calA$,
        \item There exists some $\CpushX > 0$ such that the transition model satisfies $T(\sx'\mid\sx, \sa)/\mu(\sx') < \CpushX$, and the initial distribution $\rho$ satisfies $\rho(\sx)/\mu(\sx) < \CpushX$  for any $\sx, \sx'\in\calX, \sa\in\calA$,
    \end{enumerate} 
    then we say that the \MDP's pushforward concentrability coefficient with respect to $\mu$ is $\Cpush = \CpushX\CpushA$, and the \MDP's action concentrability coefficient with respect to $\mu$ is $\CpushA$.
\end{definition}

\paragraph{Aggregated Transitions with  Actions.} We further define the aggregated transitions with actions:
\begin{equation}\label{eq:agg-transition-model-Ta}
     \agg{T}(\phi'\mid{}\phi, \sa; \agg{M}) \ldef{} \frac{\sum_{\sx\in\phi}\sum_{\sx'\in\phi'}\mu(\sx, \sa)T(\sx' \mid{} \sx, \sa)}{\sum_{\sx\in\phi} \mu(\sx, \sa) } 
\end{equation}
Notice that when $\pi(\sx) = \delta_{\sa}(\cdot)$, i.e. $\pi$ takes action $\sa$ with probability $1$ at all states, $\agg{T}(\phi'\mid{}\phi, \sa; \agg{M})$ in \pref{eq:agg-transition-model-Ta} agrees with $\agg{T}(\phi'\mid{}\phi, \pi; \agg{M})$ in \pref{eq:agg-transition-model}.

\begin{definition}[Block MDP; \cite{du2019provably, misra2020kinematic}]
    A block \MDP is defined on top of a latent \MDP $M = (\calZ, \calA, T, r, H, \rho)$, a rich observation state space $\calX$ (partitioned into disjoint blocks $\calX_\sz$ for each latent state $\sz$), a decoder function $\xi$ and a conditional distribution $q(\cdot\mid \sz)\in \Delta(\calX_\sz)$. The block \MDP $\rich{M} = (\calX, \calA, \rich{T}, \rich{r}, H, \rich{\rho})$ with $\rich{T}(\sx\mid \sx, \sa) = q(\sx'\mid \xi(\sx'))T(\xi(\sx')\mid \xi(\sx), \sa)$, $\rich{r}(\sx, \sa) = r(\xi(\sx), \sa)$ and $\rich{\rho}(\sx) = \rho(\xi(\sx))q(\sx\mid\xi(\sx))$.
\end{definition}


\begin{definition}[$W$-function of \OPE problems]\label{def: W-function}
    Given an \OPE problem $(M, \mu, \pieval, \cF)$, the $W$-function: $W^\pieval(\cdot; \mu, M): [H]\to \mathbb{R}$ is defined as
    \begin{align*}
    W^\pieval(h; \mu, M) = \sum_{\sz\in\calZ_h}\mu_{h}(\sz, \pieval(\sz))Q_h^\pieval(\sz, \pieval(\sz); M). \numberthis \label{eq:W_definition} 
    \end{align*}
    Whenever clear from the context, the dependence on $\pieval$, \(\mu\) and \(M\) will be ignored. 
\end{definition}

\paragraph{Additional Notation.}  For $n \in \bbN$, we write $[n] = \{1, \dots, 
n\}$. For a countable set $\cS$, we write $\Delta(\cS)$ for the set of 
probability distributions on $\cS$. For any function \(u: \cX \times \cA \mapsto \bbR\) and distribution \(\rho \in  \Delta(\cX \times \cA)\), we define the norms \(\nrm{u}_{1, \rho} = \En_{(x, a) \sim \rho} \brk*{\abs{u(x, a)}}\) and  \(\nrm{u}_{2, \rho} = \sqrt{\En_{(x, a) \sim \rho} \brk*{u^2(x, a)}}\). For a distribution $\mathbb{P}\in \Delta(\calX)$, we define the cross product of $\mathbb{P}^{\otimes n}$ to be a distribution over $\calX^n$ such that $\mathbb{P}^{\otimes n}((\sx_1, \cdots, \sx_n)) = \prod_{i=1}^n\mathbb{P}(\sx_i)$, where $\sx_i\in \calX$. We use $\TV(p, q)$ and $D_{\chi^2}(p \|q )$ to denote the TV distance and $\chi^2$-divergence between two distribution $p$ and $q$. 


 
%
%
%
%
 
\section{Proof of \pref{thm: informal general data}} \label{app:proof1} 


\par Suppose we are given the Markov Transition Model (MTM) $M = \MTM(\calZ, \calA, T, H, \rho)$, and a distribution $\mu$ over $\calZ\times \calA$. $\Phi$ is an aggregated scheme so that every $\sz\in \calZ$ belongs to exact one of $\phi\in \Phi$, written as $\sz\in \phi$ (also all the latent states in $\phi$ should be at the same layer). We further define the aggregated function $\zeta: \calZ\to \Phi$, where for any $\sz\in\phi$, 
\begin{align}\label{eq:def-zeta}
    \zeta(\sz) \ldef{}  \phi
\end{align}

In the proof we will construct two class of offline policy evaluation (OPE) problems $\mathfrak{G}\ind{1}$ and $\mathfrak{G}\ind{2}$ from the given \MDP $M$ and distribution $\mu$. And we will prove \pref{thm: informal general data}  by showing that there exists an \OPE problem in $\mathfrak{G} = \mathfrak{G}\ind{1}\cup \mathfrak{G}\ind{2}$ that requires $\Omega(\aggC/\epsilon)$ number of samples for each layers to achieve accuracy $1/2$. The constructive proof is divided into three parts: 
\begin{enumerate}[label=(\roman*)]
    \item Construct aggregated MDPs $\bar{M}\ind{1}$ and $\bar{M}\ind{2}$ according to $M$ and $\Phi$ such that the concentrability coefficients of $\bar{M}\ind{1}$ and $\bar{M}\ind{2}$ are of order $\aggC$ (\pref{sec:aggregate}).
    \item Construct two \OPE problems $\mathfrak{g}\ind{1} = \OPE(M\ind{1}, \pieval, \mu', \cF)$ and $\mathfrak{g}\ind{2} = \OPE(M\ind{2}, \pieval, \mu', \cF)$ (\pref{sec:latent}), where \MDP $M\ind{1}$ and $M\ind{2}$ are obtained by adding three states $\su_h, \sv_h$ and $\sw_h$ in each layers. Distribution $\mu'$ is obtained from $\mu$ after rearranging some probability to $\su_h, \sv_h$ and $\sv_h$. And we can show that the concentrability coefficients of $\bar{M}\ind{1}$ and $\bar{M}\ind{2}$ can translate to the ratio between difference of value functions and difference of rewards between $M\ind{1}$ and $M\ind{2}$.
    
    \item Construct two class of \OPE problems $\mathfrak{G}\ind{1}$ and $\mathfrak{G}\ind{2}$ by lifting \OPE problems $\mathfrak{g}\ind{1}$ and $\mathfrak{g}\ind{2}$ into rich observations (\pref{sec:rich}).
\end{enumerate}

\subsection{Construction Sketch} 
In this subsection, we give a high-level sketch for the proof of \pref{thm: informal general data}. The full proof is detailed in the follow-up subsections. 


Suppose we are given an arbitrary Markov Transition Model $M$ with state space $\gZ$, transition dynamics \(T\) and initial distribution $\rho$, any offline data distribution $\mu$, and any state aggregation scheme $\Phi$ (see \pref{fig: original M}). Let $\gI$ and $h^*$ denote the set of aggregations and horizon that attain the maximum for $\aggC_\epsilon(M,\Phi,\mu)$ given in \pref{def:agg-concentrability}. In \pref{fig: original M}, $\gI$ is represented with the bold rectangle (for simplicity, in \pref{fig: lower bound general}, \(\cI\) only includes a single aggregation that contains a single latent state $\sz^\star$, but in general $\cI$ may include multiple aggregations each with multiple latent states). Based on $M$, we will construct two MDPs $M\ind{1}$ (with transitions $T\ind{1}$ and reward $r\ind{1}$) and $M\ind{2}$ (with transitions $T\ind{2}$ and reward $r\ind{2}$), and will argue that it is difficult for the learner to tell them apart when the MDPs are lifted to \emph{block MDPs}.


\begin{figure}[t]
\begin{center}
\subfigure[Given Markov Transition Model $M$ and aggregation scheme $\Phi$.]{\includegraphics[width=0.42\linewidth]{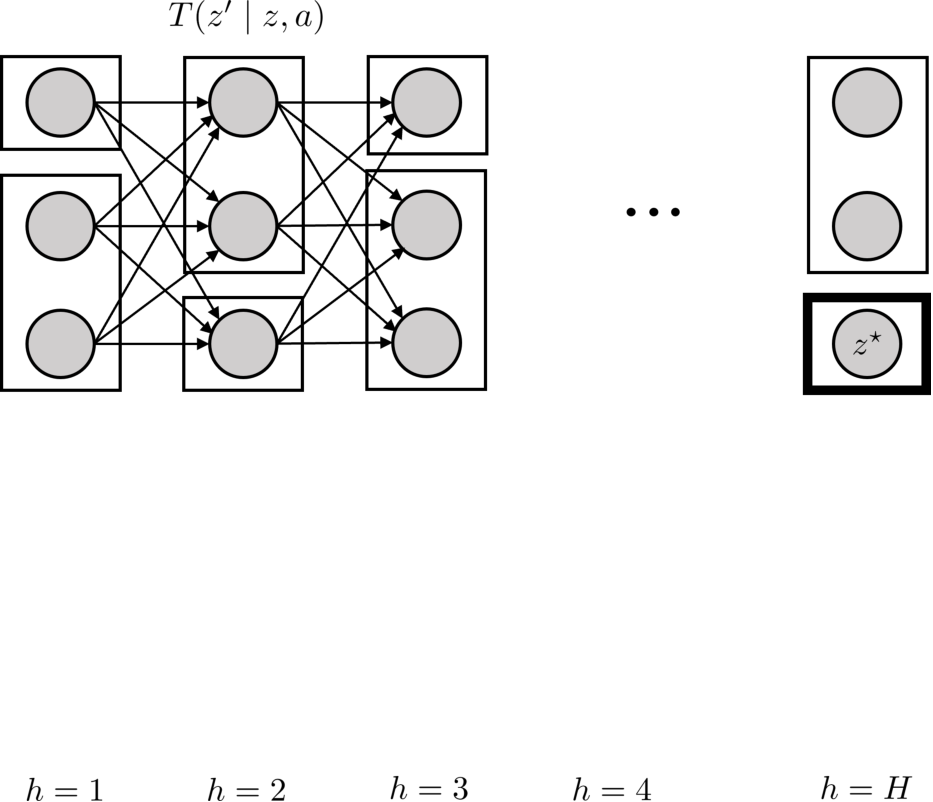} \label{fig: original M}}
\hfill  \subfigure[Augmented Markov Transition Model $M'$ and aggregation scheme $\Phi'$.] {\includegraphics[width=0.42\linewidth]{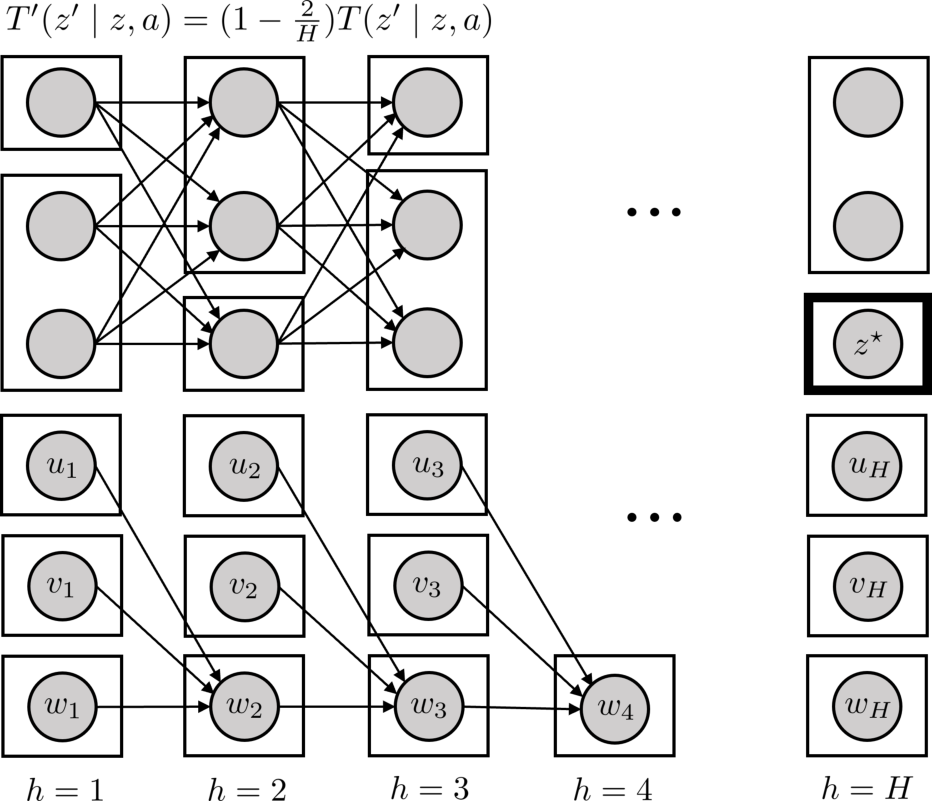} \label{fig: modified M} } \\ \vspace{50pt}   

\subfigure[Determine the reward and value functions in the layer \(H\). ]{\includegraphics[width=0.45\linewidth]{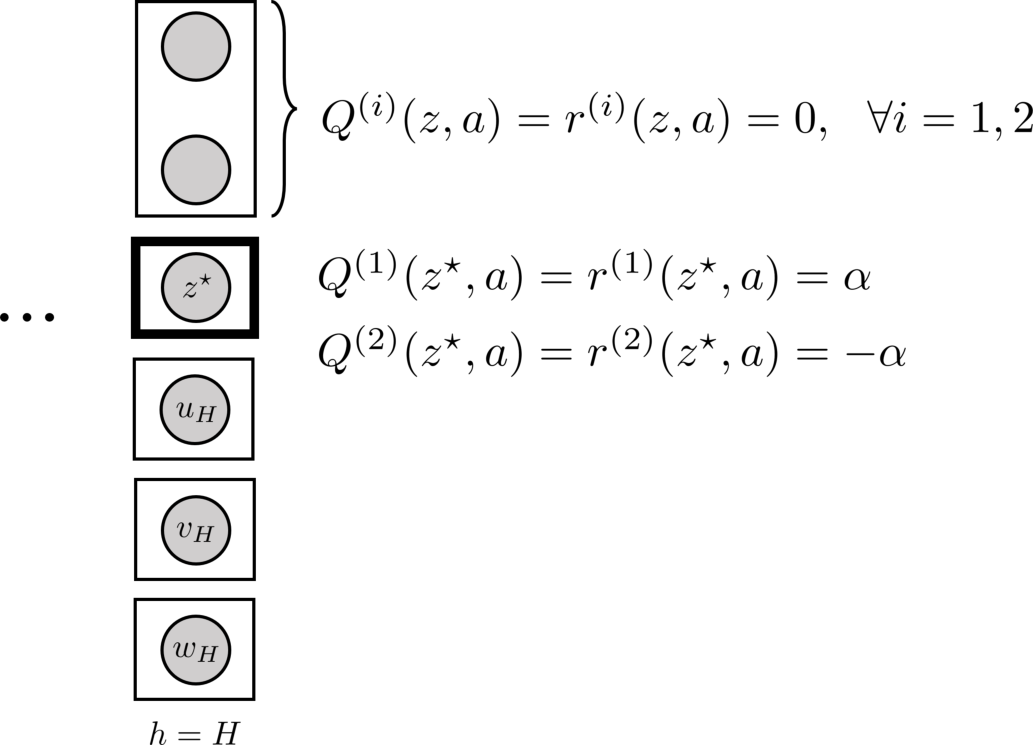} \label{fig: last layer}} 
\hfill \subfigure[Determine the value functions inductively assuming $Q\ind{i}$ are already determined on all layers $h'\geq h+1$. ] {\makebox[0.43\linewidth]{\includegraphics[width=0.26\linewidth]{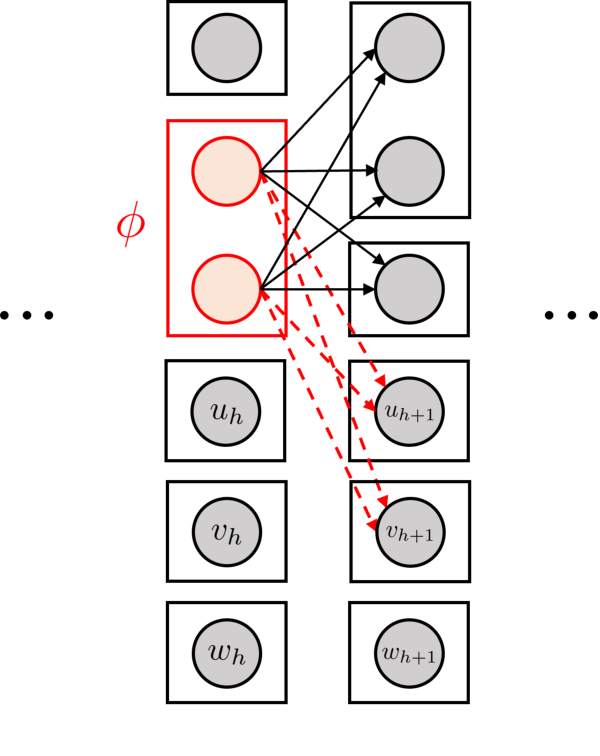} } \label{fig: back prob} }  \\

\caption{Lower bound construction used in the  proof sketch of \pref{thm: informal general data}. States are represented with circles and the corresponding state aggregations are represented with rectangles. We use the bold rectangle to denote the set of aggregations $\gI$ that attains the maximum in the definition of $\aggC_\epsilon(M,\Phi,\mu)$ (see \pref{def:agg-concentrability}). For simplicity, in the above figure \(\cI\) only contains a single aggregation that contains a single latent state $\sz^\star$, while in general $\cI$ may include multiple aggregations each with multiple latent states. } \label{fig: lower bound general} 
\end{center}
\end{figure}

\begin{enumerate}[leftmargin=14pt]
	\item \textit{Modified Markov Transition Model (\MTM) \(M'\)} (\pref{sec:latent}): We construct an \(\MTM\) $M'$ with state space \(\cZ'\) that comprises of the state space \(\cZ\) (corresponding to \(M\)) along with three additional states $\su_h, \sv_h, \sw_h$ on each layer \(h \in [H]\) (see \pref{fig: modified M}). The transition dynamics $T'$ in the \(M'\) is defined such that 
	\begin{enumerate}
		\item Each of $\su_h$, $\sv_h$, and $\sw_h$ deterministically transitions to $\sw_{h+1}$ under any action. 
		\item For any \(\sz_1,  \sz_2 \in \cZ\) and \(a \in \cA\),  \(T'(\sz_2 \mid \sz_1, a) = \prn*{1 - \nicefrac{2}{H}} T(\sz_2 \mid \sz_1, a).\) In particular, the probability of each transition from \(\cZ_h\) to \(\cZ_{h+1}\) is decreased by a  factor of \((1 - \nicefrac{2}{H})\). 
		\item The remaining  \(\nicefrac{2}{H}\) probability mass in \(T'\) is assigned to transitions from \(\cZ_h\) to \(u_{h+1}\) and \(v_{h+1}\). These transitions are different for \(M\ind{1}\) and \(M\ind{2}\), and will be specified later. 
	\end{enumerate} 
    Finally, we also define a new a modified aggregation scheme \(\Phi'\) that comprises of \(\Phi\) along with $3H$ more singleton aggregations, each consisting of $\su_h, \sv_h, \sw_h$ for \(h \in [H]\). 
    
	\item \textit{Reward functions} (\pref{sec:aggregate} and \pref{sec:latent}): 
	We create reward functions \(r\ind{1}\) and \(r\ind{2}\), for MDPs \(M\ind{1}\) and \(M\ind{2}\) respectively, such that non-zero rewards are only given to states in aggregation \(\cI\) and to \(\prn{\su_h, \sv_h, \sw_h}_{h \in [H]}\) (see \pref{fig: last layer}). In particular, we set 
	\begin{enumerate}[label=\(\bullet\)] 
	\item $r\ind{1}(\sz, \pieval(\sz))=\alpha$ and $r\ind{2}(\sz, \pieval(\sz))=-\alpha$ for any state \(\sz \in \gI\), for some properly chosen constant $\alpha$.  
	\item \(r\ind{i}(\su_h, a) = 1\), \(r\ind{i}(\sv_h, a) = -1\)  and \(r\ind{i}(\sw_h, a) = 0\) for any \(h \in [H]\) and \(a \in \cA\). 
	\item \(r\ind{i}(\sz, a) = 0\) for all other \(\sz \in \cZ\) and \(a \in \cA\). 
	\end{enumerate}
	\item \textit{Value functions and missing transitions} (\pref{sec:aggregate} and \pref{sec:latent}): We now proceed to the construction of state-action value functions  \(Q\ind{1}\) and \(Q\ind{2}\) for the evaluation policy \(\pieval\), and the transition probabilities \(T\ind{1}\) and \(T\ind{2}\), for \(M\ind{1}\) and \(M\ind{2}\) respectively. These quantities are constructed so as to ensure that: 
	\begin{enumerate}
	\item  All states that belong to the same aggregation have the same value in both $M\ind{1}$ and $M\ind{2}$, and are thus indistinguishable via the value functions, i.e. for any aggregation \(\phi \in \Phi'\), states \(\sz_1, \sz_2 \in \phi\), and \(a \in \cA\), 
\begin{align*}
	Q\ind{1}(\sz_1, \sa) = Q\ind{1}(\sz_2, \sa) \quad \text{and} \quad Q\ind{2}(\sz_1, \sa) = Q\ind{2}(\sz_2, \sa). \numberthis  \label{eq:Q_indistinguishable} 
\end{align*}
	\item From any aggregation, the probability of  transitioning to states \(\su_h\) (or to states \(\sv_h\))  is same between \(M\ind{1}\) and \(M\ind{2}\), i.e. for any \(\phi \in \Phi\) and \(h \in [H]\), 
	    \begin{align}
         &\forall \sa\in \gA, \quad \sum_{\sz\in\phi}\mu_h(\sz, \sa) T\ind{1}(\su_{h+1}\mid \sz, \sa) = \sum_{\sz\in\phi}\mu_h(\sz, \sa) T\ind{2}(\su_{h+1}\mid \sz, \sa), \nonumber  \\
         &\forall \sa\in \gA, \quad  \sum_{\sz\in\phi}\mu_h(\sz, \sa) T\ind{1}(\sv_{h+1}\mid \sz, \sa) = \sum_{\sz\in\phi}\mu_h(\sz, \sa) T\ind{2}(\sv_{h+1}\mid \sz, \sa).   \label{eq: indistinguishable}
    \end{align}
    \item For any \(\sz_1, \sz_2 \in \cZ\), we have \(T\ind{i}(\sz_2 \mid \sz_1, a) = T'(\sz_2 \mid \sz_1, a)\) for all \(a \in \cA\). 
	\end{enumerate} 
	Value functions and transitions that satisfy the above constraints are inductively constructed from time step \(h = H\) to \(1\). Since each of the above constraints is a linear equation, the corresponding solutions can be obtained by solving a system of linear equations. At a high level, the reason why we added $\su_{h+1}$ and $\sv_{h+1}$ --- 
    by splitting out some transition probabilities to $\su_{h+1}$ and $\sv_{h+1}$, and adjust their differences properly (notice that $Q\ind{i}(\su_{h+1}, \sa)=1$ and $Q\ind{i}(\sv_{h+1}, \sa)=-1$), we can calibrate the state-action values in $\phi$, making them all equal. 
    
    Jointly solving \pref{eq:Q_indistinguishable}, \pref{eq: indistinguishable}, and using the condition $T\ind{i}(\su_{h+1}\mid \sz, \sa) + T\ind{i}(\sv_{h+1}\mid \sz, \sa) =\frac{2}{H}$, we can obtain the following solution: 
    \begin{align}
        Q\ind{i}(\phi, \sa) = \sum_{\phi'\in\Phi_{h+1}} \agg{T}(\phi'\mid \phi, \sa) V\ind{i}(\phi'),  \label{eq: aggregation picture}
    \end{align}
    where $Q\ind{i}(\phi, \sa)$ is the value of $Q\ind{i}(\sz, \sa)$ shared by all $\sz\in\phi$, $\Phi_{h+1}$ is the set of aggregations on layer $h+1$, and $\agg{T}(\phi'\mid \phi, \sa) = \frac{\sum_{\sz'\in \phi'}\sum_{\sz\in \phi}\mu_h(\sz, \sa) T'(\sz'\mid\sz, \sa)}{\sum_{\sz\in \phi}\mu_h(\sz, \sa)}$ is the aggregated transition. This is where the aggregated transition comes into the picture. With \pref{eq: aggregation picture}, the argument that the aggregated transition plays a role in the sample complexity is similar to the argument in the tabular case as outlined in \pref{sec: general lower bound}. For formal proofs, see \pref{lem:sec-lower-bound-latent-MDP}\ref{lem:sec-lower-bound-latent-MDP-a}-\pref{lem:sec-lower-bound-latent-MDP}\ref{lem:sec-lower-bound-latent-MDP-c}  
    

	\item \textit{Construction of offline distribution \(\mu'\)} (\pref{sec:latent}): For any \(\sz  \neq \su_h, \sv_h, \sw_h\), we set  $\mu'_h(\sz, \sa) = \frac{1}{2}\mu_h(\sz, \sa)$. 
 Furthermore, for $\sz = \su_h, \sv_h, \sw_h$, we define $\mu'_h(\sz, \sa) = \frac{1}{6}$ 
 This construction ensures that both $\tabC$ and $\aggC$ remain unchanged up to constant factors in the original $M$ and in the modified $M\ind{1}$ and \(M\ind{2}\). See \pref{lem:sec-lower-bound-latent-MDP}\ref{lem:sec-lower-bound-latent-MDP-d} and \pref{corr:sec-lower-bound-con} for formal proofs.  
	

	\item \textit{Lifting to block MDPs} (\pref{sec:rich}): We finally lift $M\ind{1}$ and $M\ind{2}$ to block MDPs where every state $\sz$ serves as a \emph{latent state} invisible to the learner. Instead of observing the latent state $\sz$, the learner only observes a \emph{rich observation} from the set \(\cX\) corresponding to latent state $\sz$.  
\end{enumerate}
For the rest of this section, we provide a formal proof for \pref{thm: informal general data}. 

\subsection{Construction of Aggregated MDPs}  \label{sec:aggregate} 

We first construct two aggregated \MDP{}s\footnote{Throughout this proof, we use a "bar" over the variables, e.g.~in \(\bar{M}\), \(\bar{r}\), etc., to signify that they correspond to aggregated \MDP{s}.}  $\agg{M}\ind{1} = (\Phi, \calA, \agg{T}, \agg{r}\ind{1}, H, \agg{\rho})$ and $\agg{M}^{(2)} = (\Phi, \calA, \agg{T}, \agg{r}\ind{2}, H, \agg{\rho})$ of horizon \(H\), whose state space is \(\Phi\) and action space is \(\cA\). Furthermore, both of them have identical transition models and initial distributions given by: 
\begin{enumerate}[label=\(\bullet\)]
\item State Space \(\Phi\), and action space \(\cA\). 
\item \textbf{Transition model} $\agg{T}$ is defined as $\agg{T}(\phi\mid\phi, \sa)\ldef{}  \agg{T}(\phi\mid\phi, \sa; \agg{M})$, where $\agg{T}(\phi\mid\phi, \sa; \agg{M})$ is given in \pref{eq:agg-transition-model}.
\item \textbf{Initial distribution $\mb{\agg{\rho}}$} is defined as $\agg{\rho}(\phi)\ldef{}  \sum_{\sz\in \phi}\rho(\sz)$. 
\end{enumerate}

Suppose $h^*\in [H-1]$ and set $\Phiopt \subset \Phi_{h^*}$ attains the maximum in \pref{def:agg-concentrability}. The reward function of \(\agg{M}\ind{1}\) and \(\agg{M}\ind{2}\) are given by: 
\begin{enumerate}[label=\(\bullet\)]
\item \textbf{Reward function \(\agg{r}\ind{1}\) for \(\agg{M}\ind{1}\).}  We set the reward to be \(0\) for all states \(\phi \notin \gI\). Furthermore, for \(\phi \in \gI\), we set the reward to be \(0\) for actions that would not have been chosen by \(\pieval\). On the remaining \((\phi, a)\) tuples, we set a non-zero reward given by: 
   \begin{align*}
    \agg{r}\ind{1}(\phi, \sa) = \frac{\epsilon}{2H\sum_{\phi\in \Phiopt}d_{h^*}^\aggpieval(\phi; \agg{M})} \cdot \indic\{\phi\in \Phiopt, \sa = \pieval(\phi)\},
   \end{align*}
   where $h^*$ and $\Phiopt$ are the maximizer in \pref{def:agg-concentrability}. The key intuition in the above choice of reward function is to ensure that only those states-action contribute to non-zero rewards for which \(\phi \in \gI\) and \(a = \pieval(\phi)\); Hence, in order to receive a non-zero return, any agent in this \MDP needs to first find states in \(\gI\) and then play action given by \(\pieval\) on them. The denominator just consists of additional normalizing factors to ensure that the value is bounded by \(\nicefrac{\epsilon}{2H}\). 
\item \textbf{Reward function \(\agg{r}\ind{2}\) for \(\agg{M}\ind{2}\).} The reward function is similar to \(\agg{r}\ind{1}\), but with the negative sign. In particular, we define 
\begin{align*}
   \agg{r}\ind{2}(\phi, \sa) =  - \frac{\epsilon}{2H\sum_{\phi\in \Phiopt}d_{h^*}^\aggpieval(\phi; \agg{M})} \cdot \indic\{\phi\in \Phiopt, \sa = \pieval(\phi)\}.
\end{align*}
\end{enumerate} 

\begin{definition}[Aggregated MDP]\label{def:sec-lower-bound-aggMDP}
    We define aggregated MDPs $\agg{M}\ind{1} = (\Phi, \calA, \agg{T}, \agg{r}\ind{1}, H, \agg{\rho})$ and $\agg{M}^{(2)} = (\Phi, \calA, \agg{T}, \agg{r}\ind{2}, H, \agg{\rho})$ as follows: transition model $\agg{T}$ is defined as $\agg{T}(\phi\mid\phi, \sa)\ldef{}  \agg{T}(\phi\mid\phi, \sa; \agg{M})$, where $\agg{T}(\phi\mid\phi, \sa; \agg{M})$ is given in \pref{eq:agg-transition-model}, reward functions are defined as 
    \begin{align*}
    \agg{r}\ind{1}(\phi, \sa) = \frac{\epsilon}{2H\sum_{\phi\in \Phiopt}d_{h^*}^\aggpieval(\phi; \agg{M})} \cdot \indic\{\phi\in \Phiopt, \sa = \pieval(\phi)\}, \intertext{and,}
    \agg{r}\ind{2}(\phi, \sa) =  - \frac{\epsilon}{2H\sum_{\phi\in \Phiopt}d_{h^*}^\aggpieval(\phi; \agg{M})} \cdot \indic\{\phi\in \Phiopt, \sa = \pieval(\phi)\}.
    \end{align*}
\end{definition}

\begin{lemma}\label{lem:opposite}
    The value functions of $\bar{M}\ind{1}$ and $\bar{M}\ind{2}$ satisfies that for any policy $\pi$ and $\phi\in \Phi$,
    $$V^\pi(\phi; \bar{M}\ind{1}) = - V^\pi(\phi; \bar{M}\ind{@}).$$
\end{lemma}
\begin{proofof}[\pref{lem:opposite}]
    This lemma is easy to see after noticing that the transitions of $\bar{M}\ind{1}$ and $\bar{M}\ind{2}$ are the same, while the reward functions of $\bar{M}\ind{1}$ and $\bar{M}\ind{2}$ are of opposite signs.
\end{proofof}

\begin{lemma}
\label{lem:bounded_V}
    The value functions of $\bar{M}\ind{1}$ and $\bar{M}\ind{2}$ satisfies that
    $$V^\aggpieval(\agg{\rho}; \agg{M}\ind{1}) = \frac{\epsilon}{2H},\quad \text{and }V^\aggpieval(\agg{\rho}; \agg{M}\ind{2}) = -\frac{\epsilon}{2H}.$$
    Additionally, for each $\phi\in \Phi$,
    \begin{equation}\label{eq:sec-lower-bound-bound}
        0\le V^\aggpieval(\phi; \agg{M}\ind{1})\le \frac{1}{2H},\qquad  -\frac{1}{2H}\le V^\aggpieval(\phi; \agg{M}\ind{2})\le 0.
    \end{equation}
\end{lemma}

\begin{proofof}[\pref{lem:bounded_V}]
    According to \pref{lem:opposite}, we only need to prove results for $\bar{M}\ind{1}$. First of all, we can write the value functions as weighted averages of rewards with occupancy-measure-weights:
    \begin{equation}\begin{aligned}\label{eq:sec-lower-bound-value}
        V^\aggpieval(\agg{\rho}; \agg{M}\ind{1}) & = \sum_{\phi\in \gI} r\ind{1}(\phi, \pieval(\phi))d_{h^*}^\aggpieval(\phi; \agg{M}) = \frac{\epsilon}{2H}.
    \end{aligned}\end{equation}
    Next, We let $\Phiopt$ to be the set which attains the second maximum in \pref{eq: concentrability coeff} of the definition of $\tabC_\epsilon(M, \mu)$. Then $\Phiopt\subset\Phi_h$ for some $h\in[H-1]$, and it also satisfies $\sum_{\phi\in\Phiopt} d_{h^*}^\aggpieval(\phi; \agg{M})\ge \epsilon$, which implies that
    $$0\le \agg{r}\ind{1}(\phi, \sa)\le \frac{\epsilon}{2H\cdot \epsilon}\le \frac{1}{2H}, \quad \forall \phi\in \Phi_h, a\in\calA\qquad \text{and }\qquad \agg{r}\ind{1}(\phi, \sa) = 0,\quad \forall \phi\not\in\Phi_h, \sa\in\calA.$$
    Hence for any $\phi\in\Phi$, the sum of rewards along any trajectory which starts from $\phi$ is always between $0$ and $\nicefrac{1}{2H}$ in $\agg{M}\ind{1}$. Therefore we get $0\le V^\aggpieval(\phi; \agg{M}\ind{1})\le \nicefrac{1}{2H}$ for any $\phi\in\Phi$.
\end{proofof}

\subsection{Construction of Latent-State MDPs and \OPE Problems} \label{sec:latent}
Based on $\agg{M}\ind{1}$ and $\agg{M}\ind{2}$, we next construct two MDPs $M\ind{1}$ and $M\ind{2}$ which will be used as latent-state dynamics for rich observation MDPs that we construct in the next section. For \(i \in \crl{1, 2}\), we define 
\begin{align*}
M\ind{i} &= \text{MDP}(\calZ', \cA, T\ind{i}, r\ind{i}, H, \rho), 
\end{align*} 
where 
\begin{enumerate}[label=\(\bullet\)] 
\item \textbf{State space \(\mb{\cZ'}\)} is defined such that \(\cZ' = \cup_{h=1}^H\cZ'_h\) where, for each \(h \in [H]\), in addition to the states in \(\cZ_h\), the set \(\cZ_h'\) contains  three additional states \(\crl{\su_h, \sv_h, \sw_h}\) for all \(h \in [H]\). Formally,  \(\cZ'_h \ldef{} \cZ_h \cup \crl{\su_h, \sv_h, \sw_h}\).

The roles of \(u_h, v_h\) and \(w_h\) is to ensure that for every aggregated state \(\phi \in \Phi\), each state \(z \in \phi\) has the same value functione; How we achieve this will become clear later when we define the transition model \(T\ind{i}\). 

\item \textbf{Initial distribution \(\mb{\rho}\)}  is the same as the initial distribution in \(M\) (the original \MDP that was used in the construction of \(\agg{M}\ind{1}\) and \(\agg{M}\ind{2}\)). 

\item \textbf{Reward function \(r\ind{i}\)}  is set as 

\begin{align*}
r\ind{i}(\sz, \sa) & = \begin{cases}
    \agg{r}\ind{i}(\zeta(\sz), \sa) & \text{if}\quad \sz \in \calZ_h,\\   
    1 & \text{if}\quad \sz = \su_h,\\   
    - 1 & \text{if}\quad \sz = \sv_h,\\ 
    0 & \text{if}\quad \sz = \sw_h.\\ 
\end{cases} 
\end{align*}
for all \(h \in [H-1]\), where $\zeta(\sz)$ is defined in \pref{eq:def-zeta}. In particular, we use the same reward in \(M\ind{i}\) as in \(\agg{M}\ind{i}\) for the (old) states \(z \in \cZ\), and define new rewards for (newly added) states \(\su_h, \sv_h\) and \(\sw_h\). By definition,  the reward \(r\ind{i}(\sz', a) = r\ind{i}(\sz'', \sa)\) whenever \(\sz'\) and \(\sz''\) belong the same aggregated state \(\phi\), for all \(\sa \in \cA\). 

\item \textbf{Transition model \(\mb{T \ind{i}}\).} The transitions are defined such that \(T\ind{i}(\sz' \mid \sz, \sa)\) is proportional to \(T(\sz' \mid \sz, \sa)\) for tuples \((\sz, \sa, \sz') \in \cZ_h \times \cA \times \cZ_h\) that corresponds to transitions amongst (old) states that were also present in \(M\), and the remaining probability mass is redirected it to new states \(\crl{\su_h, \sv_h, \sw_h}_{h \in [H]}\). Formally, for $\sz\in \calZ_h$ and action $\sa\in\calA$, we set 

\begin{align*}
T\ind{i}(\sz'\mid \sz, \sa) =  \begin{cases}
	~\left(1 - \frac{2}{H}\right) T(\sz'\mid \sz, \sa) &\text{if} \quad \sz' \in \cZ_{h+1},\\ 
	~\Dti{i}(z)  + \frac{1}{H} &\text{if} \quad \sz' = \su_{h+1},\\ 
	~- \Dti{i}(z)  + \frac{1}{H} &\text{if} \quad \sz' = \sv_{h+1},\\ 
	~0 &\text{if} \quad \sz' = \sw_{h+1},
\end{cases}\numberthis \label{eq:sec-lower-bound-trans1}
\end{align*}
where we defined 
\begin{align*}
\Dti{i}(z) \ldef{} \frac{1}{2} \prn*{\sum_{\phi' \in \Phi_{h+1}} \agg{T}(\phi' \mid \zeta(\sz), \sa)V^\aggpieval(\phi'; \agg{M}\ind{i}) - \left(1 - \frac{2}{H}\right) \sum_{\sz' \in \cZ_{h+1}}  T(\sz'\mid \sz, \sa) V^\aggpieval(\zeta(\sz'); \agg{M}\ind{i})}. \numberthis \label{eq:sec-lower-bound-trans_inter}
\end{align*}

Furthermore, for any \(\sz \in \crl{\su_{h}, \sv_{w}, \sw_{h}}\) and \(\sa \in \cA\), $T\ind{i}$ transits to $w_{h+1}$ with probability $1$, i.e. 
\begin{align*}
T\ind{i}(\sz'\mid \sz, \sa) =  \begin{cases}
~ 1 &\text{if} \quad \sz' = \sw_{h+1} \\
~ 0 &\text{otherwise}  
\end{cases}  \numberthis \label{eq:sec-lower-bound-trans2}
\end{align*}  
Intuitively, the above transitions imply that states \(\crl{\sw_h}_{h \in [H-1]}\) act as terminal states. 
%
%
 \end{enumerate}

We further define distribution $\mu_h'$ over $\calZ_h'$ of layer $h\in [H-1]$ as follows:
\begin{equation}\label{def: sec-lower-mu-prime}
\mu'_h(\sz) = \begin{cases}
    \frac{\mu_h(\sz)}{2} &\text{if}\quad \sz\in\calZ_h,\\
    \frac{1}{6} &\text{if}\quad \sz\in \{\su_h, \sv_h, \sw_h\}. 
\end{cases}\end{equation}

Additionally, to be formal, we define $\piexp(\sz) = \pieval(\sz) \equiv a_0$ (an arbitrary fixed state in $\mathcal{A}$) for any $\sz\in \{\su_h, \sv_h, \sw_h\}$, and the \OPE problems $\mathfrak{g}\ind{1}$ and $\mathfrak{g}\ind{2}$ as 
\begin{equation}\label{eq: ope-original}
    \mathfrak{g}\ind{1} = \OPE(M\ind{1}, \pieval, \mu', \cF), \quad \mathfrak{g}\ind{2} = \OPE(M\ind{2}, \pieval, \mu', \cF),
\end{equation}
where $\cF\ldef{} \{Q\ind{1}, Q\ind{2}\}$ with $Q\ind{1}$ and $Q\ind{2}$ are state-action value functions of $M\ind{1}$ and $M\ind{2}$ under policy $\pieval$.
Before we proceed, we note the following technical lemma.

\begin{lemma}  \label{lem:valid}
For \(i \in \crl{1, 2}\), \(T\ind{i}\) constructed via \pref{eq:sec-lower-bound-trans1} and \pref{eq:sec-lower-bound-trans2} above is a valid transition model. 
\end{lemma}
\begin{proofof}[\pref{lem:valid}]
We first show that \(T\ind{i}(\sz' \mid \sz, \sa) > 0\) for all \(\sa \in \cA\) and \(\sz, \sz' \in \cZ'_{h+1}\). This is trivial for \(\sz' \in \cZ_h\) or when \(\sz' = \sw_{h+1}\). We next show that the same holds when \(\sz' \in \crl{\su_{h+1}, \sv_{h+1}}\). 

Note that due to \pref{lem:bounded_V}, we have that \(\abs{V^\aggpieval(\phi; \agg{M}\ind{i})} \leq \frac{1}{2H}\) for all \(\phi \in \Phi\). Plugging this in \pref{eq:sec-lower-bound-trans_inter}, and using Triangle inequality, we get that  
\begin{align*}
\abs*{\Dti{i}(z)} \leq  \sup_{\phi \in \Phi} \abs*{V^\aggpieval(\phi; \agg{M}\ind{i})} \leq \frac{1}{2H}. 
\end{align*} 
Using this in  \pref{eq:sec-lower-bound-trans1}, we immediately get that \(T\ind{i}(\sz' \mid \sz, \sa) > 0\) for any \(\epsilon \leq 1\). Furthermore, it is easy to check that for any \(\sz \in \cZ'_{h}\) and \(a \in \cA\), 
$\sum_{\sz' \in \cZ'_{h+1}} T\ind{i}(\sz' \mid \sz, \sa)  = 1.$ 
Thus, $T\ind{i}$ is a valid transition model. 
\end{proofof}

We note the following technical lemma, which will be used in the rest of the analysis. 
\begin{lemma} \label{lem:sec-lower-bound-latent-MDP}
We have the following properties of value functions, state-action value functions and occupancy measures of $M\ind{1}$ and $M\ind{2}$:
    \begin{enumerate}[label=\((\alph*)\)]
        \item \label{lem:sec-lower-bound-latent-MDP-a}  For \(i \in \crl{1, 2}\), and for any $\phi\in \Phi$, $\sz\in \phi$, and  $a\in \calA$, 
        $$Q^\pieval(\sz, \sa; M\ind{i}) = Q^\aggpieval(\zeta(\sz), \sa; \agg{M}\ind{i}).$$
        \item \label{lem:sec-lower-bound-latent-MDP-b} Corresponding to the initial distribution \(\rho\), the expected values  satisfy 
        $$V^\pieval(\rho; M\ind{1}) - V^\pieval(\rho; M\ind{2}) = \frac{\epsilon}{H}.$$ 
        \item \label{lem:sec-lower-bound-latent-MDP-c} For any $h\in [H-1]$, latent state $\phi\in \Phi_h$, action $\sa\in\calA$ and latent state $\sz'\in\calZ'$, 
        $$\sum_{\sz\in \phi}\mu(\sz\mid \sz\in \phi, \sa)T\ind{1}(\sz'\mid \sz, \sa) = \sum_{\sz\in \phi}\mu(\sz\mid \sz\in \phi, \sa)T\ind{2}(\sz'\mid \sz, \sa).$$
        \item \label{lem:sec-lower-bound-latent-MDP-d} For each \(i \in \crl{1, 2}\), for any $h\in [H-1]$ and $\sz\in\calZ_h$, and policy \(\pi\), 
        $$\frac{1}{16}d_h^\pi(\sz; M)\le d_h^\pi(\sz; M\ind{i})\le d_h^\pi(\sz; M).$$
    \end{enumerate}
\end{lemma} 
\begin{proofof}[\pref{lem:sec-lower-bound-latent-MDP}]
Observe that, by construction, we have for \(i \in \crl{1, 2}\), 
\begin{align*}
    Q\ind{i}(\su_h, \sa)  = 1, \qquad
    Q\ind{i}(\sv_h, \sa)  = -1,\text{and} \qquad 
    Q\ind{i}(\sw_h, \sa)  = 0, 
\end{align*} 
for any $h$ and  $\sa\in\calA$. 

\paragraph{Proof of \ref{lem:sec-lower-bound-latent-MDP-a}: } The proof follows via downward induction from \(h = H\) to \(h = 1\). First note that for any \(\sz \in \cZ_{H}\) and \(a \in \cA\), we have  $Q^\pieval(\sz, \sa; M\ind{i}) = 0 = Q^\aggpieval(\zeta(\sz), \sa; \agg{M}\ind{i})$; Thus the base case is satisfied. For the induction hypotesis, assume that the desired claim holds for \(h +1\). Thus, for layer $h$, using Bellman equation for the policy \(\pieval\), we have that  
\begin{align*}
    &\hspace{-0.3cm}Q^\pi(\sz, \sa; M\ind{i})\\
    & = r\ind{i}(\sz, \sa) + T\ind{i}(\su_{h+1}\mid \sz,\sa) - T\ind{i}(\sv_{h+1}\mid \sz,\sa) + \sum_{\sz' \in \cZ_{h+1}} T\ind{i}(\sz'\mid \sz,\sa)Q^\pieval(\sz', \pieval(\sz'); M\ind{i})\\ 
    & = \agg{r}\ind{i}(\zeta(\sz), \sa) + T\ind{i}(\su_{h+1}\mid \sz,\sa) - T\ind{i}(\sv_{h+1}\mid \sz,\sa) + \sum_{\sz' \in \cZ_{h+1}} T\ind{i}(\sz'\mid \sz,\sa)V^\aggpieval(\zeta(\sz'); \agg{M}\ind{i})\\
    & = \agg{r}\ind{i}(\zeta(\sz), \sa) + \sum_{\phi' \in \Phi_{h+1}} \agg{T}(\phi'\mid \zeta(\sz), \sa)V^\aggpieval(\phi'; \agg{M}\ind{i}) \\ &= Q^\aggpieval(\zeta(\sz), \sa; \agg{M}\ind{i}), 
\end{align*}
where in the last line, we used Bellman equation for policy \(\pieval\) under the \MDP $\agg{M}\ind{1}$. The above completes the induction step, thus showing that the claim holds for all \(h \in [H]\). 

\paragraph{Proof of \ref{lem:sec-lower-bound-latent-MDP-b}: } 
Using part-\ref{lem:sec-lower-bound-latent-MDP-a} above, we note that 
\begin{align*} 
    V^\pieval(\sz; M\ind{1}) = Q^\pieval(\sz, \pieval(\sz); M\ind{1}) = Q^\aggpieval(\zeta(\sz), \aggpieval(\zeta(\sz)); \agg{M}\ind{1}) = V^\aggpieval(\zeta(\sz); \agg{M}\ind{1}).
\end{align*}
Similarly, we also get that $V^\pieval(\sz; M\ind{2}) = V^\aggpieval(\zeta(\sz); \agg{M}\ind{2})$. The desired bound follows by noting that \pref{eq:sec-lower-bound-value} implies 
$$V^\aggpieval(\agg{\rho}; \agg{M}\ind{1}) - V^\aggpieval(\agg{\rho}; \agg{M}\ind{2}) = \frac{\epsilon}{H},$$
which implies 
$$V^\pieval(\rho; M\ind{1}) - V^\pieval(\rho; M\ind{2}) = \frac{\epsilon}{H}.$$

\paragraph{Proof of \ref{lem:sec-lower-bound-latent-MDP-c}: } First note that whenever $\sz'\not\in \{\su_h, \sv_h\}_{h=1}^H$, we always have $T\ind{1}(\sz'\mid \sz, \sa) = T\ind{2}(\sz'\mid\sz, \sa)$  according to the definition of $T\ind{1}$ and $T\ind{2}$. We next show that the same holds for $\sz' = \su_h$. We only need to verify $\sum_{\sz\in \phi} \mu(\sz, \sa)(T\ind{1}(\su_h\mid\sz, \sa) - T\ind{2}(\su_h\mid\sz, \sa)) = 0$. 
\begin{align*}
    \hspace{0.5in}&\hspace{-0.5in}\quad \sum_{\sz\in \phi} \mu(\sz, \sa)(T\ind{1}(\su_h\mid\sz, \sa) - T\ind{2}(\su_h\mid\sz, \sa))\\
    & = \sum_{\sz\in \phi} \mu(\sz, \sa)\prn*{\sum_{\phi'\in\Phi_{h+1}} \agg{T}(\phi'\mid \zeta(\sz), \sa)\agg{V}\ind{1}(\phi') - \sum_{\sz'\in\calZ_{h+1}} T\ind{1}(\sz'\mid \sz,\sa)\agg{V}\ind{1}(\zeta(\sz'))}\\
    & = \sum_{\sz\in \phi} \mu(\sz, \sa)\sum_{\phi'\in\Phi_{h+1}} \agg{T}(\phi'\mid \phi, \sa)\agg{V}\ind{1}(\phi') - \sum_{\phi'\in\Phi_{h+1}}\sum_{\sz'\in \phi'} \agg{V}\ind{1}(\zeta(\sz'))\sum_{\sz\in \phi} \mu(\sz, \sa)T\ind{1}(\sz'\mid \sz,\sa)\\
    & = \sum_{\sz\in \phi} \mu(\sz, \sa)\sum_{\phi'\in\Phi_{h+1}} \agg{T}(\phi'\mid \phi, \sa)\agg{V}\ind{1}(\phi') - \sum_{\phi'\in\Phi_{h+1}}\agg{V}\ind{1}(\phi')\sum_{\sz'\in \phi'} \sum_{\sz\in \phi} \mu(\sz, \sa)T\ind{1}(\sz'\mid \sz,\sa)\\
    & = \sum_{\sz\in \phi} \mu(\sz, \sa)\sum_{\phi'\in\Phi_{h+1}} \agg{T}(\phi'\mid \phi, \sa)\agg{V}\ind{1}(\phi') - \sum_{\phi'\in\Phi_{h+1}}\agg{V}\ind{1}(\phi')\sum_{\sz\in \phi} \mu(\sz, \sa)\agg{T}(\phi'\mid \phi, \sa) \\ 
    &= 0,
\end{align*}
where the first equation uses the definition of $T\ind{1}$ and $T\ind{2}$ in \pref{eq:sec-lower-bound-trans1} and \pref{lem:opposite}, and the last equation follows from the definition of $\agg{T}$ in \pref{eq:agg-transition-model}. Hence \ref{lem:sec-lower-bound-latent-MDP-c} is verified for $\sz' = \su_h$. The proof for $\sz' = \sv_h$ follows similarly.

\paragraph{Proof of \ref{lem:sec-lower-bound-latent-MDP-d}: } We only prove the result for $M\ind{1}$; the proof for $M\ind{2}$ follows similarly. In fact, we show a slightly stronger result that for all \(h \in [H]\) and $\sz\in\calZ_h$,
\begin{align*}
\left(\frac{H-2}{H}\right)^{h-1} d_h^\pi(\sz; M)\le d_h^\pi(\sz; M\ind{1})\le d_h^\pi(\sz; M). \numberthis \label{eq:induction_hypothesis1}
\end{align*}

The proof follows via induction over \(h\). For the base case, note that for \(h = 1\), by definition, we have $d_1^\pi(\sz; M) = \rho(\sz) = d_1^\pi(\sz; M\ind{1})$ for any $\sz\in\calZ_1$, which implies \pref{eq:induction_hypothesis1}. 

For the induction step, suppose \pref{eq:induction_hypothesis1} holds for a certain \(h \leq H-1\). For the upper bound, note that  any $\sz\in\calZ_{h+1}$ satisfies 
\begin{align*}
    d_{h+1}^\pi(\sz; M\ind{1}) = \sum_{\sz'\in\calZ_h}d_h^\pi(\sz'; M\ind{1})T\ind{1}(\sz\mid\sz', \pi(\sz')), 
\end{align*} 
which combined with the upper bound in \pref{eq:induction_hypothesis1} implies \begin{align*}
    d_{h+1}^\pi(\sz; M\ind{1}) & \le \sum_{\sz'\in\calZ_h}d_h^\pi(\sz'; M)T(\sz\mid\sz', \pi(\sz')) = d_{h+1}^\pi(\sz; M). 
\end{align*}

For the lower bound, recall from the definition of $T\ind{1}$, which implies that 
$$T\ind{1}(\sz\mid\sz', \pi(\sz')) = \frac{H-2}{H}T(\sz\mid\sz', \pi(\sz')).$$
Using the above with the lower bound in \pref{eq:induction_hypothesis1}, we get that 
\begin{align*}
    d_{h+1}^\pi(\sz; M\ind{1}) &= \sum_{\sz'\in\calZ_h'}d_h^\pi(\sz'; M\ind{1})T\ind{1}(\sz\mid\sz', \pi(\sz')) \\ 
    & \ge \left(\frac{H-2}{H}\right)^{h-1}\sum_{\sz'\in\calZ_h}d_h^\pi(\sz'; M)\cdot \frac{H-2}{H}T(\sz\mid\sz', \pi(\sz')) = \left(\frac{H-2}{H}\right)^{h}d_h^\pi(\sz; M). 
\end{align*}
The two bounds above imply the \pref{eq:induction_hypothesis1} also holds for \(h + 1\). This completes the induction step. 
\end{proofof}

This lemma has two direct corollaries:
\begin{corollary}\label{corr:sec-lower-bound-con}
    The concentratbility coefficients $\tabC(M\ind{1}, \mu')$ and $\tabC(M\ind{2}, \mu')$  of $M\ind{1}$ and $M\ind{2}$, respectively, satisfy that
    $$\tabC(M\ind{1}, \mu')\le 6\tabC(M, \mu, \pieval)\quad \text{and}\quad \tabC(M\ind{2}, \mu')\le 6\tabC(M, \mu, \pieval).$$
\end{corollary}
\begin{proofof}[\pref{corr:sec-lower-bound-con}]
    This corollary directly follows from \pref{def: sec-lower-mu-prime} and \pref{lem:sec-lower-bound-latent-MDP} \ref{lem:sec-lower-bound-latent-MDP-d}.
\end{proofof}
\begin{corollary}\label{corr:sec-lower-bound-d-pi}
    For any two policies $\pi$ and $\pi'$ and any $i\in \{1, 2\}$, we have
    $$\sup_{h\in [H]}\sup_{\sz\in\calZ_h'} \frac{d_h^\pi(\sz; M\ind{i})}{d_h^{\pi'}(\sz; M\ind{i})}\le 48\cdot \sup_{h\in [H]}\sup_{\sz\in\calZ_h} \frac{d_h^\pi(\sz; M)}{d_h^{\pi'}(\sz; M)}$$
\end{corollary}
\begin{proofof}[\pref{corr:sec-lower-bound-d-pi}]
    First of all, for those $\sz\in\calZ_h$ and $h\in [H-1]$, \pref{lem:sec-lower-bound-latent-MDP} \ref{lem:sec-lower-bound-latent-MDP-d} indicates that 
    $$\frac{d_h^\pi(\sz; M\ind{i})}{d_h^{\pi'}(\sz; M\ind{i})}\le 16\cdot \frac{d_h^\pi(\sz; M)}{d_h^{\pi'}(\sz; M)}\le 48\cdot \sup_{h\in [H]}\sup_{\sz\in\calZ_h} \frac{d_h^\pi(\sz; M)}{d_h^{\pi'}(\sz; M)}.$$
    Next, we verify cases where 
    $\sz\in \{\su_h, \sv_h, \sw_h\}_{h=1}^H$. Notice that the transition model of $M\ind{i}$ gives that
    $$d_h^\pi(\sw_h; M\ind{i}) = d_{h-1}^\pi(\su_h; M\ind{i}) + d_{h-1}^\pi(\sv_{h-1}; M\ind{i}),$$
    we only need to verify that for any $\sz\in \{\su_h, \sv_h\}_{h=1}^H$, we have
    $$\frac{d_h^\pi(\sz; M\ind{i})}{d_h^{\pi'}(\sz; M\ind{i})}\le 48\cdot \sup_{h\in [H]}\sup_{\sz\in\calZ_h} \frac{d_h^\pi(\sz; M)}{d_h^{\pi'}(\sz; M)}.$$
    Without loss of generality, we only verify for $\sz = \su_h$. We write
    $$d_h^\pi(\su_h; M\ind{i}) = \sum_{\sz\in\calZ_{h-1}}T\ind{i}(\su_h\mid \sz, \pi(\sz)) d_{h-1}^\pi(\sz).$$
    According to the transition model of $M\ind{i}$, we have for any $\sz\in\calZ_{h-1}$ and any $\sa\in\calA$, 
    $$T\ind{i}(\su_h\mid \sz, \sa)\in \left[\frac{1}{2H}, \frac{3}{2H}\right],$$
    which indicates that for any policy $\pi, \pi'$, we have
    $$\sup_{\sz\in\calZ_{h-1}}\frac{T\ind{i}(\su_h\mid \sz, \pi(\sz))}{T\ind{i}(\su_h\mid \sz, \pi'(\sz))}\in [1, 3].$$
    Therefore, we have
    $$\frac{d_h^\pi(\su_h; M\ind{i})}{d_h^{\pi'}(\su_h; M\ind{i})}\le 3\cdot \sup_{\sz\in\calZ_{h-1}}\frac{d_{h-1}^\pi(\sz; M\ind{i})}{d_{h-1}^{\pi'}(\sz; M\ind{i})}\le 48\cdot \sup_{h\in [H]}\sup_{\sz\in\calZ_h} \frac{d_h^\pi(\sz; M)}{d_h^{\pi'}(\sz; M)},$$
    where the last inequality follows from \pref{lem:sec-lower-bound-latent-MDP} \ref{lem:sec-lower-bound-latent-MDP-d}.
\end{proofof} 

\subsection{Construction of the Class \(\famGg\) of Offline Policy Evaluation Problems}\label{sec:rich} 
In this section, we construct the class \(\mathfrak{G}\) of \OPE problems that are used in  \pref{thm: informal general data}. The corresponding MDPs in \(\mathfrak{G}\)  are block MDPs based on \(M\ind{1}\)  and \(M\ind{2}\) (constructed in the previous section), and certain decoder functions.  The organization of this section is as follows: 
\begin{enumerate}[label=\(\bullet\)] 
    \item In \pref{app:block_MDP}, we provide a general procedure to lift an \OPE problem $\OPE(M, \pieval, \mu, \cF)$ over state space \(\cZ\) into a block \OPE problem \((\check M, \ripieval, \rich{\mu}, \rich{\cF})\) with rich-observations in a set \(\cX\) and latent states in \(\cZ\), given a decoder function \(\psi\).  
    \item Then, in \pref{app:family_construction}, we first provide a class \(\Psi\) of decoder function and use the above procedure to construct the family $\mathfrak{G}$ of offline RL problems. 
\end{enumerate}

\subsubsection{Lifting from \OPE Problems to Block \OPE Problems} \label{app:block_MDP}
In this section, we will discuss how to lift a normal \OPE problem $\OPE(M, \pieval, \mu, \cF)$ ($\cF$ satisfies $Q$-realizability) into a block\footnote{Throughout this section, we use a "check" over the variables, e.g.~in \(\check{M}\), \(\check{r}\), etc., to signify that they correspond to block \MDP{s}.} 
\OPE problem $\OPE(\rich{M}, \ripieval, \rich{\mu}, \rich{\cF})$ where $\rich{\cF}$ satisfies $(Q, W)$-realizability. 

Let \(\cZ = \bigcup_{h \in [H-1]} \cZ_h\) be the state space of \MDP $M$, and we fix \(\cX = \crl{\cX(\sz)}_{\sz \in \cZ}\) to be a family of disjoint sets that denote rich-observations corresponding to latent states \(\sz \in \cZ\). Furthermore, let \(\cX_h = \crl{\cX(\sz)}_{\sz \in \cZ_h}\). Then, for  $M = \MDP(\calZ, \calA, T, r, H, \rho)$, we can  define a block \MDP  
\begin{align*}
\rich{M} = \text{MDP}(\calX, \calA, \rich{T},  \rich{R}, H, \rich{\rho}) 
\end{align*}
with latent state \(\cZ\) and rich observations in \(\cX\), where 
\begin{itemize}
\item \textbf{State space}: \(\cX\) consists of rich-observations corresponding to latent states. We assume that the state space \(\cX = \bigcup_{h \in [H-1]} \cX_h\) is layered. 

\item \textbf{Transition Model} depends only on the latent transition model \(T\). In particular, for any \(h \in [H-1]\), \(\sx \in \cX(\sz)\) corresponding to \(\sz \in \cZ_h\), and \(\sx' \in \cX(\sz')\) corresponding to \(\sz' \in \cZ_{h+1}\), we have 
\begin{align*}
\rich{T}(\sx'\mid\sx, \sa) \ldef{}  \frac{1}{|\calX(\sz')|} \cdot T(\sz'\mid \sz, \sa). 
\end{align*}
\item \textbf{Rewards}: For any \(\sz \in \cZ\) and \(\sx \in \cX\), and \(\sa \in \cA\), $\rich{R}(\sx, \sa)$ is a $\{-1, 1\}$-valued  random variable with expected value $r(\sz, \sa)$.
\item \textbf{Initial distribution}: \(\rich{\rho} \in \Delta(\cX_1)\) is defined such that for any \(x \in \cX_1\),  
\begin{equation}\label{eq:rich-mu}
    \rich{\rho}(\sx)\ldef{}  \frac{1}{|\calX(\sz)|} \cdot \rho(\sz)
\end{equation}
where \(\sz\) is such that $\sx\in\calX(\sz)$. 
\end{itemize}

In particular, corresponding to a latent state \(\sz\), the observations are sampled from \(\text{Uniform}(\cX(z))\). In order to construct respective offline RL problems on the \MDP \(\rich{M}\), we also lift the offline data distribution \(\mu\) to \(\rich{\mu}\), and offline policy \(\pieval\) to \(\ripieval\) as follows: 
\begin{itemize} 
\item \textbf{Offline distribution} For any \(h \in [H-1]\), we define   \(\rich{\mu}_h \in \Delta(\cX_h \times \cA)\) such that for any \(x \in \cX_h\) and \(a \in \cA\), 
 \begin{align*}
\rich{\mu}_h(\sx, \sa) \ldef{}  \frac{1}{|\calX(\sz)|} \cdot \mu_h(\sz, \sa),  
\end{align*}
where \(\sz \in \cZ_h\) is such that \(\sx \in \cX(\sz)\). 
\item \textbf{Evaluation policy} \(\ripieval : \cX \mapsto \cA\) is defined such that for any \(\sx \in \cX\), $\ripieval(\sx) \ldef{}  \pieval(\sz)$ where \(\sz\) is such that $\sx\in\calX(\sz)$.
\item \textbf{Function class} $\rich{\cF}$ consists of tuples $(\rich{f}, \rich{W})$, where each $f\in\cF$ generates a tuple $(\rich{f}, \rich{W})$ with $\rich{f}: \calX\times\calA \to \mathbb{R}$ defined as $\rich{f}(\sx, \sa) = f(\sz, \sa)$ for any $\sx\in\calX(\sz)$ and $\rich{W}: [H-1]\to \mathbb{R}$ is defined as $\rich{W}[H-1] = \sum_{\sx\in\calX_h}\rich{\mu}_h(\sx, \ripieval(\sx))\rich{f}(\sx, \ripieval(\sa))$.
\end{itemize}

The following lemma indicates that $\rich{f}$ is the state-action value function of block \MDP $\rich{M}$ as long as $f$ is the state-action value function of \MDP $M$. 

\begin{lemma}\label{lem:sec-lower-bound-Q-1} For any $h\in[H]$, $\sz\in\calZ_h$, \(x \in \cX(\sz)\) and \(a \in \cA\), we have 
$$Q_h^{\ripieval}(\sx, \sa;\rich{M}) = Q_h^\pieval(\sz, \sa; M).$$
\end{lemma}
\begin{proofof}[\pref{lem:sec-lower-bound-Q-1}]  
    We prove this equation by induction from $h = H$ to $h = 1$. When $h = H$, we have $Q_H^{\ripieval}(\sx, \sa;\rich{M}) = Q_H^\pieval(\sz, \sa; M)$ for any $\sz\in\calZ_H$, $\sz\in\calX(\sz)$ and $\sa\in\calA$. Next, suppose $Q_{h+1}^{\ripieval}(\sx, \sa;\rich{M}) = Q_{h+1}^\pieval(\sz, \sa; M)$ holds for $\sz\in\calZ_{h+1}$ and $\sx\in\calX(\sz)$.
    According to the Bellman equation and definitions of $\rich{T}, \rich{R}$, for any $\sz\in\calZ_h$ and $\sx\in\calX(\sz)$, we have
    \begin{align*}
        Q_h^\ripieval(\sx, \sa; \rich{M}) & = \mathbb{E}[\rich{R}(\sx, \sa)] + \sum_{\sx'\in \calX_{h+1}}\rich{T}(\sx'\mid\sx, \sa) Q_{h+1}^\ripieval(\sx', \ripieval(\sx'); \rich{M})\\
        & = r(\sz, \sa) + \sum_{\sz'\in\calZ_{h+1}}Q_{h+1}^\pieval(\sz', \pieval(\sz'))\sum_{\sx'\in \calX(\sz')}\rich{T}(\sx'\mid\sx, \sa)\\
        & = r(\sz, \sa) + \sum_{\sz'\in\calZ_{h+1}}Q_{h+1}^\pieval(\sz', \pieval(\sz'); M)\sum_{\sz'\in \phi'}T(\sz'\mid\sz, \sa) \\
        &= Q_h^\pieval(\sz, \sa; M), 
    \end{align*}
    where we use the induction hypothesis and the fact that $\ripieval(\sx') = \pieval(\sz)$ for $\sx'\in\calX(\sz')$ in the second equation.
    
\end{proofof} 
    
\begin{lemma}\label{lem:sec-lower-bound-eq-h} 
	For any $h \in [H-1]$, 
        $$\sum_{\sx\in\calX_h}\rich{\mu}_{h}(\sx, \ripieval(\sx))Q_h^\ripieval(\sx, \ripieval(\sx); \rich{M}) = \sum_{\sz\in\calZ_h}\mu_{h}(\sz, \pieval(\sz))Q_h^\pieval(\sz, \pieval(\sx); M).$$
\end{lemma}
\begin{proofof}[\pref{lem:sec-lower-bound-eq-h}]
	We first notice that $\sx\in\calX(\sz)$, we have $\ripieval(\sx) = \pieval(\sz)$. Hence according to \pref{lem:sec-lower-bound-Q-1} we only need to verify that, for any $\sz\in\calZ$,
    $$\sum_{\sx\in\calX(\sz)}\rich{\mu}_{h}(\sx, \ripieval(\sx)) = \mu_{h}(\sz, \pieval(\sz)),$$
    which is given by definition $\rich{\mu}_h(\sx, \sa) \ldef{} \frac{\mu_h(\sz, \sa)}{|\calX(\sz)|}$. 
\end{proofof}

\begin{lemma}\label{lem:sec-lower-bound-rich-d}
	For any policy $\pi$ over \MDP $M$, let policy $\rich{\pi}$ over $\rich{M}$ to be $\rich{\pi}(\sx) = \pi(\sz)$ for any $\sx\in\calX(\sz)$. Then we have for any $h\in[H]$, $\sz\in\calZ_h$ and $\sx\in\calX(\sz)$,
        \begin{equation}\label{eq:sec-lower-bound-Q-3}
            d_h^{\rich{\pi}}(\sx; \rich{M}) = \frac{d_h^\pi(\sz; M)}{|\calX(\sz)|}.
        \end{equation}
\end{lemma}
\begin{proofof}[\pref{lem:sec-lower-bound-rich-d}]
We will prove via induction on the layer of $\sx$. For $\sx\in\calX_1$, \pref{eq:sec-lower-bound-Q-3} holds acording to the definition of initial distribution $\rich{\rho}$. The induction from layer $h$ to layer $h+1$ can be achieved by
    $$d_{h+1}^{\rich{\pi}}(\sx; \rich{M}) = \sum_{\sx'\in\calX_h} d_h^{\rich{\pi}}(\sx'; \rich{M})\rich{T}(\sx\mid\sx', \rich{\pi}(\sx')) =\sum_{\sz'\in\calZ_h} d_h^\pi(\sz'; M)\frac{T(\sz\mid\sz', \pi(\sz'))}{|\calX(\sz)|} = \frac{d_h^\pi(\sz; M)}{|\calX(\sz)|}$$
    for any $\sx\in\calX(\sz)$ and $\sz\in\calZ_{h+1}$.	
\end{proofof}

The above lemma has the following two corollaries: 
\begin{corollary}\label{corr:sec-lower-bound-agg-con}
    The concentrability coefficient $\tabC(\rich{M}, \mu, \ripieval) = \tabC(M, \mu, \pieval)$. 
\end{corollary}
\begin{proofof}[\pref{corr:sec-lower-bound-agg-con}]
    According to \pref{lem:sec-lower-bound-rich-d}, we have for any $h\in [H-1]$, $\sz\in\calZ_h$ and $\sx\in\calX(\sz)$,
    $$\frac{d_h^{\ripieval}(\sx; \rich{M})}{\rich{\mu}(\sx)} = \frac{\nicefrac{d_h^\pieval(\sz; M)}{|\calX(\sz)|}}{\nicefrac{\mu(\sz)}{|\calX(\sz)|}} = \frac{d_h^\pieval(\sz; M)}{\mu(\sz)}.$$
    Taking the supremum over all $h\in [H-1]$, $\sz\in\calZ_h$ and $\sx\in\calX(\sz)$, we get 
    $$\tabC(\rich{M}, \mu, \ripieval) = \sup_{h\in[H-1]}\sup_{\sx\in\calX_h}\frac{d_h^{\ripieval}(\sx; \rich{M})}{\rich{\mu}(\sx)} = \sup_{h\in[H-1]}\sup_{\sz\in\calZ_h} \frac{d_h^\pieval(\sz; M)}{\mu(\sz)} = \tabC(M, \mu, \pieval).$$
\end{proofof}
\begin{corollary}\label{corr:sec-lower-bound-agg-d}
    For any two policies $\pi$ and $\pi'$, let policy $\rich{\pi}$ and $\rich{\pi}'$ over $\rich{M}$ be defined such that $\rich{\pi}(\sx) = \pi(\sz)$ and $\rich{\pi}'(\sx) = \pi'(\sz)$ for any $\sx\in\calX(\sz)$. Then, we have
    $$\sup_{h\in[H]}\sup_{\sx\in\calX_h}\frac{d_h^{\rich{\pi}}(\sx; \rich{M})}{d_h^{\rich{\pi}'}(\sx; \rich{M})} = \sup_{h\in[H]}\sup_{\sz\in\calZ_h} \frac{d_h^\pi(\sz; M)}{d_h^{\pi'}(\sz; M)}.$$
\end{corollary}
\begin{proofof}[\pref{corr:sec-lower-bound-agg-d}]
    According to \pref{lem:sec-lower-bound-rich-d}, we have for any $h\in [H]$, $\sz\in\calZ_h$ and $\sx\in\calX(\sz)$,
    $$\frac{d_h^{\rich{\pi}}(\sx; \rich{M})}{d_h^{\rich{\pi}'}(\sx; \rich{M})} = \frac{\nicefrac{d_h^\pi(\sz; M)}{|\calX(\sz)|}}{\nicefrac{d_h^{\pi'}(\sz; M)}{|\calX(\sz)|}} = \frac{d_h^\pi(\sz; M)}{d_h^{\pi'}(\sz; M)}.$$
    Taking the supremum over all $h\in [H]$, $\sz\in\calZ_h$ and $\sx\in\calX(\sz)$, we get 
    $$\sup_{h\in[H]}\sup_{\sx\in\calX_h}\frac{d_h^{\rich{\pi}}(\sx; \rich{M})}{d_h^{\rich{\pi}'}(\sx; \rich{M})} = \sup_{h\in[H]}\sup_{\sz\in\calZ_h} \frac{d_h^\pi(\sz; M)}{d_h^{\pi'}(\sz; M)}.$$
\end{proofof}


\subsubsection{Construction of the family of offline RL problems} \label{app:family_construction}
We will construct two \OPE families $\mathfrak{G}\ind{1}$ and $\mathfrak{G}\ind{2}$ by lifting \OPE problems $\mathfrak{g}\ind{1} = \OPE(M\ind{1}, \pieval, \mu', \cF)$ and $\mathfrak{g}\ind{2} = \OPE(M\ind{2}, \pieval, \mu', \cF)$ defined in \pref{eq: ope-original} into Block \OPE problems. Each Block \OPE problem has the same observation space but a different emission distributions. Furthermore, each of these Block \OPE problem has latent state space \(\cZ\) and is based on the same aggregation scheme \(\Phi\).

let \(\crl{\cX(\phi)}_{\phi \in \cZ}\) be a family of disjoint sets that denote rich-observations corresponding to aggregated states states \(\phi \in \Phi\) such that 
\begin{equation}
    |\calX(\phi)|\gtrsim \frac{|\phi|^3H^8\cdot \sup_{h\in H}|\Phi_h|\cdot \sup_{\phi\in \Phi} |\phi|\aggC_\epsilon(M,\Phi,\mu)^3}{\epsilon^3}.\label{eq:rich_count} 
\end{equation}

The observation space for all the Block-MDPs is given by $\calX = \cup_{h=1}^H \calX_h$, where 
\begin{align*}
\cX_h = {\crl{\su_h, \sv_h, \sw_h}} \cup \prn*{\cup_{\phi \in \Phi_h} \cX(\phi)}.    
\end{align*}
The Block-MDPs that we construct next will different in terms of which observations from \(\cX(\phi)\) will be assigned to latent states \(\sz \in \phi\). To make this explicit, we rely on decoder functions that map 
\(\psi: \cX \mapsto \cZ\). Without loss of generality, assume that all decoders that we will consider satisfy \(\psi(\su_h) = \su_h\),  \(\psi(\sv_h) = \sv_h\)  and  \(\psi(\sw_h) = \sw_h\) for all \(h \in [H]\). Additionally, given a decoder \(\psi\), we define the set \(\cX_\psi(\sz) = \crl{\sx \in \cX \mid \psi(\sx) = \sz}\). We finally define the set  \(\Psi\)  as the set of all possible decoders which ensure that for any \(\phi \in \Phi\), each latent state \(\sz \in \cZ\) gets the same number of observations from \(\cZ\). In particular,  
\begin{align*}
\Psi = \crl*{\psi: \cX \mapsto \cZ \mid{} \cX_\psi(\sz) \subseteq \cX ~\text{and}~ \abs{\cX_\psi(\sz)} = \frac{\abs{\cX(\phi)}}{\abs{\phi}} \quad \forall \phi \in \Phi, \forall \sz \in \phi}.  
\end{align*}

\paragraph{Offline Policy Evaluation (OPE) Problem given \(\psi\).} 
Given a decoder \(\psi\), and the above notation, we will lift \OPE problem $\mathfrak{g}\ind{1} = \OPE(M\ind{1}, \pieval, \mu', \cF)$ and $\mathfrak{g}\ind{2} = \OPE(M\ind{2}, \pieval, \mu', \cF)$ into \OPE problems $\mathfrak{g}\ind{1}_\psi = \OPE(\rich{M}\ind{1}_\psi, \ripieval, \rich{\mu}_\psi', \rich{\cF}_\psi)$ and $\mathfrak{g}\ind{2}_\psi =\OPE(\rich{M}\ind{2}_\psi, \ripieval, \rich{\mu}_\psi', \rich{\cF}_\psi)$ using the recipe in \pref{app:block_MDP}, with $\rich{\cF}_\psi$ satisfies 
$$(Q^\ripieval(\cdot; \rich{M}\ind{1}_\psi), W^\ripieval(\cdot; \rich{\mu}_\psi', \rich{M}\ind{1}_\psi))\in\rich{\cF}_\psi\quad\text{and}\quad (Q^\ripieval(\cdot; \rich{M}\ind{2}_\psi), W^\ripieval(\cdot; \rich{\mu}_\psi', \rich{M}\ind{2}_\psi))\in\rich{\cF}_\psi,$$
where the $W$ function is defined in \pref{def: W-function}.

\paragraph{Family of \OPE problems.} We finally define the family \(\mathfrak{G}\)  for \OPE problems as 
\begin{align*}
\mathfrak{G} = \bigcup_{\psi \in \Psi} \crl{\mathfrak{g}\ind{1}_\psi,\mathfrak{g}\ind{2}_\psi}. 
\end{align*}

We note the following useful technical lemma. 
\begin{lemma}\label{lem:sec-lower-bound-equalF} For any \(\psi, \psi' \in \Psi\), we have that \(\rich{\cF}_\psi = \rich{\cF}_{\psi'}\). 
\end{lemma}
\begin{proofof}[\pref{lem:sec-lower-bound-equalF}]
	We will only prove the results for $\rich{M}_\psi\ind{1}$. To verify this, we only need to prove that for any $\psi, \psi'\in\Psi$, we have for any $1\le h\le H$,
    $$Q_h^\ripieval(\cdot; \rich{M}_\psi\ind{1}) = Q_h^\ripieval(\cdot; \rich{M}_{\psi'}\ind{1})\quad \text{and}\quad W^\ripieval(h; \rich \mu_\psi', \rich M \ind{1}_\psi)) = W^\ripieval(h; \rich \mu_\psi', \rich M \ind{1}_{\psi'})).$$
    The second equation directly follows from \pref{lem:sec-lower-bound-eq-h}. In the next, we will verify the first equation. \pref{lem:sec-lower-bound-Q-1} gives that for any $\phi\in\Phi_h$ and $\sx\in\calX(\phi)$ and $\sa\in\calA$, we have
    $$Q_h^\ripieval(\sx, \sa; \rich{M}_\psi\ind{1}) = Q_h^\pieval(\psi(\sx), \sa; M\ind{1}).$$
    Next, \pref{lem:sec-lower-bound-latent-MDP} \ref{lem:sec-lower-bound-latent-MDP-a} indicates that
    $$Q_h^\pieval(\psi(\sx), \sa; M\ind{1}) = Q_h^\aggpieval(\zeta(\psi(\sx)), \sa; \agg{M}\ind{1}).$$
    Notice that for every $\sx\in\calX(\phi)$, we have $\zeta(\psi(\sx)) = \phi$ for any $\psi\in\Psi$. Hence for any $\psi\in\Psi$, we have $Q_h^\ripieval(\sx, \sa; \rich{M}_\psi\ind{1}) = Q_h^\aggpieval(\phi, \sa; \agg{M}\ind{1})$ for every $\phi\in\Phi$ and $\sx\in\calX(\phi)$, which is independent to $\psi$. 
\end{proofof}

In the following, we let $\cF = \rich{\cF}_\psi$ for some $\psi\in\Psi$. Then we have for any $\psi'\in\Psi$, $\cF = \rich{\cF}_{\psi'}$ as well.

%
%

\subsection{Proof of \pref{thm: informal general data}}
After constructing the class $\mathfrak{G}$, we have finished the construction step. In the rest of the section we will prove the theorem by analyzing properties of \OPE problems in $\mathfrak{G}$. In fact, we will prove the following stronger results than \pref{thm: informal general data}:
\begin{theorem}\label{thm:stronger}
    Class $\mathfrak{G}$ satisfies: for every \OPE problem \(\opeg = \OPE\prn*{M\ind{\opeg}, \pieval\ind{\opeg}, \mu\ind{\opeg}, \cF\ind{\opeg}}\) in \(\famGg\),
    \begin{enumerate}[label=(\alph*)] 
        \item The function class \(\cF\ind{\opeg}\) satisfies \(\abs{\cF\ind{\opeg}} = 2\) and 
        $$(Q^{\pieval\ind{\opeg}}(\cdot ; M\ind{\opeg}),  W^{\pieval\ind{\opeg}}(\cdot; \mu\ind{\opeg}; M\ind{\opeg}))\in \cF\ind{\opeg}.$$ 
        \item The concentrability coefficients satisfy $\tabC(M\ind{\opeg}, \mu\ind{\opeg}, \pieval\ind{\opeg}) \le 6 \tabC(M, \mu, \pieval)$;
    \end{enumerate}
    Furthermore, for $N = o(\nicefrac{H\aggC_\epsilon(M,\Phi,\mu)}{\epsilon})$ and any algorithm which takes $\mathcal{D} = \cup_{h=1}^H\mathcal{D}_h$ as input and output the evaluation of value function $\widehat{V}$, there must exist some $\mathfrak{g} = \OPE(M\ind{\opeg}, \pieval\ind{\opeg}, \mu\ind{\opeg}, \cF\ind{\opeg})\in\mathfrak{G}$ such that the algorithm fails to output $\nicefrac{\epsilon}{H}$-accurate evaluation with probability at least $\nicefrac{1}{2}$, if the dataset $\mathcal{D}_h = \{(\sx_h, \sa_h, r_h, \sx_{h+1}')\}$ consists of $N$ i.i.d. samples collected where $(\sx_h, \sa_h)\sim \mu\ind{\opeg}$ and $r_h, \sx_{h+1}'$ are collected according to the reward function and transition model of $M\ind{\opeg}$.
\end{theorem}

\begin{proposition}\label{prop: rich aggregated concentrability}
    The aggregated concentrability coefficient of \OPE problems $\mathfrak{g} = \OPE(\rich{M}, \ripieval, \rich{\mu}', \mathcal{F})\in\mathfrak{G}$ satisfies $\ophatC{\rich{M}}{\rich{\Phi}}{\rich{\mu}'}{\epsilon} = \Theta(\ophatC{M}{\Phi}{\mu}{\epsilon})$ where the  aggregation scheme $\rich{\Phi}$ is defined over $\calX$ such that $\rich{\Phi}_h = \{\calX(\phi): \phi\in\Phi_h\}\cup \{\{\su_h\}, \{\sv_h\}, \{\sw_{h}\}\}$.
\end{proposition}

\subsubsection{Useful Technical Lemmas for the Proof of \pref{thm:stronger}} 
In this subsection we provide several useful technical lemmas for proof of \pref{thm:stronger}.

To begin with, we denote 
$$\rich{\mu}_\psi(\sx)\ldef{} \begin{cases}
    \nicefrac{\mu(\sz)}{|\calX_\psi(\sz)|} &\quad \forall \sz\in\calZ, \sx\in\calX_\psi(\sz),\\
    0 &\quad \forall \sx\in \{\su_h, \sv_h, \sw_h\}_{h=1}^H
\end{cases}.$$
Since the transition of $\su_h, \sv_h, \sw_h$ are already known, and $\rich{\mu}_\psi(\sx)\ge \rich{\mu}'_\psi(\sx)$ for all $\sx\in \calX_{\psi}(\sz)$ and $\sz\in\calZ$, in the following we only need to prove the results for \OPE problems $\mathfrak{g} = \OPE(\rich{M}, \ripieval, \rich{\mu}, \mathcal{F})$.

In the following, we use $\mathbb{P}_{h, n}(\cdot; \rich{\mu}, \rich{M})$ where $\rich{M} = \MDP(\calX, \calA, H, \rich{T}, \rich{R}, \rich{\rho})$ to denote the law of $n$ tuples of $(\sx, \sa, r, \sx')$ jointly, where each tuple is i.i.d.~ collecting as follows: first sample $(\sx, \sa)\sim \rich{\mu}_{h}$, then sample $r\sim \rich{R}(\cdot\mid\sx, \sa), \sx'\sim \rich{T}(\sx'\mid \sx, \sa)$. Let  
\begin{align*}
\mathbb{P}_{h, n}\ind{1} = \frac{1}{|\Psi|}\sum_{\psi\in \Psi} \mathbb{P}_{h, n}(\cdot; \rich{\mu}_\psi, \rich{M}_\psi\ind{1}) \quad\text{and}\quad \mathbb{P}_{h, n}\ind{2} = \frac{1}{|\Psi|}\sum_{\psi\in \Psi} \mathbb{P}_{h, n}(\cdot; \rich{\mu}_\psi, \rich{M}_\psi\ind{2}). \numberthis \label{eq:P_n_defhn}
\end{align*}
Furthermore, for any \(\rich{M}\), let 
$\mathbb{P}_{n}(\cdot; \rich{\mu}, \rich{M}) \ldef{} \bigotimes_{h=1}^{H-1} \mathbb{P}_{h, n}(\cdot; \rich{\mu}, \rich{M})$, and using this notation, define 
\begin{align*}
    \mathbb{P}_{n}\ind{1} = \frac{1}{|\Psi|}\sum_{\psi\in \Psi} \mathbb{P}_{n}(\cdot; \rich{\mu}_\psi, \rich{M}_\psi\ind{1}) \quad\text{and}\quad \mathbb{P}_{n}\ind{2} = \frac{1}{|\Psi|}\sum_{\psi\in \Psi} \mathbb{P}_{n}(\cdot; \rich{\mu}_\psi, \rich{M}_\psi\ind{2}). \numberthis \label{eq:P_n_defn}
\end{align*}

Additionally, since the state space \(\cS = \cS_1 \cup \dots \cup \cS_H\) is layered, we get that \(\psi\) can be separated across layers, and thus the above definitions imply that 
\begin{align*}
   \mathbb{P}_{n}\ind{1} = \bigotimes_{h=1}^{H-1} \mathbb{P}_{h, n}\ind{1}\quad\text{and}\quad \mathbb{P}_{n}\ind{2} = \bigotimes_{h=1}^{H-1} \mathbb{P}_{h, n}\ind{2}.\numberthis \label{eq:P_n_tensor}
\end{align*}
    

We have the following inequality for TV distance between product measures.  
\begin{lemma}[\cite{polyanskiy2014lecture},   I.33(b)] \label{lem:tensor}
    For distributions $\mathbb{P}_1, \cdots, \mathbb{P}_H$ and $\mathbb{Q}_1, \cdots, \mathbb{Q}_H$, we have
    $$\TV\prn*{\bigotimes_{h=1}^{H-1} \mathbb{P}_h, \bigotimes_{h=1}^{H-1} \mathbb{Q}_h}\le \sum_{h=1}^{H-1} \TV\prn*{\mathbb{P}_h, \mathbb{Q}_h}.$$
\end{lemma}


\par To prove this theorem, we first show the following lemma.

\begin{lemma}\label{lem:TV-lemma}
    For any algorithm which takes $D_{h, n} = \{(\sx_{h, i},\sa_{h, i}, r_{h, i}, \sx_{h, i}')\}_{i=1}^n$ where $h\in [H-1]$ as input and returns a value $\widehat{V}(D_{1:H-1, n})$ (where we use $D_{1:H-1, n}$ to denote $D_{1, n}, \cdots, D_{H-1, n}$), it must satisfy
    $$\sup_{\psi\in\Psi, i\in \{1, 2\}}\mathbb{E}_{D_{1:H-1, n}\sim \mathbb{P}_{n}(\cdot; \rich{\mu}_\psi, \rich{M}_\psi\ind{i})}\left[\left|\widehat{V}(D_{1:H-1, n}) - V(\rich\rho; \rich{M})\right|\right]\ge \frac{\epsilon}{4H}\cdot \left(1 - \sum_{h=1}^{H-1} D_{\mathrm{TV}}(\mathbb{P}_{h, n}\ind{1}, \mathbb{P}_{h, n}\ind{2})\right).$$
\end{lemma}
\begin{proofof}[\pref{lem:TV-lemma}] \pref{lem:sec-lower-bound-equalF} gives that for any $\rich{M}, \rich{M}'\in \{\rich{M}_\psi\ind{1}: \psi\in \Psi\}$, we have $Q_h^\ripieval(\cdot; \rich{M}) = Q_h^\ripieval(\cdot; \rich{M}')$, which implies that MDPs in $\{\rich{M}_\psi\ind{1}: \psi\in \Psi\}$ share the same value function. Hence we have 
    $$V_h^\ripieval(\rich{\rho}; \rich{M}) = V_h^\ripieval(\rich{\rho}; \rich{M}'), \quad \forall M, M'\in \{\rich{M}_\psi\ind{1}: \psi\in \Psi\}.$$ 
    In the following, we denote the above quantity to be $V\ind{1}(\rich{\rho})$. Similarly, we denote $V\ind{2}(\rich{\rho})$ to be the counterpart for MDPs in $\{\rich{M}_\psi\ind{2}: \psi\in \Psi\}$.

    For any dataset $D_{1:H-1, n}$, we use $\delta\prn{D_{1:H-1, n}}$ to denote the following random variable:
    $$\delta\prn{D_{1:H-1, n}} = \indic\left\{\widehat{V}(D_{1:H-1, n})\le \frac{V\ind{1}(\rich\rho) + V\ind{2}(\rich\rho)}{2}\right\}\in \{0, 1\},$$
    Then for any $\psi\in\Psi$, we have 
    \begin{align*}
        &\hspace{-0.5in}\mathbb{E}_{D_{1:H-1, n}\sim \mathbb{P}_{n}(\cdot; \rich{\mu}_\psi, \rich{M}_\psi\ind{1})}\left|\widehat{V}(D_{1:H-1, n}) - V(\rich\rho; \rich{M})\right|\\
        &\ge \mathbb{E}_{D_{1:H-1, n}\sim \mathbb{P}_{n}(\cdot; \rich{\mu}_\psi, \rich{M}_\psi\ind{1})} \left[\delta\prn{D_{1:H-1, n}}\cdot \left|\widehat{V}(D_{1:H-1, n}) - V\ind{1}(\rich{\rho})\right|\right]\\
        &\ge \mathbb{E}_{D_{1:H-1, n}\sim \mathbb{P}_{n}(\cdot; \rich{\mu}_\psi, \rich{M}_\psi\ind{1})}\left[\delta\prn{D_{1:H-1, n}}\cdot \left|\frac{V\ind{1}(\rich{\rho}) + V\ind{2}(\rich{\rho})}{2} - V\ind{1}(\rich{\rho})\right|\right]\\
        & \ge \mathbb{P}_{D_{1:H-1, n}\sim \mathbb{P}_{n}(\cdot; \rich{\mu}_\psi, \rich{M}_\psi\ind{1})}(\delta\prn{D_{1:H-1, n}} = 1)\cdot \frac{V\ind{1}(\rich{\rho}) - V\ind{2}(\rich{\rho})}{2}\\
        & = \frac{\epsilon}{2H}\cdot \mathbb{P}_{D_{1:H-1, n}\sim \mathbb{P}_{n}(\cdot; \rich{\mu}_\psi, \rich{M}_\psi\ind{1})}(\delta\prn{D_{1:H-1, n}} = 1),
    \end{align*}
    where in the last equation we use \ref{lem:sec-lower-bound-latent-MDP-b} and \pref{lem:sec-lower-bound-Q-1}. Similarly, for $\psi\in\Psi$, we have
    $$\mathbb{E}_{D_{1:H-1, n}\sim \mathbb{P}_{n}(\cdot; \rich{\mu}_\psi, \rich{M}_\psi\ind{2})}\left|\widehat{V}(D_{1:H-1, n}) - V(\rich\rho; \rich{M})\right|\ge \frac{\epsilon}{2H}\cdot \mathbb{P}_{D_{1:H-1, n}\sim \mathbb{P}_{n}(\cdot; \rich{\mu}_\psi, \rich{M}_\psi\ind{2})}(\delta\prn{D_{1:H-1, n}} = 0).$$
    Therefore, we obtain that 
    \begin{align*}
        &\hspace{-0.5in}\qquad \sup_{\psi\in\Psi, i\in \{1, 2\}}\mathbb{E}_{D_{1:H-1, n}\sim \mathbb{P}_{n}(\cdot; \rich{\mu}_\psi, \rich{M}_\psi\ind{i})}\left[\left|\widehat{V}(D_{1:H-1, n}) - V(\rich\rho; \rich{M})\right|\right]\\
        & \ge \frac{1}{2|\Psi|}\sum_{\psi\in\Psi}\mathbb{E}_{D_{1:H-1, n}\sim \mathbb{P}_{n}(\cdot; \rich{\mu}_\psi, \rich{M}_\psi\ind{1})}\left[\left|\widehat{V}(D_{1:H-1, n}) - V(\rich\rho; \rich{M})\right|\right]\\
        &\qquad + \frac{1}{2|\Psi|}\sum_{\psi\in\Psi}\mathbb{E}_{D_{1:H-1, n}\sim \mathbb{P}_{n}(\cdot; \rich{\mu}_\psi, \rich{M}_\psi\ind{2})}\left[\left|\widehat{V}(D_{1:H-1, n}) - V(\rich\rho; \rich{M})\right|\right]\\
        & \ge \frac{\epsilon}{4H}\cdot \left(\frac{1}{|\Psi|}\sum_{\psi\in\Psi}\mathbb{P}_{D_{1:H-1, n}\sim \mathbb{P}_{n}(\cdot; \rich{\mu}_\psi, \rich{M}_\psi\ind{1})}(\delta\prn{D_{1:H-1, n}} = 1) + \frac{1}{|\Psi|}\sum_{\psi\in\Psi}\mathbb{P}_{D_{1:H-1, n}\sim \mathbb{P}_{n}(\cdot; \rich{\mu}_\psi, \rich{M}_\psi\ind{2})}(\delta\prn{D_{, n}} = 0)\right)\\
        &\overgeq{\proman{1}} \frac{\epsilon}{4H}\cdot \left(1 - D_{\mathrm{TV}}\left(\frac{1}{|\Psi|}\sum_{\psi\in\Psi}\mathbb{P}_{n}(\cdot; \rich{\mu}_\psi, \rich{M}_\psi\ind{1}),\quad \frac{1}{|\Psi|}\sum_{\psi\in\Psi}\mathbb{P}_{n}(\cdot; \rich{\mu}_\psi, \rich{M}_\psi\ind{2})\right)\right) \\ 
        &\overgeq{\proman{2}} \frac{\epsilon}{4H}\cdot \left(1 - \TV(\mathbb{P}_n\ind{1}, \mathbb{P}_n\ind{2})\right) \\ 
        &\overgeq{\proman{2}} \frac{\epsilon}{4H}\cdot \left(1 - \sum_{h=1}^{H-1} D_{\mathrm{TV}}\prn{\mathbb{P}_{h, n}\ind{1}, \mathbb{P}_{h, n}\ind{2}}\right), 
    \end{align*}
    where in the inequality $\proman{1}$ we use $\mathbb{P}(\cE) + \mathbb{Q}(\cE^c)\ge 1 - D_{\mathrm{TV}}(\mathbb{P}, \mathbb{Q})$ for any event $\cE$, in the inequality $\proman{2}$ we use \pref{eq:P_n_defn}, and in the inequality $\proman{3}$ we use \pref{eq:P_n_tensor} and \pref{lem:tensor}. 
\end{proofof}

Hence we only need to upper bound the TV distance between $\mathbb{P}_{h, n}\ind{1}$ and $\mathbb{P}_{h, n}\ind{2}$, which is proved in the following lemma.

\begin{lemma}\label{lem:main-lemma}
    Suppose for every $\phi\in \Psi$, we have 
    $$|\calX(\phi)|\gtrsim \frac{|\phi|^3H^8\cdot \sup_{h\in H}|\Phi_h|\cdot \sup_{\phi\in \Phi} |\phi| \cdot \aggC_\epsilon(M,\Phi,\mu)^3}{\epsilon^3}.$$ 
    If $n\le \frac{H}{8\epsilon} \aggC_\epsilon(M,\Phi,\mu)$, then we have
    $$\sum_{h=1}^{H-1} D_{\mathrm{TV}}(\mathbb{P}_{h, n}\ind{1}, \mathbb{P}_{h, n}\ind{2})\le \frac{1}{2}.$$
\end{lemma}
\begin{proofof}[\pref{lem:main-lemma}] 
At a high level, the proof contains three steps: 
\begin{enumerate}[label=\((\roman*)\)] 
	\item We first define intermediate distributions \(\bbP_{h, n}\ind{0}\) over tuples \((\sx, \sa, \sr, \sx')\) and observe via Triangle inequality that 
    $$\sum_{h=1}^{H-1} D_{\mathrm{TV}}(\mathbb{P}_{h, n}\ind{1}, \mathbb{P}_{h, n}\ind{2})\le \sum_{h=1}^{H-1} D_{\mathrm{TV}}(\mathbb{P}_{h, n}\ind{1}, \mathbb{P}_{h, n}\ind{0}) + \sum_{h=1}^{H-1} D_{\mathrm{TV}}(\mathbb{P}_{h, n}\ind{0}, \mathbb{P}_{h, n}\ind{2}).$$
    The final bound follows by showing that \(\sum_{h=1}^{H-1} D_{\mathrm{TV}}(\mathbb{P}_{h, n}\ind{i}, \mathbb{P}_{h, n}\ind{0})\leq \nicefrac{1}{4}\) for all \(i \in \crl{1, 2}\). 
	\item Note that \(\mathbb{P}_{h, n}\ind{i}\) is a distribution over tuples \((\sx, \sa, \sr, \sx')\) where the instantaneous reward \(\sr \sim  \check{R}\ind{i}(r \mid \sx, \sa)\). We first simplify our objective a bit by converting  \(\mathbb{P}_{h, n}\ind{i}\) to \(\widetilde{\mathbb{P}}_{h, n}\ind{i}\), where \(\widetilde{\mathbb{P}}_{h, n}\ind{i}\) is a distribution over tuples \((\sx, \sa, \sr, \sx')\) where \(\sr \sim \check{R}\ind{0}(\sr \mid \sx, \sa)\) where  \(\check{R}\ind{0}\) is the reward function in \(\bbP_{h, n}\ind{0}\). Another application of Triangle inequality implies that, 
\begin{align*}
	\sum_{h=1}^{H-1} D_{\mathrm{TV}}(\mathbb{P}_{h, n}\ind{1}, \mathbb{P}_{h, n}\ind{0}) &\leq \sum_{h=1}^{H-1} D_{\mathrm{TV}}(\mathbb{P}_{h, n}\ind{1}, \widetilde {\mathbb{P}}_{h, n}\ind{1}) + \sum_{h=1}^{H-1}	D_{\mathrm{TV}}(\widetilde {\mathbb{P}}_{h, n}\ind{1}, \mathbb{P}_{h, n}\ind{0}). 
\end{align*}
Bounding the first term above is straightforward. 
\item We finally bound the term \(D_{\mathrm{TV}}(\widetilde {\mathbb{P}}_{h, n}\ind{1}, \mathbb{P}_{h, n}\ind{0})\) for each $h\in [H-1]$ by delving further into the structure of the MDPs and the underlying data distribution in \(\widetilde {\mathbb{P}}_n\ind{1}\). Most of the proof will be spend on bound. 
\end{enumerate}

\paragraph{Part-\(\proman{1}\):  Construction of $\mathbb{P}_n\ind{0}$.}  We first define additional notation. Let the distribution \(\nu \in \Delta(\cX \times \cA \times \cX)\) such that for any $h\in [H-1]$, \(\phi\in \Phi_h, \phi'\in \Phi_{h+1}\) and \(\sx\in \cX(\phi), \sx'\in \calX(\phi')\), \(a \in \cA\), we have 
    \begin{equation}\label{eq:sec-lower-bound-def-nu}
        \nu_h(\sx, \sa, \sx') = \sum_{\sz\in\phi}\sum_{\sz'\in\phi'} \frac{\mu(\sz, \sa)T\ind{1}(\sz'\mid \sz, \sa)}{|\calX(\phi)||\calX(\phi')|}, 
    \end{equation}
    Additionally, we define a reward distribution \(\rich{R}_h\ind{0} \in \Delta(\crl{-1, 1})\) such that 
    \begin{align*}
        \rich{R}_h\ind{0}(\cdot \mid \sx,\sa) = \begin{cases}
            \delta_1(\cdot) &\text{if}\quad \sx = \su_{h+1},\\
            \delta_{-1}(\cdot) &\text{if}\quad \sx = \sv_{h+1},\\
            \frac{1}{2}\delta_1(\cdot) + \frac{1}{2}\delta_{-1}(\cdot) &\text{otherwise},
        \end{cases} \numberthis \label{eq:reward_zero_defn}
    \end{align*}
    where we use $\delta_{t}(\cdot)$ denote the density of delta-distribution at $t$. Given \(\nu\) and \(\rich{R}_h\ind{0}\) above, we  define $\mathbb{P}_h\ind{0}\in\Delta(\cX_h\times\cA\times \{-1, 1\}\times \cX_{h+1})$ as 
    \begin{equation}\label{eq:sec-lower-bound-def-P0}
        \mathbb{P}_h\ind{0}((\sx,\sa, r, \sx')) \ldef{}  \nu_h(\sx, \sa, \sx')\rich{R}_h\ind{0}(r\mid \sx, \sa),
    \end{equation}
    and set $\mathbb{P}_{h, n}\ind{0} = (\mathbb{P}_h\ind{0})^{\otimes n}$.

    As a sanity check, note that 
    \begin{align*}
         &\hspace{-2cm}\sum_{(\sx, \sa, r, \sx')\in\calX_h\times\calA\times\{-1, 1\}\times\calX_{h+1}}\mathbb{P}_h\ind{0}((\sx,\sa, r, \sx'))\\
         &= \sum_{(\sx, \sa, \sx')\in\calX_h\times\calA\times\calX_{h+1}}\nu_h(\sx,\sa, \sx')\\
        &= \sum_{\phi\in\Phi_h}\sum_{\phi'\in\Phi_{h+1}'\cup\{\su_{h+1}, \sv_{h+1}\}}\sum_{\sz\in\phi}\sum_{\sz'\in\phi'}\mu(\sz, \sa)T\ind{1}(\sz'\mid \sz, \sa) = 1, 
    \end{align*}
    and thus $\mathbb{P}_h\ind{0}$ is a valid distribution; the above also implies that \(\nu_h\) (defined above) is a valid distribution. Furthermore, while the above definition is based in \(T\ind{1}\), we could have also defined  $\mathbb{P}_h\ind{0}$ using \(T\ind{2}\) and would have ended up with the same distribution since $\sum_{\sz\in\phi}\mu(\sz, \sa)T\ind{1}(\sz'\mid \sz, \sa) = \sum_{\sz\in\phi}\mu(\sz, \sa)T\ind{2}(\sz'\mid \sz, \sa)$ for any $\phi\in\Phi$, $\sz\in\calZ'$ due to \pref{lem:sec-lower-bound-latent-MDP}-\ref{lem:sec-lower-bound-latent-MDP-c}. 
 

 Given $\mathbb{P}_{h, n}\ind{0}$, using Triangle inequality we have
\begin{align*}
\sum_{h\in [H-1]}D_{\mathrm{TV}}(\mathbb{P}_{h, n}\ind{1}, \mathbb{P}_{h, n}\ind{2})\le \sum_{h\in [H-1]} D_{\mathrm{TV}}(\mathbb{P}_{h, n}\ind{1}, \mathbb{P}_{h, n}\ind{0}) + \sum_{h\in [H-1]} D_{\mathrm{TV}}(\mathbb{P}_{h, n}\ind{0}, \mathbb{P}_{h, n}\ind{2}). \numberthis \label{eq:TV_triangle} 
\end{align*}
    
    In the next part, we prove that 
    $\sum_{h\in [H-1]}\TV(\mathbb{P}_{h, n}\ind{1}, \mathbb{P}_{h, n}\ind{0})\le \nicefrac{1}{4}$. The proof for  $\mathbb{P}_{h, n}\ind{2}$ follows similarly, and combining the two bound gives the desired statement. 
\paragraph{Part-\(\proman{2}\):  Construction of $\widetilde{\mathbb{P}}_n\ind{1}$ and bounding $D_{\mathrm{TV}}(\mathbb{P}_{h, n}\ind{1}, \widetilde {\mathbb{P}}_{h, n}\ind{1})$.} First recall that from \pref{eq:P_n_defn}, we can write 
    $$\mathbb{P}_{h, n}\ind{1}(\{(\sx_i,\sa_i, r_i, \sx_i')\}_{i=1}^n) = \frac{1}{|\Psi|}\sum_{\psi\in\Psi} \prod_{i \in [n]} \rich{\mu}_{h, \psi}(\sx_i,\sa_i)\rich{R}_h\ind{1}(r_i \mid \sx_i,\sa_i)\rich{T}_{\psi}\ind{1}(\sx_i'\mid \sx_i,\sa_i),$$
    where $\rich{R}_h\ind{1}$ is given by
    \begin{align*}
        \rich{R}_h\ind{1}(\cdot \mid \sx,\sa) = \begin{cases}
               \delta_1(\cdot) &\text{if}\quad \sx = \su_{h+1},\\
               \delta_{-1}(\cdot) &\text{if}\quad \sx = \sv_{h+1},\\
               R^\star &\text{if}\quad \sx\in \calX(\phi)\text{ for }\phi\in\Phiopt, a = \ripieval(\sx),\\
               \frac{1}{2}\delta_1(\cdot) + \frac{1}{2}\delta_{-1}(\cdot) &\text{otherwise},
        \end{cases} \numberthis \label{eq:reward_one_defn}
    \end{align*}
    where we use $R^\star$ to denote
    $$R^\star = \frac{\delta_1(\cdot) + \delta_{-1}(\cdot)}{2} + \frac{\delta_1(\cdot) - \delta_{-1}(\cdot)}{2}\cdot \frac{\epsilon}{2H\sum_{\phi\in \Phiopt}d_{h^*}^\aggpieval(\phi; \agg{M})}.$$
 
 We next define the distribution  $\widetilde {\mathbb{P}}_{h, n}\ind{1}$ similar to  $ {\mathbb{P}}_{h, n}\ind{1}$, but where we use \(\rich{R}\ind{0}\) (given in \pref{eq:reward_zero_defn}) instead of \(\rich{R}\ind{1}\) to remove the dependence on the rewards on \(i\). In particular, we define  
    \begin{equation}\label{eq:sec-lower-bound-def-tildeP}
        \tilde{\mathbb{P}}_{h, n}\ind{1}(\{(\sx_i,\sa_i, r_i, \sx_i')\}_{i=1}^n) = \frac{1}{|\Psi|}\sum_{\psi\in\Psi} \prod_{i \in [n]} \rich{\mu}_{h, \psi}(\sx_i,\sa_i)\rich{R}_h\ind{0}(r_i \mid \sx_i,\sa_i)\rich{T}_{\psi}\ind{1}(\sx_i'\mid \sx_i,\sa_i).
    \end{equation}

If we denote
$$\rich{r}\ind{1}_h(\sx_i, \sa_i) \ldef{}  \mathbb{E}[\rich{R}\ind{1}_h(\cdot \mid\sx, \sa)]\quad \text{and}\quad\rich{r}\ind{0}_h(\sx_i, \sa_i) \ldef{}  \mathbb{E}[\rich{R}\ind{0}_h(\cdot \mid\sx, \sa)].$$ 
Note that when $n\le \frac{H\aggC_\epsilon(M,\Phi,\mu)}{2\epsilon}$, using the above definitions, we have
    \begin{align*}
        &\hspace*{-0.3in}\sum_{h=1}^{H-1} \TV(\mathbb{P}_{h, n}\ind{1}, \tilde{\mathbb{P}}_{h, n}\ind{1})\\
        & = \frac{1}{2}\sum_{h\in [H-1]}\sum_{\substack{\{(\sx_i,\sa_i, r_i, \sx_i')\}_{i=1}^n \\ \in (\calX_h\times \calA\times \{0, 1\}\times \calX_{h+1})^n}}\left|\mathbb{P}_{h, n}\ind{1}\prn*{\{(\sx_i,\sa_i, r_i, \sx_i')\}_{i=1}^n} - \tilde{\mathbb{P}}_{h, n}\ind{1}\prn*{\{(\sx_i,\sa_i, r_i, \sx_i')\}_{i=1}^n}\right|\\ 
        &\overleq{\proman{1}} \frac{1}{2}\sum_{h\in [H-1]}\sum_{\{(\sx_i,\sa_i)\}_{i=1}^n\in (\calX_h\times\calA)^n}\frac{1}{|\Psi|}\sum_{\psi\in\Psi}\sum_{i \in [n]}|\rich{r}_h\ind{1}(\sx_i,\sa_i) - \rich{r}_h\ind{0}(\sx_i,\sa_i)|\prod_{i \in [n]} \rich{\mu}_{h, \psi}(\sx_i,\sa_i)\\
         &\overeq{\proman{2}}  \frac{1}{2|\Psi|}\sum_{\psi\in\Psi}\sum_{i \in [n]}\sum_{h\in [H-1]}\sum_{(\sx_i,\sa_i)\in \calX_h\times\calA}\rich{\mu}_{h, \psi}(\sx_i,\sa_i)|\rich{r}_h\ind{1}(\sx_i,\sa_i) - \rich{r}_h\ind{0}(\sx_i,\sa_i)|\\
         &\overleq{\proman{3}}  \frac{1}{2|\Psi|}\sum_{\psi\in\Psi}\sum_{i \in [n]}\sum_{h\in [H-1]}\sum_{\sx_i\in \calX_h}\rich{\mu}_\psi(\sx_i, \ripieval(\sx_i))\frac{\epsilon\indic\{\zeta(\psi(\sx_i))\in\Phiopt)\}}{2H\sum_{\phi\in \Phiopt}d_h^\aggpieval(\phi; \agg{M})}\\
        & = \frac{n\epsilon}{4H\sum_{\phi\in \Phiopt}d_h^\aggpieval(\phi; \agg{M})}\sum_{\sz: \zeta(\sz)\in \Phiopt}\mu(\sz, \pieval(\sz)),     \end{align*}
    where inequality \(\proman{1}\) follows from Triangle Inequality, inequality \(\proman{2}\) follows by rearranging the terms and using the fact that \(\sum_{x_i, a_i} \rich{\mu}_{h, \psi}(x_i, a_i) = 1\). The inequality \(\proman{3}\) is from plugging in the forms of \(\rich R\ind{1}\) and \(\rich R \ind{0}\) from \pref{eq:reward_zero_defn} and \pref{eq:reward_one_defn}. Finally, the last line uses the fact that 
    \begin{align*}
\frac{1}{|\Psi|}\sum_{\psi\in\Psi}\sum_{\sx_i}\rich{\mu}_\psi(\sx_i, \ripieval(\sx_i))\indic\{\zeta(\psi(\sx_i))\in\Phiopt)\} = \sum_{\sz: \zeta(\sz)\in \Phiopt}\mu(\sz, \pieval(\sz)), 
\end{align*}
from the definition of \(\rich \mu\). Next, using the definition of $\aggC_\epsilon(M,\Phi,\mu)$  in \pref{def:agg-concentrability}, and recalling that \(\gI\)  is the maximizer aggregation in \pref{def:agg-concentrability}, we get that 
\begin{align*}
 \aggC_\epsilon(M,\Phi,\mu) = \frac{\sum_{\phi\in \Phiopt}d_h^\aggpieval(\phi; \agg{M})}{\sum_{\sz: \zeta(\sz)\in \Phiopt}\mu(\sz, \pieval(\sz))}, 
 \end{align*}
 which implies that 
 \begin{align*}
    \sum_{h=1}^{H-1} \TV(\mathbb{P}_{h, n}\ind{1}, \tilde{\mathbb{P}}_{h, n}\ind{1}) &\leq  \frac{n\epsilon}{4H\aggC_\epsilon(M,\Phi,\mu)} \le \frac{1}{8}, 
\end{align*}
where the last inequality holds since  $n\le \frac{H\aggC_\epsilon(M,\Phi,\mu)}{2\epsilon}$. Thus, using Triangle inequality, 
  \begin{equation}\label{eq:sec-lower-bound-tv}
    \begin{aligned}
        \sum_{h\in [H-1]}\TV(\mathbb{P}_{h, n}\ind{1}, \mathbb{P}_{h, n}\ind{0}) & \le \sum_{h\in [H-1]}\TV(\mathbb{P}_{h, n}\ind{1}, \tilde{\mathbb{P}}_{h, n}\ind{1}) + \sum_{h\in [H-1]}\TV(\mathbb{P}_{h, n}\ind{0}, \tilde{\mathbb{P}}_{h, n}\ind{1})\\
        & \le \frac{1}{8} + \sum_{h\in [H-1]}\TV(\mathbb{P}_{h, n}\ind{0}, \tilde{\mathbb{P}}_{h, n}\ind{1}).
    \end{aligned}
    \end{equation}

\paragraph{Part-\(\proman{3}\): Bound on $\TV(\mathbb{P}_{h, n}\ind{0}, \tilde{\mathbb{P}}_{h, n}\ind{1})$.}  First note that, from  \citet[Proposition 7.13]{polyanskiy2014lecture}, we have 
    \begin{equation}\label{eq:tv-chi2}
        \TV(\mathbb{P}_{h, n}\ind{0}, \tilde{\mathbb{P}}_{h, n}\ind{1})\le \frac{1}{2}\sqrt{D_{\chi^2}(\tilde{\mathbb{P}}_{h, n}\ind{1}\|\mathbb{P}_{h, n}\ind{0})}.
    \end{equation}
Using the form of $\chi^2$-divergence, we note that 
    \begin{align*}
        \hspace{0.3in}&\hspace{-0.3in} D_{\chi^2}(\tilde{\mathbb{P}}_{h, n}\ind{1}\|\mathbb{P}_{h, n}\ind{0}) \\ 
        & = \mathbb{E}_{\{(\sx_i,\sa_i, r_i, \sx_i')\}_{i=1}^n \sim \mathbb{P}_{h, n}\ind{0}}\left[\left(\frac{\tilde{\mathbb{P}}_{h, n}\ind{1}(\{(\sx_i,\sa_i, r_i, \sx_i')\}_{i=1}^n)}{\mathbb{P}_{h, n}\ind{0}(\{(\sx_i,\sa_i, r_i, \sx_i')\}_{i=1}^n)}\right)^2\right] - 1\\
        &= \mathbb{E}_{\{(\sx_i,\sa_i, r_i, \sx_i')\}_{i=1}^n \sim \mathbb{P}_{h, n}\ind{0}}\left[\left(\frac{\frac{1}{|\Psi|}\sum_{\psi\in\Psi}\prod_{i \in [n]} \rich{\mu}_{h, \psi}(\sx_i,\sa_i)\rich{R}_h\ind{0}(r_i\mid\sx_i,\sa_i)\rich{T}_{\psi}\ind{1}(\sx_i'\mid \sx_i,\sa_i)}{\prod_{i \in [n]} \nu_h(\sx_i, \sa_i,\sx_i')\rich{R}_h\ind{0}(r_i\mid\sx_i,\sa_i)}\right)^2\right] - 1, 
 \end{align*}
 where the second line plugs in the definition of $\tilde{\mathbb{P}}_{h, n}\ind{1}$ in \pref{eq:sec-lower-bound-def-tildeP} and $\mathbb{P}_{h, n}\ind{0}$ in \pref{eq:sec-lower-bound-def-P0}. We next note that the terms $\rich{R}_h\ind{0}(r_i\mid\sx_i,\sa_i)$ will cancel out in the ratio, thus implying that the expression is independent of \(\crl{r_i}_{i=1}^n\). Furthermore, from the definition of $\mathbb{P}_{h, n}\ind{0}$  in \pref{eq:sec-lower-bound-def-P0}, we note that sampling  $\{(\sx_i,\sa_i, \sx_i')\}_{i=1}^n\sim \mathbb{P}_{h, n}\ind{0}$ is same as sampling $\{(\sx_i,\sa_i, \sx_i')\}_{i=1}^n\sim \nu_h^{\otimes n}$. Thus, 
 
\begin{align*}
\hspace{0.3in}&\hspace{-0.3in} D_{\chi^2}(\tilde{\mathbb{P}}_{h, n}\ind{1}\|\mathbb{P}_{h, n}\ind{0}) \\ 
        & = \mathbb{E}_{\{(\sx_i,\sa_i, \sx_i')\}_{i=1}^n\sim \nu_h^{\otimes n}}\left[\left(\frac{\frac{1}{|\Psi|}\sum_{\psi\in\Psi}\prod_{i \in [n]} \rich{\mu}_{h, \psi}(\sx_i,\sa_i)\rich{T}_{\psi}\ind{1}(\sx_i'\mid \sx_i,\sa_i)}{\prod_{i \in [n]} \nu_h(\sx_i, \sa_i,\sx_i')}\right)^2\right] - 1\\
        & = \frac{1}{|\Psi|^2}\sum_{\psi_1, \psi_2\in \Psi}\mathbb{E}_{\{(\sx_i,\sa_i, \sx_i')\}_{i=1}^n\sim \nu_h^{\otimes n}}\left[\prod_{i \in [n]} \frac{\rich{\mu}_{h, \psi_1}(\sx_i,\sa_i)\rich{\mu}_{h, \psi_2}(\sx_i,\sa_i)\rich{T}_{\psi_1}\ind{1}(\sx_i'\mid \sx_i,\sa_i)\rich{T}_{\psi_2}\ind{1}(\sx_i'\mid \sx_i,\sa_i)}{\nu_h(\sx_i, \sa_i,\sx_i')^2}\right] - 1\\
        &= \frac{1}{|\Psi|^2}\sum_{\psi_1, \psi_2\in \Psi}\left(\underbrace{\mathbb{E}_{(\sx, \sa, \sx')\sim \nu_h}\left[\frac{\rich{\mu}_{h, \psi_1}(\sx,\sa)\rich{\mu}_{h, \psi_2}(\sx,\sa)\rich{T}_{\psi_1}\ind{1}(\sx'\mid \sx,\sa)\rich{T}_{\psi_2}\ind{1}(\sx'\mid \sx,\sa)}{\nu_h(\sx, \sa,\sx')^2}\right]}_{\ldef{} \mathfrak{B}(\psi_1, \psi_2) }\right)^n - 1, \numberthis \label{eq:proof_interm1}
    \end{align*}
    where the second equality follows by expanding the square and rearranging the product. Finally, the last line exchange the expectation and the product by using the fact that tuples $\{(\sx_i,\sa_i, \sx_i')\}_{i=1}^n$ are i.i.d.~sampled from \(\nu_h\). In the following, we will complete the proof by giving bounds on the terms \(\B(\psi_1, \psi_2)\) under various conditions on  \(\psi_1\) and \(\psi_2\). However, we need additional notation before we proceed:
\begin{enumerate}[label=\(\bullet\)] 
\item Given any  $h\in [H-1]$, $\phi \in \Phi_h$, \(\sz, \sz' \in \phi\), and \(\psi_1, \psi_2 \in \Psi\), we  define  
\begin{align*}
\theta_h(\phi, \sz_1, \sz_2; \psi_1, \psi_2)\ldef{}  \frac{|\calX_{\psi_1}(\sz_1)\cap \calX_{\psi_2}(\sz_2)|}{|\calX(\phi)|} \numberthis \label{eq:def_gamma1}
\end{align*}
to denote the fraction of repeated observations between \(\sz\) and \(\sz'\) amongst all the observations that correspond to aggregation \(\phi.\)
\item Let \(\xi = \nicefrac{1}{(64H^2n)} \in (0, \nicefrac{1}{n})\). For any $h\in [H-1]$, $\phi \in \Phi_h$, and \(\sz_1, \sz_2, \in \phi\), we define  
\begin{align*}
\Gamma_{h}(\xi; \phi, \sz_1, \sz_2) &= \left\{(\psi_1, \psi_2) \in \Psi^2 \mid{} \theta({\phi, \sz_1, \sz_2; \psi_1, \psi_2}) \le \frac{1 + \xi}{|\phi|^2}\right\} \numberthis \label{eq:def_gamma2}
\end{align*}
to denote the set of all pairs \((\psi_1, \psi_2)\) for which the corresponding ratios $\theta({\phi, \sz_1, \sz_2; \psi_1, \psi_2})$ are small. 
\item Finally, we define the set 
\begin{align*}
    \Gamma(\xi) = \bigcap_{h\in [H-1], \phi\in \Phi_h\cup\{\su_h, \sv_h\}} \prn*{\bigcap_{\sz_1, \sz_2, \in \phi} \Gamma(\xi; \phi, \sz_1, \sz_2)}.  \numberthis \label{eq:def_gamma3}
\end{align*}
\end{enumerate}

We now have all the necessary notation to proceed with the proof. We split the terms \(\B(\psi_1, \psi_2)\)  appearing in \pref{eq:proof_interm1} under two separate scenarios: 
\begin{enumerate}[label=\(\bullet\)] 
\item {Case 1: \(\B(\psi_1, \psi_2) \in \Gamma(\xi)\).}  Here, \pref{lem:B_term_1}  (below) implies that 
\begin{align*}
\B(\psi_1, \psi_2) \leq (1 + \xi)^2. 
\end{align*}
\item {Case 2: \(\B(\psi_1, \psi_2) \notin \Gamma(\xi)\).}  Here, \pref{lem:B_term_2}  (below) implies that 
 \begin{align*}
\B(\psi_1, \psi_2) \leq \sup_{\phi\in \Phi}|\phi|^4. 
\end{align*}
\end{enumerate}
Combining the two above in \pref{eq:proof_interm1}, we get that  
    \begin{align*}
        D_{\chi^2}(\tilde{\mathbb{P}}_{h, n}\ind{1}\|\mathbb{P}_{h, n}\ind{0}) &\le \frac{1}{|\Psi|^2}\sum_{\psi_1, \psi_2\in \Psi} \left[\prn*{ \indic\{(\psi_1, \psi_2) \in \Gamma(\xi)\} (1 + \xi)^{2n} + \indic\{(\psi_1, \psi_2)\not\in \Gamma(\xi)\}\sup_{\phi\in \Phi}|\phi|^{4n}} \right] - 1 \\
        &\leq  \frac{1}{|\Psi|^2}\sum_{\psi_1, \psi_2\in \Psi} \left[\prn*{1 + 2n \xi + \indic\{(\psi_1, \psi_2)\not\in \Gamma(\xi)\}\sup_{\phi\in \Phi}|\phi|^{4n}} \right] - 1 \\
        &=  2\xi n + \sup_{\phi\in \Phi}|\phi|^{4n}\cdot \frac{1}{|\Psi|^2}\sum_{\psi_1, \psi_2\in \Psi}\indic\{(\psi_1, \psi_2)\not\in \Gamma(\xi)\}, 
    \end{align*}
    where the second line uses the fact that $(1 + \xi)^n \le 1 + 2\xi n$ for any $\xi \le \nicefrac{1}{n}$. Next, we notice that for $\phi\in \cup_{h=1}^H\{\su_h, \sv_h\}$ and $\sz_1, \sz_2\in \phi$, since $|\phi| = 1$, we always have $(\psi_1, \psi_2)\in \Gamma(\xi; \phi, \sz_1, \sz_2)$. Therefore, we have
    $$\Gamma(\xi) = \bigcap_{h\in [H-1], \phi\in \Phi_h} \prn*{\bigcap_{\sz_1, \sz_2, \in \phi} \Gamma(\xi; \phi, \sz_1, \sz_2)}.$$
    \pref{lem:g_phi} gives that for any $h\in H, \phi\in \Phi_h$ and $\sz_1, \sz_2\in \Phi$, we always have
    $$\frac{1}{|\Psi|^2}\sum_{\psi_1, \psi_2\in \Psi} \indic\{(\psi_1, \psi_2)\not\in \Gamma_{h}(\xi; \phi, \sz_1, \sz_2)\} \le e^{ -2\xi^2\frac{ |\calX(\phi)|}{|\phi|^3}}\le \Gamma_{h}(\xi; \phi, \sz_1, \sz_2)\} \le e^{ -2\xi^2\inf_{\phi\in \Phi}\frac{|\calX(\phi)|}{|\phi|^3}},$$
    which indicates that
    \begin{align*}
        \frac{1}{|\Psi|^2}\sum_{\psi_1, \psi_2\in \Psi}\indic\{(\psi_1, \psi_2)\not\in \Gamma(\xi)\} & \le \sum_{h\in [H-1], \phi\in \Phi_h}\sum_{\sz_1, \sz_2\in \phi}\frac{1}{|\Psi|^2}\sum_{\psi_1, \psi_2\in \Psi} \indic\{(\psi_1, \psi_2)\not\in \Gamma_{h}(\xi; \phi, \sz_1, \sz_2)\}\\
        & \le H\cdot \sup_{h\in H}|\Phi_h|\cdot \sup_{\phi\in \Phi}|\phi|^2\cdot e^{ -2\xi^2\inf_{\phi\in \Phi}\frac{|\calX(\phi)|}{\phi|^3}}
    \end{align*}
    This suggests that
    $$D_{\chi^2}(\tilde{\mathbb{P}}_{h, n}\ind{1}\|\mathbb{P}_{h, n}\ind{0})\le 2\xi n + H\cdot \sup_{h\in H}|\Phi_h|\cdot \sup_{\phi\in \Phi}|\phi|^{4n+2}\cdot e^{ -2\xi^2\inf_{\phi\in \Phi}\frac{|\calX(\phi)|}{|\phi|^3}}.$$
    Finally, when 
    $$|\calX(\phi)|\ge c_0\cdot n^3|\phi|^3H^5\cdot \sup_{h\in H}|\Phi_h|\cdot \sup_{\phi\in \Phi} |\phi|\qquad \forall \phi\in\Phi$$
    for some sufficiently large constant $c_0$, with choice $\xi = \nicefrac{1}{(64H^2n)}$ we will have
    \begin{align*}
        D_{\chi^2}(\tilde{\mathbb{P}}_{h, n}\ind{1}\|\mathbb{P}_{h, n}\ind{0}) \le \frac{1}{16H^2}.
    \end{align*}
    Hence when
    $$|\calX(\phi)|\gtrsim \frac{|\phi|^3H^8\cdot \sup_{h\in H}|\Phi_h|\cdot \sup_{\phi\in \Phi} |\phi|\aggC_\epsilon(M,\Phi,\mu)^3}{\epsilon^3}\quad \text{and}\quad n\le \frac{H}{8\epsilon} \aggC_\epsilon(M,\Phi,\mu)$$
    according to \pref{eq:tv-chi2} we have
    $$\TV(\mathbb{P}_{h, n}\ind{0}, \tilde{\mathbb{P}}_{h, n}\ind{1})\le \frac{1}{8H}.$$
    According to \pref{eq:sec-lower-bound-tv}, we have
    $$\sum_{h\in [H-1]}\TV(\mathbb{P}_{h, n}\ind{1}, \mathbb{P}_{h, n}\ind{0})\le \frac{1}{8} + H\cdot \frac{1}{8H}\le \frac{1}{4}.$$
    Similarly, we can also prove that
    $$\sum_{h\in [H-1]}\TV(\mathbb{P}_{h, n}\ind{2}, \mathbb{P}_{h, n}\ind{0}) \cdot \frac{1}{4}.$$ 
    Plugging in these two in \pref{eq:TV_triangle}, we get the desired bound. 
\end{proofof}

\begin{lemma}
\label{lem:B_term_1}  
We have the following property for $\B_h(\psi_1, \psi_2)$ defined in \pref{eq:proof_interm1}: For any \(\psi_1, \psi_2 \in \Gamma(\xi)\), we have
$$\B_h(\psi_1, \psi_2) \leq (1 + \xi)^2.$$
\end{lemma}
\begin{proofof}[\pref{lem:B_term_1}]
    From \pref{eq:proof_interm1}, recall that 
    $$\B_h(\psi_1, \psi_2) = \mathbb{E}_{(\sx, \sa, \sx')\sim \nu_h}\left[\frac{\rich{\mu}_{h, \psi_1}(\sx,\sa)\rich{\mu}_{h, \psi_2}(\sx,\sa)\rich{T}_{\psi_1}\ind{1}(\sx'\mid \sx,\sa)\rich{T}_{\psi_2}\ind{1}(\sx'\mid \sx,\sa)}{\nu_h(\sx, \sa,\sx')^2}\right].$$
    According to our construction of $\rich{\mu}_{h, \psi}$ in \pref{app:block_MDP}, for any $\sx\in \calX(\sz)$, we have
    $$\rich{\mu}_{h, \psi}(\sx, \sa) = \frac{\mu_h(\sz, \sa)}{|\calX(\sz)|} = \mu_h(\sz, \sa)\cdot \frac{|\phi|}{|\calX(\phi)|}.$$
    Additionally, according to our construction of $\rich{T}_\psi\ind{1}(\sx'\mid\sx, \sa)$, for any $\sx\in\calX(\sz)$ and $\sx'\in\calX(\sz')$, we have
    $$\rich{T}_\psi\ind{1}(\sx'\mid\sx, \sa) = \frac{T\ind{1}(\sz'\mid\sz, \sa)}{|\calX(\phi')|} = T\ind{1}(\sz'\mid\sz, \sa)\cdot \frac{|\phi'|}{|\calX(\phi')|}.$$
    This implies that for any $\phi\in\Phi_h, \phi'\in\Phi_{h+1}'$, and $\sx\in\calX_{\psi_1}(\sz_1)\cap \calX_{\psi_2}(\sz_2), \sx'\in\calX_{\psi_1}(\sz_1')\cap \calX_{\psi_2}(\sz_2')$, we have
    \begin{align*}
        &\hspace*{-0.5in}\rich{\mu}_{h, \psi_1}(\sx,\sa)\rich{\mu}_{h, \psi_2}(\sx,\sa)\rich{T}_{\psi_1}\ind{1}(\sx'\mid \sx,\sa)\rich{T}_{\psi_2}\ind{1}(\sx'\mid \sx,\sa)\\
        & = \frac{|\phi|^2}{|\calX(\phi)|^2}\cdot \frac{|\phi'|^2}{|\calX(\phi')|^2}\cdot \mu_h(\sz_1, \sa)\mu_h(\sz_2, \sa)T\ind{1}(\sz_1\mid \sz_1, \sa)T\ind{1}(\sz_2\mid\sz_2, \sa).\numberthis \label{eq:mu-T}
    \end{align*}
    And the definition of $\theta$ in \pref{eq:def_gamma1} gives that
    \begin{align*}
        |\calX_{\psi_1}(\sz_1)\cap \calX_{\psi_2}(\sz_2)| & = \theta_h(\phi, \sz_1, \sz_2; \psi_1, \psi_2)\cdot |\calX(\phi)|\\
        |\calX_{\psi_1}(\sz_1')\cap \calX_{\psi_2}(\sz_2')| & = \theta_{h+1}(\phi', \sz_1', \sz_2'; \psi_1, \psi_2)\cdot |\calX(\phi)|
    \end{align*} 
    Hence we can write
    \begin{align*} 
        &\hspace*{0.3in}\mathbb{E}_{(\sx, \sa, \sx')\sim \nu_h}\left[\frac{\rich{\mu}_{h, \psi_1}(\sx,\sa)\rich{\mu}_{h, \psi_2}(\sx,\sa)\rich{T}_{\psi_1}\ind{1}(\sx'\mid \sx,\sa)\rich{T}_{\psi_2}\ind{1}(\sx'\mid \sx,\sa)}{\nu_h(\sx, \sa,\sx')^2}\right]\\
        & = \sum_{\sx\in \calX_h}\sum_{\sa\in\calA}\sum_{\sx'\in\calX_{h+1}\cup\{\su_{h+1}, \sv_{h+1}\}}\frac{\rich{\mu}_{h, \psi_1}(\sx,\sa)\rich{\mu}_{h, \psi_2}(\sx,\sa)\rich{T}_{\psi_1}\ind{1}(\sx'\mid \sx,\sa)\rich{T}_{\psi_2}\ind{1}(\sx'\mid \sx,\sa)}{\nu_h(\sx, \sa,\sx')}\\
        &= \sum_{a \in \cA} \sum_{\substack{\phi\in \Phi_h \\ \phi'\in \Phi_{h+1}\cup\{\su_{h+1}, \sv_{h+1}\}}} \Term(a, \phi, \phi'), \numberthis \label{eq:formula-R}
        \end{align*}
        where we defined the quantity 
        \begin{align*}
        \Term(a, \phi, \phi') \ldef{}  \sum_{\sx\in \calX(\phi)}\sum_{\sx'\in\calX(\phi')}\frac{\rich{\mu}_{h, \psi_1}(\sx,\sa)\rich{\mu}_{h, \psi_2}(\sx,\sa)\rich{T}_{\psi_1}\ind{1}(\sx'\mid \sx,\sa)\rich{T}_{\psi_2}\ind{1}(\sx'\mid \sx,\sa)}{\sum_{\sz\in\phi}\sum_{\sz'\in\phi'} \frac{\mu_h(\sz, \sa)T\ind{1}(\sz'\mid \sz, \sa)}{|\calX(\phi)||\calX(\phi')|}}. 
        \end{align*}

        We next bound the terms \(\Term(a, \phi, \phi')\). Notice that 
        \begin{align*}
        &\hspace*{-0.3in}\Term(a, \phi, \phi') \\  
        & \overeq{\proman{1}} \sum_{\sx\in \calX(\phi)}\sum_{\sx'\in\calX(\phi')}\frac{\rich{\mu}_{h, \psi_1}(\sx,\sa)\rich{\mu}_{h, \psi_2}(\sx,\sa)\rich{T}_{\psi_1}\ind{1}(\sx'\mid \sx,\sa)\rich{T}_{\psi_2}\ind{1}(\sx'\mid \sx,\sa)}{\sum_{\sz\in\phi}\sum_{\sz'\in\phi'} \frac{\mu_h(\sz, \sa)T\ind{1}(\sz'\mid \sz, \sa)}{|\calX(\phi)||\calX(\phi')|}}\\
        & \overeq{\proman{2}}\sum_{\sz_1, \sz_2\in \phi}\sum_{\sz_1', \sz_2'\in\phi'}\mu_h(\sz_1, \sa)\mu_h(\sz_2, \sa)T\ind{1}(\sz_1\mid \sz_1, \sa)T\ind{1}(\sz_2\mid\sz_2, \sa)\\
        & \hspace*{2in}\cdot \frac{\frac{|\phi|^2}{|\calX(\phi)|^2}\cdot \frac{|\phi'|^2}{|\calX(\phi')|^2}\cdot |\calX_{\psi_1}(\sz_1')\cap \calX_{\psi_2}(\sz_2')||\calX_{\psi_1}(\sz_1)\cap \calX_{\psi_2}(\sz_2)|}{\sum_{\sz\in\phi}\sum_{\sz'\in\phi'} \frac{\mu_h(\sz, \sa)T\ind{1}(\sz'\mid \sz, \sa)}{|\calX(\phi)||\calX(\phi')|}}\\
        & \overeq{\proman{3}} \sum_{\sz_1, \sz_2\in \phi}\sum_{\sz_1', \sz_2'\in\phi'}\mu_h(\sz_1, \sa)\mu_h(\sz_2, \sa)T\ind{1}(\sz_1\mid \sz_1, \sa)T\ind{1}(\sz_2\mid\sz_2, \sa)\\
        & \hspace*{1in}\cdot \frac{\frac{|\phi|^2}{|\calX(\phi)|^2}\cdot \frac{|\phi'|^2}{|\calX(\phi')|^2}\cdot \theta_h(\phi, \sz_1, \sz_2; \psi_1, \psi_2)\cdot \theta_{h+1}(\phi', \sz_1', \sz_2'; \psi_1, \psi_2)\cdot |\calX(\phi)|}{\sum_{\sz\in\phi}\sum_{\sz'\in\phi'} \frac{\mu_h(\sz, \sa)T\ind{1}(\sz'\mid \sz, \sa)}{|\calX(\phi)||\calX(\phi')|}}\\
        & \overeq{\proman{4}} \sum_{\sz_1, \sz_2\in \phi}\sum_{\sz_1', \sz_2'\in\phi'}|\phi|^2|\phi'|^2\cdot \mu_h(\sz_1, \sa)\mu_h(\sz_2, \sa)T\ind{1}(\sz_1\mid \sz_1, \sa)T\ind{1}(\sz_2\mid\sz_2, \sa)\\
        &\hspace*{1in}\cdot\frac{\theta_h(\phi, \sz_1, \sz_2; \psi_1, \psi_2)\cdot \theta_{h+1}(\phi', \sz_1', \sz_2'; \psi_1, \psi_2)}{\sum_{\sz\in\phi}\sum_{\sz'\in\phi'} \mu_h(\sz, \sa)T\ind{1}(\sz'\mid \sz, \sa)},\numberthis \label{eq:formula-B}
    \end{align*}
    where in $\proman{1}$ we use the exact form of $\nu_h$ defined in \pref{eq:sec-lower-bound-def-nu}, in $\proman{2}$ we group those $\sx\in \calX_{\psi_1}(\sz_1)\cap \calX_{\psi_2}(\sz_2)$ and $\sx'\in\calX_{\psi_1}(\sz_1')\cap \calX_{\psi_2}(\sz_2')$ together (because the summand gives the same value) and use \pref{eq:mu-T}, in $\proman{3}$  we use the definition of $\theta_h(\phi, \sz_1, \sz_2; \psi_1, \psi_2)$ in \pref{eq:def_gamma1}, and $\proman{4}$ is just algebraic calculation.

    Next, when $(\psi_1, \psi_2)\in \Gamma(\xi)$, according to the definition of $\gamma(\xi)$ in \pref{eq:def_gamma3}, we have 
    $$(\psi_1, \psi_2)\in \Gamma_{h}(\xi; \phi, \sz_1, \sz_2) \quad \forall h\in [H-1], \phi\in\Phi_h, \sz_1, \sz_2\in \psi,$$
    which implies that
    \begin{align*}
        \theta_h(\phi, \sz_1, \sz_2; \psi_1, \psi_2)\le \frac{1 + \xi}{|\phi|^2} \quad \text{and}\quad  \theta_{h+1}(\phi', \sz_1', \sz_2'; \psi_1, \psi_2)\le \frac{1 + \xi}{|\phi'|^2}. \numberthis\label{eq:theta-1-2}
    \end{align*}
    Bringing this back to \pref{eq:formula-B}, we obtain that
    \begin{align*}
        & \hspace*{-0.3in}\Term(a, \phi, \phi')\\
        & \overleq{\proman{1}} (1 + \xi)^2\cdot \sum_{\sz_1, \sz_2\in \phi}\sum_{\sz_1', \sz_2'\in\phi'} \frac{\mu_h(\sz_1, \sa)\mu_h(\sz_2, \sa)T\ind{1}(\sz_1'\mid \sz_1, \sa)T\ind{1}(\sz_2'\mid\sz_2, \sa)}{\sum_{\sz\in\phi}\sum_{\sz'\in\phi'} \mu_h(\sz, \sa)T\ind{1}(\sz'\mid \sz, \sa)}\\
        & \overeq{\proman{2}} (1 + \xi)^2\cdot\sum_{\sz_1\in \phi}\sum_{\sz_1'\in\phi'}\mu_h(\sz_1, \sa)T\ind{1}(\sz_1'\mid \sz_1, \sa),
    \end{align*}
    where in $\proman{1}$ we use \pref{eq:theta-1-2}, and $\proman{2}$ is by straightforward algebraic calculations. Bringing this back to \pref{eq:formula-R}, we obtain
    \begin{align*} 
        &\hspace*{-0.5in}\mathbb{E}_{(\sx, \sa, \sx')\sim \nu_h}\left[\frac{\rich{\mu}_{h, \psi_1}(\sx,\sa)\rich{\mu}_{h, \psi_2}(\sx,\sa)\rich{T}_{\psi_1}\ind{1}(\sx'\mid \sx,\sa)\rich{T}_{\psi_2}\ind{1}(\sx'\mid \sx,\sa)}{\nu_h(\sx, \sa,\sx')^2}\right]\\
        & = \sum_{a \in \cA} \sum_{\substack{\phi\in \Phi_h \\ \phi'\in \Phi_{h+1}\cup\{\su_{h+1}, \sv_{h+1}\}}} \Term(a, \phi, \phi')\\
        & \overleq{\proman{3}} (1 + \xi)^2 \sum_{\sz\in\calZ_h}\sum_{a\in\calA}\sum_{\sz'\in\calZ_{h+1}'} \mu_h(\sz_1, \sa)T\ind{1}(\sz'\mid \sz, \sa) = (1 + \xi)^2,
    \end{align*}
    where in $\proman{3}$ we use the fact that $\mu_h(\cdot)\in \Delta(\calZ_h\times\calA)$ and $T\ind{1}(\cdot \mid\sz, \sa)\in \Delta(\calZ_{h+1})$ for any $\sz\in \calZ$ and $\sa\in\calA$.
\end{proofof} 

\begin{lemma}\label{lem:B_term_2}  
For any \(\psi_1, \psi_2 \not\in \Gamma(\xi)\), the term $\B_h(\psi_1, \psi_2)$ satisfies 
$$\B_h(\psi_1, \psi_2) \leq \sup_{\phi\in \Phi}|\phi|^4.$$
\end{lemma}
\begin{proofof}[\pref{lem:B_term_2}]
    For $(\psi_1, \psi_2)\not\in\Gamma(\xi)$, since for any $h\in [H-1]$, $\phi\in \Phi_h\cup \{\su_{h}, \sv_h\}$ and $\sz_1, \sz_2\in \phi$, we always have
    $$\calX_{\psi_1}(\sz_1)\cap \calX_{\psi_2}(\sz_2)\subset \calX(\phi),$$
    we have $\theta_h(\phi, \sz_1, \sz_2; \psi_1, \psi_2)\le 1$. Bringing this back to \pref{eq:formula-B}, we obtain that
    \begin{align*} 
        &\hspace*{-0.2in}\B_h(\psi_1, \psi_2)  \\
        &= \mathbb{E}_{(\sx, \sa, \sx')\sim \nu_h}\left[\frac{\rich{\mu}_{h, \psi_1}(\sx,\sa)\rich{\mu}_{h, \psi_2}(\sx,\sa)\rich{T}_{\psi_1}\ind{1}(\sx'\mid \sx,\sa)\rich{T}_{\psi_2}\ind{1}(\sx'\mid \sx,\sa)}{\nu_h(\sx, \sa,\sx')^2}\right]\\
        & \le \sup_{\phi\in \Phi_h}|\phi|^2\sup_{\phi'\in\Phi'\cup\{\su_{h+1}, \sv_{h+1}\}}|\phi'|^2\\
        &\quad \hspace{1in} \cdot \sum_{\sa\in\calA} \sum_{\substack{\phi\in \Phi_h \\ \phi'\in \Phi_{h+1}\cup\{\su_{h+1}, \sv_{h+1}\}}} \sum_{\substack{ \sz_2\in \phi \\ \sz_2'\in\phi'}} 
        \prn*{\frac{\mu_h(\sz_1, \sa)\mu_h(\sz_2, \sa)T\ind{1}(\sz_1'\mid \sz_1, \sa)T\ind{1}(\sz_2'\mid\sz_2, \sa)}{\sum_{\sz\in\phi}\sum_{\sz'\in\phi'} \mu_h(\sz, \sa)T\ind{1}(\sz'\mid \sz, \sa)}} \\
        & = \sup_{\phi\in \Phi_h}|\phi|^2\sup_{\phi'\in\Phi'\cup\{\su_{h+1}, \sv_{h+1}\}}|\phi'|^2\cdot\sum_{\sa\in\calA} \sum_{\substack{\phi\in \Phi_h \\ \phi'\in \Phi_{h+1}\cup\{\su_{h+1}, \sv_{h+1}\}}} \sum_{\substack{ \sz_1\in \phi \\ \sz_1'\in\phi'}} \prn*{\mu_h(\sz_1, \sa)T\ind{1}(\sz_1'\mid \sz_1, \sa)} \\ 
        & = \sup_{\phi\in \Phi_h}|\phi|^2\sup_{\phi'\in\Phi'\cup\{\su_{h+1}, \sv_{h+1}\}}|\phi'|^2 \sum_{a\in\calA} \sum_{\substack{\sz\in\calZ_h \\ \sz'\in\calZ_{h+1}'}} \mu_h(\sz_1, \sa)T\ind{1}(\sz_1'\mid \sz_1, \sa)\\
        & = \sup_{\phi\in \Phi_h}|\phi|^2\sup_{\phi'\in\Phi'\cup\{\su_{h+1}, \sv_{h+1}\}}|\phi'|^2.
    \end{align*}
    Further notice that for any $\phi'\in \{\su_{h+1}, \sv_{h+1}\}$, we have $|\phi'|\le 1\le \sup_{\phi'\in \Phi'}|\phi'|$, which indicates that $\B_h(\psi_1, \psi_2)\le \sup_{\phi\in \Phi_h}|\phi|^2 \cdot \sup_{\phi'\in \Phi_h} |\phi'|^2\le \sup_{\phi\in \Phi}|\phi|^4$.
\end{proofof}

\begin{lemma}\label{lem:g_phi}
    For any $h\in [H-1], \phi\in \Phi_h$ and $\sz_1, \sz_2\in \Phi$ and $\xi\in (0, 1)$, we have
    $$\frac{1}{|\Psi|^2}\sum_{\psi_1, \psi_2\in \Psi} \indic\{(\psi_1, \psi_2)\not\in \Gamma_{h}(\xi; \phi, \sz_1, \sz_2)\} \le e^{ -2\xi^2\frac{ |\calX(\phi)|}{|\phi|^3}},$$
    where $\Gamma_{h}(\xi; \phi, \sz_1, \sz_2)$ is defined in \pref{eq:def_gamma2}
\end{lemma}
\begin{proofof}[\pref{lem:g_phi}]
    We denote $S = |\calX(\phi)|$. Without loss of generality, we assume $\calX(\phi) = [S] = \{1, 2, \cdots, S\}$. Since $\psi_1$ and $\psi_2$ are samples i.i.d.~ according to $\text{Unif}(\Psi)$, without loss of generality we assume $\calX_{\psi_1}(\sz_1) = [\nicefrac{S}{|\phi|}]$. And to prove this lemma, we only need to verify that when $\psi_2\sim \mathrm{Uniform}(\Psi)$, we have 
    \begin{equation}\label{eq:sec-lower-bound-high-probability}
        \mathbb{P}\left(\left|\calX_{\psi_2}(\sz_2)\cap \left[\frac{S}{|\phi|}\right]\right|\ge \frac{S ( 1 + \xi)}{|\phi|^2}\right)\le e^{ -2\xi^2\frac{S}{|\phi|^3}}.
    \end{equation}
 
    Next, we notice that sampling $\psi_2\sim \mathrm{Uniform}(\Psi)$ is equivalent of sampling $\calX_{\psi_2}(\sz_2)$ uniformly from all subsets of $[S]$ with exact $\nicefrac{S}{|\phi|}$ elements. Hence we obtain that
    $$\mathbb{P}\left(\left|\calX_{\psi_2}(\sz_2)\cap \left[\frac{S}{|\phi|}\right]\right|\ge \frac{S ( 1 + \xi)}{|\phi|^2}\right) = \sum_{t\ge \nicefrac{S(1 + \xi)}{|\phi|^2}}\frac{\binom{\nicefrac{S}{|\phi|}}{t}\binom{S - \nicefrac{S}{|\phi|}}{\nicefrac{S}{|\phi|} - t}}{\binom{S}{\nicefrac{S}{|\phi|}}}.$$
    We further notice that according to Lemma D.7 in \cite{foster2022offline} (also in \cite{hoeffding1994probability}), we get
    \begin{align*}
    \sum_{t\ge \nicefrac{S(1 + \xi)}{|\phi|^2}}\frac{\binom{\nicefrac{S}{|\phi|}}{t}\binom{S - \nicefrac{S}{|\phi|}}{\nicefrac{S}{|\phi|} - t}}{\binom{S}{\nicefrac{S}{|\phi|}}} = \mathbb{P}\left[X\ge \left(\frac{K}{N} + \frac{\xi}{|\phi|}\right)N'\right]\le e^{ -2\xi^2\frac{S}{|\phi|^3}},
    \end{align*}
    where $X\sim \mathrm{Hypergeometric}(K, N, N')$ with $K = N' = \nicefrac{S}{|\phi|}, N = S$. This verifies \pref{eq:sec-lower-bound-high-probability}.
\end{proofof} 

\begin{lemma}[Lemma D.7 in \cite{foster2022offline}]
    Let $X\sim\mathrm{Hypergeometric}(K, N, N')$ and define $p = K/N$. Then for any $0 < \epsilon < pN'$, we have
    $$\mathbb{P}[X\ge (p + \epsilon)N']\le \exp(-2\epsilon^2N').$$
\end{lemma}

\subsubsection{Proof of \pref{thm:stronger}}

\begin{proofof}[\pref{thm:stronger}]
    Given the original \MDP $M$ and distribution $\mu$, we construct family $\mathfrak{G}$ of \OPE problems in \pref{app:family_construction}. First of all, \pref{lem:sec-lower-bound-equalF} indicates that the size of the function class $\mathcal{F}$ of these \OPE problems in $\mathfrak{G}$ is $2$ . Next, bringing \pref{lem:main-lemma} into \pref{lem:TV-lemma} and noticing that \pref{eq:rich_count} satisfies the condition of \pref{lem:main-lemma}, we have that any algorithm which takes $D_{h, n}$ for $h\in [H-1]$ must induce error at least $\nicefrac{\epsilon}{8H}$ in one case within $\mathfrak{G}$.

    Additionally, \pref{corr:sec-lower-bound-agg-con} indicates that the standard concentrability coefficient of instances in $\mathfrak{G}$ is equal to to the standard concentrability coefficient of $\tabC(M\ind{1}, \mu')$ or $\tabC(M\ind{2}, \mu')$. Additionally, \pref{corr:sec-lower-bound-con} indicates that $\tabC(M\ind{1}, \mu')$ or $\tabC(M\ind{2}, \mu')$ is no more than $6\tabC(M, \mu, \pieval)$. This verifies the second condition in \pref{thm:stronger}. 
\end{proofof}

\section{Missing Details from \pref{sec: state aggregation} 
and \pref{sec: admissible data}} 
\label{app: admissible lb}

\par In this subsection, we will prove \pref{prop:toy_example} 
 and \pref{thm:admissible}.

\subsection{Proof of \pref{prop:toy_example}} 
The following lemmas state some properties of the \MDP defined in \pref{eg:arbitrary_behavior}. \pref{prop:toy_example} is a direct corollary of \pref{lem:admissible-c}.

\begin{lemma}\label{lem:sec-admissible-mu} For \MDP \(M\), and policy $\piexp$ defined above, we have 
$$d_h^\piexp(\sz_h\inds{1}; M) \ge \frac{1}{2^{h+2}}, \quad d_h^\piexp(\sz_h\inds{2}; M) \ge \frac{1}{2^{h+2} H}, \quad\text{and}\quad d_h^\piexp(\sz_h\inds{3}; M) \ge \frac{1}{4},\quad \forall 1\le h\le H.$$
\end{lemma}
\begin{proofof}[\pref{lem:sec-admissible-mu}]
    Under policy $\piexp(\sz) = \frac{1}{H^2}\delta_{\aone}(\cdot) + \frac{H^2-1}{H^2}\delta_{\atwo}(\cdot)$, the transition satisfies
    \begin{align*}
        T(\sz_{h+1}\inds{1}\mid \sz_h\inds{1}, \piexp(\sz_h\inds{1})) \ge \frac{H^2-1}{H^2}T(\sz_{h+1}\inds{1}\mid \sz_h\inds{1}, \atwo)& = \frac{H^2-1}{2H^2},\\
        T(\sz_{h+1}\inds{2}\mid \sz_h\inds{2}, \piexp(\sz_h\inds{2})) \ge \frac{H^2-1}{H^2}T(\sz_{h+1}\inds{2}\mid \sz_h\inds{2}, \atwo)& = \frac{H^2-1}{2H^2}, ~\text{and,} \\ 
      \quad T(\sz_{h+1}\inds{3}\mid \sz_h\inds{3}, \piexp(\sz_h\inds{3})) \ge \frac{H^2-1}{H^2}T(\sz_{h+1}\inds{3}\mid \sz_h\inds{3}, \atwo)& = \frac{H^2-1}{H^2}.
    \end{align*}
    Hence according to the choice initial distribution $\rho(\cdot)$, we have for all $1\le h\le H$,
    \begin{align*}
        d_h^\piexp(\sz_h\inds{1}; M) & \ge \frac{H-1}{2H}\cdot \left(\frac{H^2-1}{2H^2}\right)^{h-1} \ge  \frac{1}{2^{h+2}},\\
        \quad
        d_h^\piexp(\sz_h\inds{2}; M) & = \frac{1}{2H}\cdot \left(\frac{H^2-1}{2H^2}\right)^{h-1}\ge \frac{1}{2^{h+2} H}, ~\text{and,} \\
        \quad d_h^\piexp(\sz_h\inds{3}; M) & \ge \frac{1}{2}\left(\frac{H^2-1}{H^2}\right)^{h-1}\ge \frac{1}{4}.
    \end{align*}
\end{proofof}

\begin{lemma}\label{lem:sec-admissible-mu-2}
    For \MDP $M$ and policy $\piexp$ defined above, we have
    $$\frac{d^\piexp_h(\sz_h\inds{1})}{d^\piexp_h(\sz_h\inds{2})} \ge \frac{H-1}{3}.$$
\end{lemma}
\begin{proofof}[\pref{lem:sec-admissible-mu-2}]
    First we can write the dynamic programming formula for $d^\pi_h$:
\begin{align*}
    d_{h+1}^\piexp(\sz_{h+1}\inds{1}) & = d^\piexp_h(\sz_h\inds{1})T(\sz_{h+1}\inds{1}\mid \sz_h\inds{1}, \piexp(\sz_h\inds{1})) + d^\piexp_h(\sz_h\inds{2})T(\sz_{h+1}\inds{1}\mid \sz_h\inds{2}, \piexp(\sz_h\inds{2})),\\
    d_{h+1}^\piexp(\sz_{h+1}\inds{2}) & = d^\piexp_h(\sz_h\inds{1})T(\sz_{h+1}\inds{2}\mid \sz_h\inds{1}, \piexp(\sz_h\inds{1})) + d^\piexp_h(\sz_h\inds{2})T(\sz_{h+1}\inds{2}\mid \sz_h\inds{2}, \piexp(\sz_h\inds{2})).
\end{align*}
According to our choice of $\piexp$, we have
\begin{align*} 
T(\sz_{h+1}\inds{1}\mid \sz_h\inds{1}, \piexp(\sz_h\inds{1})) = \frac{H^2-1}{2H^2}, &\quad T(\sz_{h+1}\inds{2}\mid \sz_h\inds{1}, \piexp(\sz_h\inds{1})) = \frac{1}{H^2},\\
T(\sz_{h+1}\inds{1}\mid \sz_h\inds{2}, \piexp(\sz_h\inds{2})) = 0, &\quad\text{and}\quad T(\sz_{h+1}\inds{2}\mid \sz_h\inds{2}, \piexp(\sz_h\inds{2})) = \frac{H^2-1}{2H^2},
\end{align*}
which indicates that
\begin{align*}
    \frac{d_{h+1}^\piexp(\sz_{h+1}\inds{2})}{d_{h+1}^\piexp(\sz_{h+1}\inds{1})} & = \frac{T(\sz_{h+1}\inds{2}\mid \sz_h\inds{1}, \piexp(\sz_h\inds{1}))}{T(\sz_{h+1}\inds{1}\mid \sz_h\inds{1}, \piexp(\sz_h\inds{1}))} + \frac{T(\sz_{h+1}\inds{2}\mid \sz_h\inds{2}, \piexp(\sz_h\inds{2}))}{T(\sz_{h+1}\inds{1}\mid \sz_h\inds{1}, \piexp(\sz_h\inds{1}))}\cdot \frac{d_{h}^\piexp(\sz_{h}\inds{2})}{d_{h}^\piexp(\sz_{h}\inds{1})}\le \frac{2}{H^2-1} + \frac{d_{h}^\piexp(\sz_{h}\inds{2})}{d_{h}^\piexp(\sz_{h}\inds{1})}.
\end{align*}
Additionally, after noticing that $\frac{d_{h}^\piexp(\sz_{1}\inds{2})}{d_{h}^\piexp(\sz_{1}\inds{1})} = \frac{\rho(\sz_1\inds{2})}{\rho(\sz_1\inds{1})} = \frac{1}{H-1}$, we have for all $1\le h\le H$, 
$$\frac{d_{h}^\piexp(\sz_{h+1}\inds{2})}{d_{h}^\piexp(\sz_{h+1}\inds{1})}\le \frac{1}{H-1} + (h-1)\cdot \frac{2}{H^2-1}\le \frac{1}{H-1} + \frac{2(H-1)}{H^2-1}\le \frac{3}{H-1}.$$
\end{proofof}

\begin{lemma}\label{lem:sec-admissible-dpi}
    For any policy $\pi$, and the \MDP $M$ defined above, we have
    $$d_h^\pi(\sz_h\inds{1}) + d_h^\pi(\sz_h\inds{2})\le \frac{1}{2^{h-1}} \quad \forall h\in [H].$$
\end{lemma} 
\begin{proofof}[\pref{lem:sec-admissible-dpi}] First we can write the dynamic programming formula for $d^\pi_h$:
\begin{align*}
    d_{h+1}^\pi(\sz_{h+1}\inds{1}) & = d^\pi_h(\sz_h\inds{1})T(\sz_{h+1}\inds{1}\mid \sz_h\inds{1}, \pi(\sz_h\inds{1})) + d^\pi_h(\sz_h\inds{2})T(\sz_{h+1}\inds{1}\mid \sz_h\inds{2}, \pi(\sz_h\inds{2}))\\
    d_{h+1}^\pi(\sz_{h+1}\inds{2}) & = d^\pi_h(\sz_h\inds{1})T(\sz_{h+1}\inds{2}\mid \sz_h\inds{1}, \pi(\sz_h\inds{1})) + d^\pi_h(\sz_h\inds{2})T(\sz_{h+1}\inds{2}\mid \sz_h\inds{2}, \pi(\sz_h\inds{2})). 
\end{align*}
We let $\pi(\sz_h\inds{1}) = p_1\delta_{a_1}(\cdot) + (1 - p_1)\delta_{a_2}(\cdot)$, and $\pi(\sz_h\inds{2}) = p_2\delta_{a_1}(\cdot) + (1 - p_2)\delta_{a_2}(\cdot)$. Then we have
\begin{align*} 
T(\sz_{h+1}\inds{1}\mid \sz_h\inds{1}, \pi(\sz_h\inds{1})) = \frac{1}{2}(1-p_1), &\quad T(\sz_{h+1}\inds{2}\mid \sz_h\inds{1}, \pi(\sz_h\inds{1})) = p_1,\\
T(\sz_{h+1}\inds{1}\mid \sz_h\inds{2}, \pi(\sz_h\inds{2})) = 0, &\quad\text{and}\quad T(\sz_{h+1}\inds{2}\mid \sz_h\inds{2}, \pi(\sz_h\inds{2})) = \frac{1}{2}(1-p_2).
\end{align*}
This implies that
\begin{align*}
   2d_{h+1}^\pi(\sz_{h+1}\inds{1}) + d_{h+1}^\pi(\sz_{h+1}\inds{2})  &= d_h^\pi(\sz_h\inds{1})\cdot \left(1-p_1p_1\right) + d_h^\pi(\sz_h\inds{1})\cdot\left(2\cdot 0 + \frac{1}{2}(1 - p_2)\right)\\
    & = \frac{1}{2}\cdot \left(2d^\pi(\sz_h\inds{1}) + d^\pi(\sz_h\inds{2})\right).
\end{align*}
Further noticing that $2d^\pi(\sz_1\inds{1}) + d^\pi(\sz_1\inds{2}) = 2\rho(\sz_1\inds{1}) + \rho(\sz_1\inds{2}) \le 1 $, we obtain that for any $h\in [H]$, 
$$2d_{h}^\pi(\sz_{h}\inds{1}) + d_{h}^\pi(\sz_{h}\inds{2})\le \frac{1}{2^{h-1}}.$$
Therefore, for any $h\in [H]$, we have 
$$d_h^\pi(\sz_h\inds{1}) + d_h^\pi(\sz_h\inds{2})\le \frac{1}{2^{h-1}}.$$
    
\end{proofof}

\begin{lemma}\label{lem:admissible-c} For the \MDP \(M\) and policy \(\piexp\) defined above, the concentrability coefficient of all policies with respect to $d_h^\piexp(\cdot; M)$ is upper bounded as 
\begin{align*}
    \max_{\pi\in \Pi}\max_{h}\max_{\sz\in\calZ_h, \sa\in\calA}\frac{d_h^\pi(\sz, \sa; M)}{d_h^\piexp(\sz, \sa; M)}\le 8H^3,
\end{align*}
where $\Pi$ is the class of all policies. However, for $\epsilon\le \nicefrac{1}{15}$, the aggregated concentrability coefficient is lower bounded as
\begin{align*}
    \ophatC{M}{\Phi}{d_h^\piexp(\cdot; M)}{\epsilon}\ge 2^{H-7}.
\end{align*}
\end{lemma}
\begin{proofof}[\pref{lem:admissible-c}]
    For our choice of $\piexp$, 
    $$d_h^\piexp(\sz, \sa; M)\ge \frac{d_h^\piexp(\sz; M)}{H^2}\quad \text{and}\quad d_h^\pi(\sz, \sa; M)\le \frac{d_h^\pi(\sz; M)}{H^2}\quad \forall \sz\in\calZ, \sa\in\calA,$$
    which implies that
    $$\frac{d_h^\pi(\sz, \sa; M)}{d_h^\piexp(\sz, \sa; M)}\le H^2\cdot \frac{d_h^\pi(\sz; M)}{d_h^\piexp(\sz; M)}\qquad \forall \sz\in\calZ.$$
    Hence, \pref{lem:sec-admissible-mu} and \pref{lem:sec-admissible-dpi} give that for any policy $\pi\in \Pi, 1\le h\le H$ and $\sz\in \calZ_h$, 
    $$\frac{d_h^\pi(\sz; M)}{d_h^\piexp(\sz; M)}\le \max\left\{4, \frac{\nicefrac{1}{2^{h-1}}}{\nicefrac{1}{(2^{h+2}H)}}\right\} = 8H.$$
    This implies that
    $$\max_{\pi\in \Pi}\max_{h}\max_{\sz\in\calZ_h, \sa\in\calA}\frac{d_h^\pi(\sz, \sa; M)}{d_h^\piexp(\sz, \sa; M)}\le 8H^3.$$

    As for the lower bound of $\ophatC{M}{\Phi}{d^\piexp(\cdot; M)}{\epsilon}$, first we notice that according to \pref{eq:agg-transition-model}, our choice of $\pieval(\sz) = \delta_{\aone}(\cdot)$ and \pref{lem:sec-admissible-mu-2}, we have
    \begin{align*}
        \agg{T}(\phi_{h+1}\inds{1}\mid \phi_h\inds{1}, \pieval) & = \frac{\sum_{\sz_h\in\phi_h\inds{1}, \sz_{h+1}'\in\phi_{h+1}\inds{1}}d_h^\piexp(\sz_h, \aone; M)T(\sz_{h+1}'\mid \sz_h, \aone)}{\sum_{\sz\in\phi_h\inds{1}}d_h^\piexp(\sz, \aone; M)} \\ 
        &\ge \frac{\sum_{\sz_{h+1}'\in\phi_{h+1}\inds{1}}d_h^\piexp(\sz_h, \aone; M)T(\sz_{h+1}'\mid \sz_h\inds{1}, \aone)}{\sum_{\sz\in\phi_h\inds{1}}d_h^\piexp(\sz_h, \aone; M)}\\
        & = \frac{d_h^\piexp(\sz_h, \aone; M)}{\sum_{\sz\in\phi_h\inds{1}}d_h^\piexp(\sz_h, \aone; M)} = \frac{H-1}{H-1 + 3} = \frac{H-1}{H+2}.
    \end{align*}
    This implies that $\agg{d}_H^\pieval(\phi_H\inds{1}; M)$ satisfies for $H\ge 2$,
    \begin{align*}
        \agg{d}_H^\pieval(\phi_H\inds{1}; M)\ge \agg{d}_1^\pieval(\phi_1\inds{1}; M)\prod_{h=1}^{H-1}\agg{T}(\phi_{h+1}\inds{1}\mid \phi_h\inds{1}, \pi)\ge \frac{1}{2}\left(\frac{H-1}{H+2}\right)^{H-1}\ge \frac{1}{2e^3}\ge \frac{1}{41}.
    \end{align*}
    Hence when $\epsilon\le \nicefrac{1}{15}$, we have 
    \begin{align*}
        \ophatC{M}{\Phi}{\mu}{\epsilon} & \ge \frac{\agg{d}_H^\pieval(\phi_H\inds{1}; M)}{\sum_{\sz\in \phi_H\inds{1}}d^\piexp_H(\sz, \pieval(\sz); M)}\\ 
        &= \frac{\agg{d}_H^\pieval(\phi_H\inds{1}; M)}{d^\piexp_H(\sz_H\inds{1}, \pieval(\sz_H\inds{1}); M) + d^\piexp_h(\sz_H\inds{2}, \pieval(\sz_H\inds{2}); M)}\\
        & \ge \frac{\nicefrac{1}{41}}{\nicefrac{1}{(2^{H-1})}} \\ 
        &\ge 2^{H-7},
    \end{align*}
    where in the second inequality we adopt \pref{lem:sec-admissible-dpi} with $\pi = \piexp$. 
\end{proofof}

\subsection{Proof of \pref{thm:admissible}}

In this section,  we will prove the following stronger results of \pref{thm:admissible}. 
\begin{theorem} \label{thm:admissible-stronger}
There exists a class \(\famGg\) of realizable \OPE problems such that for every \OPE problem \(\opeg = \OPE\prn*{M\ind{\opeg}, \pieval\ind{\opeg}, \piexp\ind{\opeg}, \cF\ind{\opeg}}\) in \(\famGg\), 
\begin{enumerate}[label=\((\alph*)\)]  
    \item  \label{thm:admissible-stronger-a} The function class \(\cF\ind{\opeg}\) satisfies \(\abs{\cF\ind{\opeg}} = 2\) and 
    $$(Q^{\pieval\ind{\opeg}}(\cdot ; M\ind{\opeg}), W^{\pieval\ind{\opeg}}(\cdot; d^{\piexp\ind{\opeg}}(\cdot; M\ind{\opeg}); M\ind{\opeg}))\in \cF\ind{\opeg}.$$ 
    \item \label{thm:admissible-stronger-b} The concentrability coefficient of all policies with respect to $d_h^{\piexp\ind{\opeg}}(\cdot; M\ind{\opeg})$ is upper bounded by $384H^3$. 
\end{enumerate}
Furthermore, any offline policy evaluation algorithm that guarantees to estimate the value of \(\pieval\ind{\opeg}\) in the \MDP \(M\ind{\opeg}\) 
up to precision \(\epsilon\), in expectation, for every \OPE problem \(\opeg \in \famGg\) must use  
\begin{align*}
\wt \Omega\prn*{\frac{H2^H}{\epsilon}}
\end{align*} 
admissible samples collected according to policy $\piexp\ind{\opeg}$ under MDP $M\ind{\opeg}$. Here one admissible sample consists of $(\sx_h, \sa_h, r_h, \sx_{h+1}')$ for all $h\in [H-1]$, where $(\sx_h, \sa_h)\sim d_h^{\piexp\ind{\opeg}}(\cdot; M\ind{\opeg})$, and $r_h, \sx_{h+1}'$ are collected according to the transition model and reward function of $M\ind{\opeg}$.
\end{theorem} 

\begin{proofof}[\pref{thm:admissible-stronger}]
    We consider the class of offline RL problems constructed in \pref{app:family_construction}: $\mathfrak{G} = \bigcup_{\psi \in \Psi} \crl{\mathfrak{g}\inds{1}_\psi,\mathfrak{g}\inds{2}_\psi}$. According to the construction, the sampling distributions over rich observations in $\mathfrak{g}\inds{1}_\psi$ or $\mathfrak{g}\inds{2}_\psi$ are $\rich{\mu}_\psi(\sx, \sa) = \nicefrac{\mu(\psi(\sz), \sa)}{|\calX_\psi(\sz)|}$. And \pref{thm:stronger} indicates that for $\epsilon\le \nicefrac{1}{41}$, for any algorithm using less than
    $$\frac{H\ophatC{M}{\phi}{\mu}{\epsilon}}{8\epsilon}\ge 2^{H-10}\frac{H}{\epsilon}$$
    samples, there must exist some $\psi\in\Phi$ such that if the samples are according to $\rich{\mu}_\psi$, the algorithm will have error greater than $\nicefrac{\epsilon}{8H}$ in $\mathfrak{g}\inds{1}_\psi$ or $\mathfrak{g}\inds{2}_\psi$ with probability at least $\nicefrac{1}{2}$.

    According to \pref{lem:sec-lower-bound-latent-MDP} \ref{lem:sec-lower-bound-latent-MDP-d},
    we have $d^\piexp(\sz; M\inds{1})\le d^\piexp(\sz; M)$ for any $\sz\in\calZ$. Hence we obtain that for any $\sx\in\calX\backslash \{\su, \sv, \sw\}_{h=1}^H$,
    $$d^{\ripiexp}(\sx; \rich{M}_\psi\inds{1}) = \frac{d^\piexp(\psi(\sx); M)}{|\calX_\psi(\sz)|}\le \frac{d^\piexp(\psi(\sx); M)}{|\calX_\psi(\sz)|}.$$
    This also implies
    $$d^{\ripiexp}(\sx, \sa; \rich{M}_\psi\inds{1}) = d^{\ripiexp}(\sx; \rich{M}_\psi\inds{1})\piexp(\sa\mid\sx)\le \frac{d^\piexp(\psi(\sx); M)\piexp(\sa\mid\sx)}{|\calX_\psi(\sz)|} = \frac{d^\piexp(\psi(\sx), \sa; M)}{|\calX_\psi(\sz)|} = \rich{\mu}_\psi(\sx, \sa).$$ 
    Similarly, we can also obtain that
    $$d^{\ripiexp}(\sx, \sa; \rich{M}_\psi\inds{2})\le \rich{\mu}_\psi(\sx, \sa).$$
  Hence, for any algorithm using less than $\nicefrac{2^{H-10}H^3}{\epsilon}$ number of samples, there must exists some $\rich{M}\in \{\rich{M}_\psi\inds{1}: \psi\in\Psi\}\cup\{\rich{M}_\psi\inds{2}: \psi\in\Psi\}$ such that if the samples are according to $d^{\ripiexp}(\cdot \mid \rich{M})$, the estimation error is at least $\nicefrac{\epsilon}{H}$.

    \par Finally, \pref{lem:admissible-c} indicates that the concentrability coefficient of $M$ is bounded by $8H^3$. Hence according to \pref{corr:sec-lower-bound-d-pi} and \pref{lem:sec-lower-bound-rich-d} we have for any such $\rich{M}$, the concentrability coefficient is upper bounded by $384H^3$, which proves condition \ref{thm:admissible-stronger-b}. And similar to the proof of \pref{thm:stronger}, we can also verify condition \ref{thm:admissible-stronger-a}.
\end{proofof}


\clearpage

\section{Missing Details from \pref{sec: trajectory lower bound}} \label{app: proof of reduction}

In this section, we will provide missing details from \pref{sec: trajectory lower bound}. In \pref{sec: algorithm} we present algorithms needed in \pref{sec: trajectory lower bound}. Finally, in \pref{sec: trajectory lower bound} we present the proof of \pref{thm:trajectory_lower_bound}.


\subsection{Algorithms}\label{sec: algorithm}
\par In this subsection, we provide algorithms mentioned in \pref{sec: trajectory lower bound}. \pref{alg:replicator} will transform a hard-case \OPE problem for admissible data into a hard-case \OPE problem for trajectory data. \pref{alg:convert} is used in the proof of lower bound with trajectory data, where the algorithm can transform admissible data collected according to $M$ (MDP in the hard-case \OPE problem of admissible data) into trajectory data of $\wt M$ (MDP in the hard-case \OPE problem of trajectory data). And finally, \pref{alg:convert-admissible} shows how to transform an algorithm for \OPE with trajectory data to an algorithm for \OPE with admissible data.

\begin{algorithm}
\begin{algorithmic}[1]
\Require Offline policy evaluation problem $\OPE(M, \pieval, \piexp, \cF)$ with $M = \MDP(\calX, \calA, T, r, H, \rho)$, and parameter \(K\ge 1\).  
    \State Define $\piother: \calX\to \calA$ to be an arbitrary mapping which satisfies $\piother(\sx)\neq \pieval(\sx)$ for all $\sx\in\calX$.
\State \algcommentbig{Construct \(\wt M = \MDP(\wt \calX, \calA, \wt T, \wt r, \wt H, \wt \rho)\)} 
    \State \textit{Horizon:} $\wt H = K(H-1)+1$.
    \State \textit{State space:} Let \(\cX = \bigcup_{l \in [\wt H]} \wt\calX_l\) ,  where the state space $\wt\calX_l = \{(\sx, k): \sx\in\calX_h, k\in [K]\}$ $l = (h-1)K+k$, for $h\in [H]$ and $k\in [K]$.  
    \State \textit{Initial distribution:} Let $\wt\rho((\sx, 1)) = \rho(\sx)$ for $\sx\in\calX_1$.
    \State \textit{Transition model:} Define the transition matrix \(\wt T: \wt \cX \times \cA \mapsto \Delta(\wt \cX)\) such that 
    \begin{enumerate}[label=\(\bullet\)] 
        \item  \textit{When $k\le K-1$}: For any \(h \in [H]\),  \(a \in \cA\), and $\sx, \sx'\in \calX_h$, let  
            $$\wt T\prn*{(\sx', k+1) \mid{} (\sx, k), \sa} = \begin{cases}
    \mathbb{I}(\sx' = \sx) & \text{if} \quad a = \pieval(\sx), \\ 
    d^\piexp_h(\sx'; M) & \text{otherwise}. 
    \end{cases}$$
 \item \textit{When \(k = K\)}: For any \(h \in [H-1]\), \(a \in \cA\),  $\sx\in\calX_h$, and  $\sx'\in\calX_{h+1}$, let     $$\wt T((\sx', 1)\mid (\sx, k), \sa) = T(\sx'\mid\sx, \sa).$$
    \end{enumerate}
    
    \State \textit{Reward function:} For any \(x \in \cX\) and \(a \in \cA\), let 
    \begin{align*}
    \wt r((\sx, k), \sa) &= \begin{cases}
        0 & \text{if} \quad k \leq K-1, \\ 
        r(s, a) & \text{otherwise}.
    \end{cases} 
    \end{align*}
    
    \State\algcommentbig{Construct evaluation policy \(\wtpieval\)} 
    \State Construct $\wtpieval$ such that  $\wtpieval((\sx, k)) = \pieval(\sx)$ for all $\sx\in\calX$ and $k\in  [K]$.
\State\algcommentbig{Construct offline policy \(\wtpiexp\)} 
    \State Construct $\wtpiexp$ such that, for any $\sx\in\calX$ and $k\in [K]$, $$\wtpiexp((\sx, k)) = \begin{cases}
    \piexp(\sx) &\text{if}\quad k = K,\\
    \frac{1}{2}\pieval(\sx) + \frac{1}{2}\piother(\sx) &\text{otherwise}.
    \end{cases}$$ 
\State\algcommentbig{Construct state-action value function class \(\wt \cF\)}  
    \State Construct $\wt \cF \ldef{} \{\wt Q: Q\in\cF \}$, where $\wt Q: \wt\calX\times\calA\to \mathbb{R}$ is defined such that     \begin{equation}\label{eq:sec-trajectory-Q}
    \begin{aligned}
    \wt Q((\sx_h, k), \sa) = \begin{cases}
        Q(\sx_h, \sa) &\text{if}\quad k = K,\\
        Q(\sx_h, \pieval(\sx_h)) &\text{if}\quad k < K, \sa = \pieval(\sx_h),\\
        \sum_{\sx_h'\in\calX_h}d^\piexp_h(\sx_h'; M)Q(\sx_h', \pieval(\sx_h')) &\text{if}\quad k < K, \sa \neq \pieval(\sx_h).
    \end{cases}
    \end{aligned}\end{equation}
for any $\sx\in\calX$ and $k\in[K]$. 
    \State \textbf{Return:} Offline policy evaluation problem \(\OPE(\wt M, \wtpieval, \wtpiexp, \wt \cF)\). 
\end{algorithmic} 
	\caption{\Replicator}
	\label{alg:replicator} 
\end{algorithm} 

\begin{algorithm}
    \begin{algorithmic}[1] 
        \State \textbf{Input:} Admissible datasets $\Da_h = \{(\sx_h\up{l}, \sa\up{l}_h, \sr\up{l}_h, \shp \up{l}_h)\}_{l \leq K}$ for \(h \in [H]\). 
        \State Set \(\wt \sx_1 = \sx_{1}\up{1}\), and initialize \(\tau = 
        (\wt \sx_1)\). 
        
        \For{$h = 1, \dots, H-1$} 
        \State Set \(l = 1\). \\  
       \algcommentbig{Construct Trajectory within \Block~\(h\)} 
        
            \For{$k = 1, \dots, K-1$}  
                \State Let $\wt h\ldef{}  (h-1)K+k$.  
                \State Sample $\wt \sa_{\wt h} \sim \mathrm{Uniform}(\crl{\pieval(\wt\sx_{\wt h}), \piother(\wt\sx_{\wt h})})$, and update $l \leftarrow l + 1$ if \(\wt \sa_{\wt h} = \piother(\wt\sx_{\wt h})\). 
                \State Set 
\begin{align*} 
\hspace{1in} \wt \sx_{\wt h + 1} = \begin{cases} 
\wt \sx_{\wt h} & \text{if}\quad \wt \sa_{\wt h} = \pieval(\wt \sx_{\wt h}) \hspace{0.3in} \algcomment{Map to the Same State}    \\   
\sx_{h} \up{l} & \text{if}\quad \wt \sa_{\wt h} = \piother(\wt \sx_{\wt h}) \hspace{0.3in}
\algcomment{Read Fresh State from \(\Da_h\)} 
\end{cases}
\end{align*}
\State Update \(\tau = \tau \circ \prn{\wt \sa_{\wt h}, \wt \sr_{\wt h} = 0, \wt \sx_{\wt h + 1}}\). 
            \EndFor
            \State Update  \(\tau = \tau \circ \prn{\wt \sa_{\wt h} = \sa_h\up{l}, \wt \sr_{\wt h} = \sr_h\up{l}, \wt \sx_{\wt h + 1} = \bar\sx_h \up{l}}\). 
        \EndFor 
        \State \textbf{Return} trajectory \(\tau\) of length \((H-1)K + 1\).  
    \end{algorithmic}
\caption{\converter}
\label{alg:convert}
\end{algorithm}

\begin{algorithm}
    \caption{Reduction of \OPE with Admissible Data to Trajectory Data} 
    \label{alg:convert-admissible} 
    \begin{algorithmic}[1] 
        \Require 
        \item[]  \begin{itemize}
            \item Admissible dataset \(\Da_h\) of size \(Kn\) for each \(h \in [H]\). 
            \item State-action value function class \(\cF\). 
            \item \OPE algorithm \(\At\) that takes evaluation policy, state-action value function class and trajectory data as input and returns a value estimation. 
        \end{itemize}   
        \State Initialize \(\Dt = \emptyset\). 
        \For{$j = 1, \dots, n$}
            \State For \(h \in [K]\), construct \(\cD^{\textsc{adm}, j}_h = \Da_h[K(j-1) + 1: Kj]\) 
            \State Get trajectory \(\tau^j = \converter\prn*{\cD^{\textsc{adm}, j}_1, \dots, \cD^{\textsc{adm}, j}_H}\). 
            \State Update \(\Dt \leftarrow \Dt \cup \crl{\tau^j}\). 
        \EndFor
        \State Construct state-action value function class $\wt{\mathcal{F}}$ according to \pref{eq:sec-trajectory-Q} based on $\mathcal{F}$.
        \State Return \(\wh V \leftarrow \At(\Dt, \wt{\mathcal{F}})\). 
    \end{algorithmic}
\end{algorithm}

\subsection{Proof of \pref{thm:trajectory_lower_bound} } 

For \MDP\footnote{Throughout this section, we use a "tilde" over the variables, e.g.~in \(\tilde{M}\), \(\tilde{\pi}\), etc., to signify that they correspond to the artificially constructed OPE problem in \pref{alg:replicator} for analyzing trajectory data.} $\wt M$ and policies $\wtpieval$ and $\wtpiexp$, we have the following properties:  

\begin{lemma}\label{lem:sec-trajectory-concentrability}
    Suppose under \OPE problem $\OPE(\wt M, \wtpieval, \wtpieval, \cF)$ is the output of \pref{alg:replicator} after inputing \OPE problem $\OPE(M, \pieval, \piexp, \cF)$ and arbitrary integer $K$. Then, 
    \begin{enumerate}[label = \((\alph*)\)]
        \item \label{lem:sec-trajectory-concentrability-a} The standard concentrability of $\wt M$ and $M$ satisfies that
        $$\tabC(\wt M, \wt d_h^\wtpiexp(\cdot; \wt M)\le 2\tabC(M, d_h^\piexp(\cdot; M)).$$
        \item \label{lem:sec-trajectory-concentrability-b} $\wt Q$ calculated in \pref{eq:sec-trajectory-Q} equal to the state-action value function of \MDP $\wt M$. Especially, the value functions of $\wt M$ and $M$ satisfies
        \begin{align}\label{eq:sec-trajectory-rho}
        \wt V(\wt\rho; \wt M) = V(\rho; M),
        \end{align}
        where $\wt \rho$ and $\rho$ are initial distributions of $\wt M$ and $M$ respectively.
    \end{enumerate}
\end{lemma} 
\begin{proofof}[\pref{lem:sec-trajectory-concentrability}]
    We first prove \pref{lem:sec-trajectory-concentrability} \ref{lem:sec-trajectory-concentrability-a}. We only need to verify that for any $1\le h\le H$ and $1\le k\le K$, 
    \begin{equation}\label{eq:sec-trajectory-occupancy}
    \frac{d_h^{\wtpieval}((\sx_h, k), \sa; \wt M)}{d_h^{\wtpiexp}((\sx_h, k), \sa; \wt M)}\le 2\cdot \frac{d_h^\pieval(\sx_h, \sa; M)}{d_h^\piexp(\sx_h, \sa; M)}.
    \end{equation}

    Our first observation is that $d_h^{\wtpieval}((\sx_h, k); \wt M) = d^{\pieval}_h(\sx_h; M).$ This can be proved via induction on $(h-1)K + k$. When $(h-1)K + k = 1$, i.e. $h = k = 1$, we have 
    $$d_h^{\wtpieval}((\sx_h, k); \wt M) = \wt\rho(\sx_1) = \rho(\sx_1) = d^{\pieval}_h(\sx_h; M).$$
    Next, we will do the induction. When $k\ge 2$, we have $\wt T((\sx_h, k)\mid (\sx_h, k-1), \wtpieval((\sx_h, k-1))) = 1$, which implies $$d_h^{\wtpieval}((\sx_h, k); \wt M) = d_h^{\wtpieval}((\sx_h, k-1); \wt M) = d^\pieval(\sx_h; M).$$ 
    When $k = 1$, we have $\wt T((\sx_h, 1)\mid (\sx_{h-1}, K), \wtpieval((\sx_{h-1}, K))) = T(\sx_h\mid \sx_{h-1}, \pieval(\sx_{h-1}))$. Hence by induction, we have
    \begin{align*}
        d_h^{\wtpieval}((\sx_h, 1); \wt M) & = \sum_{\sx_{h-1}\in\calX_{h-1}}d_h^{\wtpieval}((\sx_{h-1}, K); \wt M)\wt T((\sx_h, 1)\mid (\sx_{h-1}, K), \wtpieval((\sx_{h-1}, K)))\\
        & = \sum_{\sx_{h-1}\in\calX_{h-1}}d_{h-1}^{\pieval}(\sx_{h-1}; M)T(\sx_h\mid \sx_{h-1}, \pieval(\sx_{h-1})) \\
        &= d^\pieval_h(\sx_h; M). 
    \end{align*}

    Our next observation is that $d_h^{\wtpiexp}((\sx_h, k); \wt M) = d^{\piexp}_h(\sx_h; M)$. This can be proved via induction on $(h-1)K + k$. When $(h-1)K + k = 1$, i.e. $h = k = 1$, we have 
    $$d_h^{\wtpiexp}((\sx_h, k); \wt M) = \wt\rho(\sx_1) = \rho(\sx_1) = d^{\piexp}_h(\sx_h; M).$$
    Next, we will do the induction. When $k\ge 2$, according to the transition model of $\wt M$, we have
    $$\wt T((\sx_h, k)\mid (\sx_h', k-1), \piother(\sx_h'))) = d_h^\piexp(\sx_h),$$
    which indicates that 
    \begin{align*}
        d_h^{\wtpiexp}((\sx_h, k); \wt M) & = \sum_{\sx_h'\in\calX_h} d_h^{\wtpiexp}((\sx_h', k-1); \wt M)\wt T((\sx_h, k)\mid (\sx_h', k-1), \wtpiexp((\sx_h', k-1))))\\
        & = \frac{1}{2}\sum_{\sx_h'\in\calX_h} d_h^{\piexp}(\sx_h'; M)\wt T((\sx_h, k)\mid (\sx_h', k-1), \pieval(\sx_h')))\\
        &\quad +\frac{1}{2}\sum_{\sx_h'\in\calX_h} d_h^{\piexp}(\sx_h'; M)\wt T((\sx_h, k)\mid (\sx_h', k-1), \piother(\sx_h')))\\
        & = \frac{1}{2}d_h^{\piexp}(\sx_h; M) + \frac{1}{2}\sum_{\sx_h'\in\calX_h} d_h^{\piexp}(\sx_h'; M)d_h^\piexp(\sx_h; M) \\ &= \frac{1}{2}d_h^\piexp(\sx_h; M)
    \end{align*}
    When $k = 1$, we have $\wt T((\sx_h, 1)\mid (\sx_{h-1}, K), \wtpiexp((\sx_{h-1}, K))) = T(\sx_h\mid \sx_{h-1}, \piexp(\sx_{h-1}))$. Hence by induction, we have
    \begin{align*}
        d_h^{\wtpiexp}((\sx_h, 1); \wt M) & = \sum_{\sx_{h-1}\in\calX_{h-1}}d_h^{\wtpiexp}((\sx_{h-1}, K); \wt M)\wt T((\sx_h, 1)\mid (\sx_{h-1}, K), \wtpiexp((\sx_{h-1}, K)))\\
        & = \sum_{\sx_{h-1}\in\calX_{h-1}}d_{h-1}^{\piexp}(\sx_{h-1}; M)T(\sx_h\mid \sx_{h-1}, \pieval(\sx_{h-1})) \\ &= d^\piexp_h(\sx_h; M). 
    \end{align*}
    
    Our final observation is that for $a\neq \pieval(\sx_h)$, we always have
    $$d_h^{\wtpieval}((\sx_h, k), \sa; \wt M) = 0,$$
    hence \pref{eq:sec-trajectory-occupancy} holds. For $a = \pieval(\sx_h)$, we have
    \begin{align*}
        \frac{d_h^{\wtpieval}((\sx_h, k), \sa; \wt M)}{d_h^{\wtpiexp}((\sx_h, k), \sa; \wt M)} & = \frac{d_h^{\wtpieval}((\sx_h, k); \wt M)}{d_h^{\wtpiexp}((\sx_h, k); \wt M)}\cdot \frac{\wtpieval(\sa\mid (\sx_h, k))}{\wtpiexp(\sa\mid(\sx_h, k))},\\
        \frac{d_h^\pieval(\sx_h, \sa; M)}{d_h^\piexp(\sx_h, \sa; M)} & = \frac{d_h^\pieval(\sx_h; M)}{d_h^\piexp(\sx_h; M)}\cdot \frac{\pieval(\sa\mid\sx_h)}{\piexp(\sa\mid\sx_h)}.
    \end{align*} 
    According to our construction of $\wtpieval$ and $\wtpiexp$, we further have 
    $$\frac{\wtpieval(\sa\mid (\sx_h, K))}{\wtpiexp(\sa\mid(\sx_h, K))} = \frac{\pieval(\sa\mid\sx_h)}{\piexp(\sa\mid\sx_h)}.$$
    And for $k\neq K$, we have
    $$\frac{\wtpieval(\sa\mid (\sx_h, k))}{\wtpiexp(\sa\mid(\sx_h, k))} = 2.$$
    After noticing that $\piexp(\pieval(\sx_h)\mid \sx_h)\le 1$, we obtain that for $a = \pieval(\sx_h)$, 
    $$\frac{\pieval(\sa\mid\sx_h)}{\piexp(\sa\mid\sx_h)}\ge 1 = \frac{1}{2}\cdot \frac{\wtpieval(\sa\mid (\sx_h, k))}{\wtpiexp(\sa\mid(\sx_h, k))}.$$
    Hence as long as $\sa = \pieval(\sx_h)$, we always have
    $$\frac{\wtpieval(\sa\mid (\sx_h, K))}{\wtpiexp(\sa\mid(\sx_h, K))}\le 2\cdot \frac{\pieval(\sa\mid\sx_h)}{\piexp(\sa\mid\sx_h)}.$$
    This implies
    \begin{align*} 
        \frac{d_h^{\wtpieval}((\sx_h, k), \sa; \wt M)}{d_h^{\wtpiexp}((\sx_h, k), \sa; \wt M)}\le 2\cdot \frac{d_h^\pieval(\sx_h, \sa; M)}{d_h^\piexp(\sx_h, \sa; M)}.
    \end{align*}

    Next, we will prove \pref{lem:sec-trajectory-concentrability}\ref{lem:sec-trajectory-concentrability-b} by induction on $l\ldef{}  (h-1)H+k$ from $\wt H$ to $1$. When $l = \wt H$, we have 
    $$\wt Q((\sx_H, 1), \sa) = Q(\sx_H, \sa) = 0, \quad \forall \sx_H\in\calX_H, \sa\in\calA,$$
    which satisfies the induction hypothesis. Next, assuming the induction hypothesis holds for $l + 1$, we will prove the induction hypothesis for $l$. We assume $l = (h-1)K + k$. 
    
    When $k = K$, according to Bellman equation and induction hypothesis, we have
    \begin{align*} 
        \wt Q((\sx_h, K), \sa) & = \wt r((\sx_h, K), \sa) + \sum_{\sx'_{h+1}\in \calX_{h + 1}} \wt T((\sx_{h+1}', 1)\mid (\sx_h, K), \sa)\wt Q((\sx_{h+1}', 1), \wtpieval((\sx_{h+1}', 1)))\\
        & = r(\sx_{h}, \sa) + \sum_{\sx_{h+1}'\in \calX_{h+1}} T(\sx_{h+1}'\mid \sx_{h}, \sa)Q(\sx_{h+1}', \pieval(\sx_{h+1}')) = Q(\sx_h, \sa).
    \end{align*} 

    When $k < K$ and $\sa = \pieval(\sx_h)$, according to Bellman equation, we have
    \begin{align*} 
        \wt Q((\sx_h, k), \sa) & = \wt r((\sx_h, k), \sa) + \sum_{\sx'_h\in \calX_{h}} \wt T((\sx_h', k+1)\mid (\sx_h, k), \sa)\wt Q((\sx_h', k+1), \wtpieval((\sx_h', k+1)))\\
        & = \wt Q((\sx_h, k+1), \wtpieval((\sx_h, k+1))) = Q(\sx_h, \pieval(\sx_h)),
    \end{align*} 
    where in the second last equation we use the fact that $\wt T((\sx_h, k+1)\mid (\sx_h, k), \pieval(\sx_h)) = 1$.

    When $k < K$ and $\sa\neq \pieval(\sx_h)$, according to Bellman equation, we have
    \begin{align*} 
        \wt Q((\sx_h, k), \sa) & = \wt r((\sx_h, k), \sa) + \sum_{\sx'_h\in \calX_{h}} \wt T((\sx_h', k+1)\mid (\sx_h, k), \sa)\wt Q((\sx_h', k+1), \wtpieval((\sx_h', k+1)))\\
        & = \sum_{\sx'_h\in \calX_{h}} d_h^\piexp(\sx_h'; M)\wt Q((\sx_h', k+1), \wtpieval((\sx_h', k+1))) = d^\piexp_h(\sx_h; M)Q(\sx_h, \pieval(\sx_h)),
    \end{align*} 
    
    Finally for \pref{eq:sec-trajectory-rho} follows from $\wt\rho((\sx_1, 1)) = \rho(\sx_1)$ for any $\sx_1\in \calX_1$.
\end{proofof}

\begin{lemma}\label{lem:sec-trajectory-convert}
    For a fix OPE problem $\mathfrak{g} = (M, \pieval, \piexp, \cF)$ and parameter $K$, suppose $(\wt M, \wtpieval, \wtpiexp, \wt\cF)$ to be the output of \pref{alg:replicator} with input to be $\mathfrak{g}$ and $K$. Suppose $\Da_h$ are admissible data collected according to $M$ and $\piexp$ and we feed $\Da_h$ into \pref{alg:convert}. Let $\wt{\mathbb{P}}_{M}$ to be the distribution of the output. Additionally, we suppose the distribution of a trajectory collected from $\wt M$ and $\wtpiexp$ to be $\mathbb{P}_{\wt{M}}$. Let $\cE_0$ to be the following set of trajectories:
    \begin{align}\label{eq:sec-trajectory-e0}
        \mathcal{E}_0 \ldef{}  \{\wt\tau = (\wt\sx_1, \wt\sa_1, \wt r_1, \wt\sx_2, \wt \sa_2, \wt r_2, \cdots, \wt\sx_{\wt{H}})\mid \exists 1\le h\le H-1, \wt{a}_{(h-1)K+1} = \cdots = \wt{a}_{(h-1)K+K-1} = 1\}.
    \end{align}
    Then for trajectory $\wt\tau = (\wt\sx_1, \wt\sa_1, \wt r_1, \wt\sx_2, \wt \sa_2, \wt r_2, \cdots, \wt\sx_{\wt{H}})\not\in \mathcal{E}_0$, we have
    \begin{align*} 
        \wt{\mathbb{P}}_{M}(\wt\tau) = \mathbb{P}_{\wt{M}}(\wt\tau).
    \end{align*}
\end{lemma}
\begin{proofof}[\pref{lem:sec-trajectory-convert}] We denote $\wt\tau_l = (\wt\sx_1, \wt\sa_1, \wt r_1, \wt\sx_2, \wt \sa_2, \wt r_2, \cdots, \wt\sx_{l})$ to be a partial trajectory till step $l$, and for a trajectory $\wt\tau$, we say $\wt\tau\in \wt\tau_l$ if the first $l$ steps of $\wt\tau$ is exactly $\wt\tau_l$. We denote
    \begin{align*}
        \wt{\mathbb{P}}_M(\wt\tau_l) = \sum_{\wt\tau\in\wt\tau_l}\wt{\mathbb{P}}_M(\wt\tau)\quad\text{and}\quad \mathbb{P}_{\wt{M}}(\wt\tau) = \sum_{\wt\tau\in\wt\tau_l}\mathbb{P}_{\wt{M}}(\wt\tau).
    \end{align*}
    We will prove by induction on $l$ that
    \begin{align}\label{eq:sec-trajectory-induction}
        \wt{\mathbb{P}}_{M}(\wt\tau_l) = \mathbb{P}_{\wt{M}}(\wt\tau_l).
    \end{align}
    When $l = 1$, this is true since $\wt{\mathbb{P}}_{M}(\wt\tau_1) = \wt\rho(\wt\sx_1) = \mathbb{P}_{\wt{M}}(\wt\tau_l)$. Next, to finish induction from $l$ to $l+1$, by the chain rule of probability, we only need to show that
    \begin{align}\label{eq:sec-trajectory-condition-probability}
        \wt{\mathbb{P}}_{M}(\wt{\sa}_{l}, \wt{r}_l, \wt{\sx}_{l+1}\mid \wt\tau_l) = \mathbb{P}_{\wt{M}}(\wt{\sa}_{l}, \wt{r}_l, \wt{\sx}_{l+1}\mid \wt\tau_l).
    \end{align}
    We write $l = hK + k$ with $1\le k\le K$, and $\wt\sx_l = (\sx_h, k)$ and $\wt\sx_{l+1} = (\sx_h', k+1)$ (or $(\sx_{h+1}', 1)$). If $k = K$, we have
    \begin{align*}
        \wt{\mathbb{P}}_{M}(\wt{\sa}_{l}, \wt{r}_l, \wt{\sx}_{l+1}\mid \wt\tau_l) = \piexp(\wt\sa_l\mid \sx_h)R(\wt{r}_l\mid \sx_h, \wt\sa_l)T(\sx_{h+1}'\mid \sx_h, \wt\sa_l).
    \end{align*}
    According to \pref{alg:convert}, as long as there exists some $(h-1)K+1\le t\le (h-1)K+K-1$ with $\wt{a}_t\neq 1$, we will have $(\wt\sx_l, \wt\sa_l, \wt r_l, \wt\sx_{l+1})\in \Da_h$, which indicates that
    \begin{align*}
        \mathbb{P}_{\wt{M}}(\wt{\sa}_{l}, \wt{r}_l, \wt{\sx}_{l+1}\mid \wt\tau_l) = \piexp(\wt\sa_l\mid \sx_h)R(\wt{r}_l\mid \sx_h, \wt\sa_l)T(\sx_{h+1}'\mid \sx_h, \wt\sa_l) = \wt{\mathbb{P}}_{M}(\wt{\sa}_{l}, \wt{r}_l, \wt{\sx}_{l+1}\mid \wt\tau_l).
    \end{align*}
    If $l = hK + k$ for some $h$ and $1\le k\le K-1$, then the transition model of $\wt M$ gives that
    \begin{align*}
        \mathbb{P}_{\wt{M}}(\wt{\sa}_{l}, \wt{r}_l, \wt{\sx}_{l+1}\mid \wt\tau_l) = \begin{cases}
            \frac{1}{2}\cdot\mathbb{I}(\wt{r}_l = 0) \cdot \mathbb{I}(\sx_{h}' = \sx_h), &\text{if}\quad \wt\sa_l = \pieval(\sx_l),\\
            \frac{1}{2}\cdot\mathbb{I}(\wt{r}_l = 0) \cdot d^\piexp_h(\sx_h'; M),&\text{if}\quad \wt\sa_l = \piother(\sx_l).
        \end{cases}
    \end{align*}
    Additionally, according to \pref{alg:convert}, action $1$ is taken with probability $\nicefrac{1}{2}$. Hence if $\wt\sa_l = \pieval(\sx_l)$, then the algorithm chooses $\sx_{h}' = \sx_h$ and $\wt{r} = 0$, we have
    \begin{align*}
        \text{LHS of }\pref{eq:sec-trajectory-condition-probability} = \frac{1}{2}\cdot\mathbb{I}(\wt{r}_l = 0)\mathbb{I}(\sx_{h}' = \sx_h).
    \end{align*} 
    And if $\wt{\sa}_l = \piother(\sx_l)$, the algorithm will sample $\sx_h'$ from $d^\piexp_h(\cdot; M)$. Hence we obtain
    \begin{align*}
        \text{LHS of }\pref{eq:sec-trajectory-condition-probability} =\frac{1}{2}\cdot\mathbb{I}(\wt{r}_l = 0)d^\piexp_h(\sx_h'; M).
    \end{align*}
    This finishes the proof of induction at $l+1$.
    
    \par Finally, by induction, \pref{eq:sec-trajectory-induction} holds for $l = \wt{H}$ as long as $\wt\tau\not\in \mathcal{E}_0$. Hence for $\wt\tau\not\in\mathcal{E}_0$ we always have $\wt{\mathbb{P}}_{M}(\wt\tau) = \mathbb{P}_{\wt{M}}(\wt\tau)$.
\end{proofof}

\begin{theorem}\label{thm:trajectory}
  Suppose for $\wt{M}$, algorithm $\At$ taking $n$ trajectories $\wt\tau_{1:n}$ and state-action value function class $\wt{\mathcal{F}}$ can output $\widehat{V}_{\wt\tau_{1:n}}\in [-1, 1]$ such that
    $$\mathbb{E}_{\wt\tau_{1:n}\sim \wt M}\left[\left|\widehat{V}_{\wt\tau_{1:n}} - V^{\wt M}(\wt \rho)\right|\right]\le \epsilon.$$
    Then taking $H^2n$ admissible dataset $D_{H^2n}$ and class of tuples of state-action value function together with $W$-function, \pref{alg:convert-admissible} can output $\epsilon$-close value to $\widehat{V}_{D_{H^2n}}$ such that
    $$\mathbb{E}_{\wt\tau_{1:n}\sim M}\left[\left|\widehat{V}_{D_n} - V^{M}(\rho)\right|\right]\le \epsilon + H2^{-K+2}.$$
\end{theorem}
\begin{proofof}[\pref{thm:trajectory}]
    We use $\wt\tau_{1:n}$ to denote $n$ trajectories $\wt\tau_1, \cdots, \wt\tau_n$, and let $e(\wt\tau_{1:n}) = |\widehat{V}_{\wt\tau_{1:n}} - V^{\wt M}(\wt \rho)|$. We have
    \begin{equation}\label{eq:sec-trajectory-eq-e}
        \sum_{\wt\tau_{1:n}} e(\wt\tau_{1:n})\prod_{j=1}^n\mathbb{P}_{\wt{M}}(\wt\tau_j) = \mathbb{E}_{\wt\tau_{1:n}\sim \wt M}\left[\left|\widehat{V}_{\wt\tau_{1:n}} - V^{\wt M}(\wt \rho)\right|\right]\le \epsilon.
    \end{equation} 
    We further notice $V(\wt\rho; \wt{M}) = V(\rho; M)$ from \pref{lem:sec-trajectory-concentrability}. Furthermore, since \pref{alg:convert-admissible} first transit the admissible data $D_{HKn}$ into trajectory data $\wt\tau_{1:n}$ according to \pref{alg:convert}, and then output $\widehat{V}_{\wt\tau_{1:n}}$ according to algorithm $\At$, we have
    $$\mathbb{E}_{\wt\tau_{1:n}\sim M}\left[\left|\widehat{V}_{D_n} - V^{M}(\rho)\right|\right] = \sum_{\wt\tau_{1:n}} \left|\widehat{V}_{\wt\tau_{1:n}} - V^{\wt M}(\wt \rho)\right|\prod_{j=1}^n\wt{\mathbb{P}}_M(\wt\tau_j) = \sum_{\wt\tau_{1:n}} e(\wt\tau_{1:n})\prod_{j=1}^n\wt{\mathbb{P}}_M(\wt\tau_j),$$
    where $\wt{\mathbb{P}}_M$ is defined in \pref{lem:sec-trajectory-convert}. Next, \pref{lem:sec-trajectory-convert} indicates that if $\wt\tau_j\not\in\mathcal{E}_0$ (where $\mathcal{E}_0$ is defined in \pref{eq:sec-trajectory-e0}), then $\wt{\mathbb{P}}_M(\wt\tau_j) = \mathbb{P}_{\wt{M}}(\wt\tau_j)$. Therefore, noticing that $e(\wt\tau_{1:n}) = |\widehat{V}_{\wt\tau_{1:n}} - V^{\wt M}(\wt \rho)|\le 2$, we have
    \begin{align*}
        \sum_{\wt\tau_{1:n}} e(\wt\tau_{1:n})\prod_{j=1}^n\wt{\mathbb{P}}_M(\wt\tau_j) & = \sum_{\wt\tau_{1:n}: \exists 1\le i\le n,\wt\tau_i\in\mathcal{E}_0} e(\wt\tau_{1:n})\prod_{j=1}^n\wt{\mathbb{P}}_M(\wt\tau_j) + \sum_{\wt\tau_{1:n}: \forall 1\le i\le n,\wt\tau_i\not\in\mathcal{E}_0} e(\wt\tau_{1:n})\prod_{j=1}^n\wt{\mathbb{P}}_M(\wt\tau_j)\\
        & = 2\sum_{\wt\tau_{1:n}: \exists 1\le i\le n,\wt\tau_i\in\mathcal{E}_0} \prod_{j=1}^n\wt{\mathbb{P}}_M(\wt\tau_j) + \sum_{\wt\tau_{1:n}: \forall 1\le i\le n,\wt\tau_i\not\in\mathcal{E}_0} e(\wt\tau_{1:n})\prod_{j=1}^n\mathbb{P}_{\wt M}(\wt\tau_j)\\
        & \le 2\sum_{i=1}^n \wt{\mathbb{P}}_{M}(\wt\tau_i\in \mathcal{E}_0) + \sum_{\wt\tau_{1:n}: \forall 1\le i\le n,\wt\tau_i\not\in\mathcal{E}_0} e(\wt\tau_{1:n})\prod_{j=1}^n\mathbb{P}_{\wt M}(\wt\tau_j)\\
        & \le 2\sum_{i=1}^n \wt{\mathbb{P}}_{M}(\wt\tau_i\in \mathcal{E}_0) + \sum_{\wt\tau_{1:n}} e(\wt\tau_{1:n})\prod_{j=1}^n\mathbb{P}_{\wt M}(\wt\tau_j)
    \end{align*}
    Further notice that for any $1\le j\le n$, we have
    $$\wt{\mathbb{P}}_{M}(\wt\tau\in \mathcal{E}_0) = 1 - (1 - 2^{-H})^K\le H2^{-K+1}.$$
    Therefore, using \pref{eq:sec-trajectory-eq-e}, we obtain
    \begin{align*}
        2\sum_{i=1}^n \wt{\mathbb{P}}_{M}(\wt\tau_i\in \mathcal{E}_0) + \sum_{\wt\tau_{1:n}} e(\wt\tau_{1:n})\prod_{j=1}^n\mathbb{P}_{\wt M}(\wt\tau_j)\le 2nH2^{-K+1} + \epsilon,
    \end{align*}
    which indicates that $\mathbb{E}_{\wt\tau_{1:n}\sim M}\left[\left|\widehat{V}_{D_n} - V^{M}(\rho)\right|\right]\le \epsilon + H2^{-K+2}$.
\end{proofof}

~\\
This theorem has the following direct corollary, indicating that any algorithm taking trajectory data as input cannot do policy evaluation in polynomial number of samples.
\begin{proofof}[\pref{thm:trajectory_lower_bound}]
    According to \pref{thm:admissible-stronger}, there exists a class $\mathcal{M}$ of MDPs, where each $M\in\mathcal{M}$ has bounded coverage $384H^3$. And any algorithm which takes $o(\nicefrac{H2^H}{\epsilon})$ number of admissible samples together with realizable class of tuples of state-action value function together with $W$-function induces estimation error $\nicefrac{\epsilon}{8H}$ in at least one MDP. We further carry the lifting in this section for any \MDP in $\mathcal{M}$, and suppose these lifting MDPs form the class $\wt{\mathcal{M}}$. Lemma \pref{lem:sec-trajectory-concentrability} indicates that every instance in $\wt{\mathcal{M}}$ has bounded coverage $768H^3$.

    Next, we will prove this corollary by contradiction. Suppose the algorithm $\At$ using $\tilde{o}(\nicefrac{H2^H}{\epsilon})$ trajectories together with realizable state-action value function class induces estimation error less than $\nicefrac{\epsilon}{16H}$ in every \MDP in $\wt{\mathcal{M}}$. Then after inserting $\At$ into \pref{alg:convert-admissible}, we form an algorithm which takes $\tilde{o}(\nicefrac{H2^H}{\epsilon})$ admissible data for each layer, together with realizable function class, and outputs an estimation to the value function. 

    \pref{thm:trajectory} indicates that this algorithm will induce $\nicefrac{\epsilon}{16H} + H2^{-K+2}$ error for all MDPs in $\mathcal{M}$. Hence taking $K = 2 + \log_2\nicefrac{16H^2}{\epsilon} = \log_2\nicefrac{64H^2}{\epsilon}$, this algorithm will induces estimation error less than $\nicefrac{\epsilon}{8H}$ in all MDPs in $\mathcal{M}$. Notice that with this choice of $K$, we have $\tilde{o}(\nicefrac{KH2^H}{\epsilon})\le \tilde{o}(\nicefrac{H2^H}{\epsilon})$, which contradicts to \pref{thm:admissible-stronger}.
\end{proofof}



\section{Upper Bound for Offline Policy Evaluation} \label{app: upper bound} 

\subsection{Setup}   
In previous sections we construct the lower bound assuming access to the $Q^{\pieval}$ function class. For the upper bound, we consider a slightly more challenging scenario, where the learner only has access to the $V^{\pieval}$ function class. It is not hard to see that learning with a $V^\pieval$ function class is more challenging than learning with $Q^\pieval$ because one can always reduce a $Q^\pieval$ function set to a $V$ function set by redefining $f(\sx)\leftarrow \E_{ \sa\sim \pieval(\cdot\mid\sx)) }[f(\sx,\sa)]$. 
For simplicity, we assume that $\pieval$ is deterministic, though the extension to the stochastic case is straightforward. See \pref{sec: BVFT alg} for the discussion.  

Similar to \cite{xie2021batch}, we first establish results for the case where the function set only \emph{approximately} realizes the true $V^\pieval$. The result for the fully realizable setting can be easily deduced from it. Formally, we assume that the learner is given a function set $\gF$ that consists of mappings $\gX\to [-H,H]$ with the following approximate realizability guarantee. 
\begin{assumption}[Approximate value function realizability]\label{ass: approximate realizability} There exists an $\fstar\in\gF$ such that  
$\sup_{\sx\in\gX} |V^{\pieval}(\sx) - \fstar(\sx)|\leq \eappr$. 
\end{assumption}

Besides, the learner is given an offline dataset $\gD$, which consists of $n$ tuples of $(\sx, \sa,\sr,\sx')$ with $(\sx, \sa)$ drawn from $\mu$, and $r\sim r(\sx, \sa)$ and $s'\sim T(\cdot\mid\sx, \sa)$ are sampled according to the MDP's reward function and transition model. 

\subsection{Algorithm}\label{sec: BVFT alg}  


\begin{algorithm}[t]
    \caption{Batch Value-Function Tournament for Policy Evaluation \citep{xie2021batch}} 
    \label{alg: BVFT} 
\begin{algorithmic}[1] 
    \Require Evaluation policy \(\pieval\), Offline Dataset $\cD$ consisting of $n$ tuples of the form $(\sx, \sa,\sr,\sx')$.  
    \State Among all samples $(\sx, \sa, \sr,\sx')$ in $\gD$, only keep those such that $a=\pieval(\sx)$, and discard all others. We denote the new dataset as $\gD'=\{(\sx_i, \sr_i, \sx_i')\}_{i=1}^{n'}$, where we omit $\sa_i$'s since $\sa_i$ is always equal to 
$\pieval(\sx_i)$ in this dataset.   \label{line: preprocess}
   \State Compute 
    \begin{align} 
        \fout = \argmin_{f\in\gF}\max_{f'\in\gF} \max_h ~ \|f- \That_{\gG(f,f')}f\|_{\hat{\nu}, h}   \label{eq: BVFT}
    \end{align}
    where 
    \begin{align} 
        \gG(f,f') &\ldef{}  \bigg\{g: \gX\rightarrow \mathbb{R} ~\bigg|~ g(\sx) = g(\sy) \text{ \ \ if } f(\sx)=f(\sy)\text{ and }  f'(\sx)=f'(\sy) \text{ and } h(\sx)=h(\sy) \bigg\}, \tag{\scalebox{1.0}{$h(\sx)\in[H]$ denotes the layer at which $\sx$ lies}} \\
        \That_\gG f &\ldef{}  \argmin_{g\in\gG}\sum_{i=1}^{n'}\left(g(\sx_i) - \sr_i - f(\sx_i')\right)^2,     \label{eq: BVFT alg defs} \\ 
        \hat{\nu}(\sx)&\ldef{}  \frac{1}{n}\sum_{i=1}^{n'}\mathbb{I}\{\sx_i=\sx\}.  
    \end{align} 
    \State \textbf{Return}  \(\wh f\). 
   \end{algorithmic} 
\end{algorithm}

Our algorithm for this setting is presented in \pref{alg: BVFT}, which is an adaptation of the BVFT algorithm \citep{xie2021batch} to the case of policy evaluation. In the beginning of the algorithm, the dataset $\gD$ is pre-processed so that only $(\sx, \sa, \sr, \sx')$ samples with $\sa=\pieval(\sx)$ are kept (\pref{line: preprocess}). The core of the algorithm is to solve the min-max problem in \pref{eq: BVFT}. The high-level idea of it is that for every pair of functions $f, f'\in\gF$, the algorithm creates a ``tabular problem'' by aggregating states with the same $(f(\sx), f'(\sx))$ value, and estimates the Bellman error for this tabular problem. Intuitively, this is probably the best the learner can do, since besides the value of $(f(\sx))_{f\in\gF}$, the learner has no other ways to distinguish states in the large state space.  The output function 
$\fout$ is the one that always attains a small Bellman error estimate no matter what the other function it is paired with. For more explanation on this min-max formulation, we refer the reader to \cite{xie2021batch} or the amazing talk by \citet{jiang2022batch}. The key differences with the algorithm of \cite{xie2021batch} are the following: 
\begin{itemize}
    \item We deal with policy evaluation for $\pieval$, and our function class consists of $V^\pieval$ functions, while \cite{xie2021batch} deal with policy optimization, and their function class consists of $Q^\star$ functions. For this reason, we have a \emph{preprocessing step} in \pref{line: preprocess} of \pref{alg: BVFT}, which removes data samples whose action is not generated by $\pieval$. These samples are irrelevant to our policy evaluation task. 
    \item We consider the finite-horizon setting while \cite{xie2021batch} considers the discounted infinite-horizon setting. Therefore, different from theirs, Our aggregation is performed in a layer-by-layer manner, and only states in the same layer can be aggregated. 
\end{itemize}

If $\pieval$ is stochastic, we perform the pre-processing step (\pref{line: preprocess}) in the following way: for each sample $(\sx, \sa, \sr, \sx')\in \gD$, sample $\sa'\sim \pieval(\cdot\mid \sx)$. If $\sa=\sa'$, then keep this sample; otherwise discard this sample.

Our upper bound result is stated in the following theorem, whose proof is provided in \pref{app: BVFT definitions} to \pref{app: BVFT proof of main}.  
\begin{theorem}
\label{thm:BVFT_upper_bound}
Let $\fout\in\gF$ be the output of the BVFT algorithm given in \pref{alg: BVFT}. Let $\Phi(f,f')$ be the state aggregation scheme determined by $f$ and $f'$ (see \pref{def: partition induced by two f} for the precise definition), and let $\aggC=\max_{f,f'\in \gF}\aggC_{\epsilon^2/H^2}(M, \Phi(f,f'), \mu)$. For given $\delta>0$, if $n\geq \widetilde{\Omega}\left(\frac{\aggC^2 H^6\log(|\gF|/\delta)}{\epsilon^4}\right)$, then with probability at least $1-\delta$, 
\begin{align*}
     \abs*{\E_{\sx\sim \rho} \brk{V^\pieval(\sx) - \fout(\sx)}} \leq \order(\epsilon). 
\end{align*}
\end{theorem} 

Because of the state aggregation procedure in BVFT, the sample complexity upper bound in \pref{thm:BVFT_upper_bound} depends on the concentrability coefficient of the aggregated \MDP (i.e., $\aggC$) rather than that of the original MDP. Notice also that the sample complexity scales with $\frac{1}{\epsilon^4}$ instead of the more common $\frac{1}{\epsilon^2}$. This is similar to \cite{xie2021batch} and is because we divide the state space into $\order(\frac{1}{\epsilon^2})$ aggregations, each of which consists of states having the same value functions up to an accuracy of $\epsilon$.  Our bound have a smaller dependence on the horizon length $H$, but this is simply because we assume the range of the value function is $[-1,1]$ while they assume it to be $[-H, H]$.

Finally, we provide some implications of \pref{thm:BVFT_upper_bound}. First, as pointed out previously, we have $\aggC\leq \Cpush$ (\pref{lem: C and Cpush}). Second, in the case that the data is admissible with offline policy $\pioff$, and $\frac{1}{\pioff(\pieval(\sx)|\sx)}\leq \CpushA$ for all $\sx$, we have $\aggC\leq (\CpushA)^H$ (\pref{lem: C and BH}). Interestingly, a sample complexity bound of order $(\CpushA)^H$ is also the case if we use \emph{importance sampling} to perform offline policy evaluation. The difference is that importance sampling does not require access to any function class, but needs the data to be trajectories, while \BVFT requires access to a function class, but only needs the data to be admissible.

For the remaining of this section, we provide a proof for \pref{thm:BVFT_upper_bound}. 


\subsection{Definitions}\label{app: BVFT definitions}



\begin{definition}[partial offline data distribution $\nu(\sx)$] Given the offline data distribution \(\mu \in \Delta(\gX \times \gA)\), we define the partial distribution \(\nu \in \Delta(\gX)\) such that $\nu(\sx)=\mu(\sx,\pieval(\sx))$. 
\end{definition}


\begin{definition}[aggregation schemes $\Phi(f,f')$, $\Phi_h(f,f')$ and maximum partition number $\Phimax$]\label{def: partition induced by two f}
    Define $\Phi(f,f')$ as the state aggregation scheme (see \pref{sec: state aggregation}) where $\sx, \sy$ belongs to the same partition if and only if $f(\sx)=f(\sy)$ and $f'(\sx)=f'(\sy)$ and $\sx, \sy$ are in the same layer. Let $\Phi_h(f,f')\subset \Phi(f,f')$ be the set of partitions in layer $h$. Define $\Phimax = \max_{f,f'\in\gF}\max_h|\Phi_h(f,f')|$. 
\end{definition}

\begin{definition}[aggregation $\Phistar$, aggregated transition $\agg{T}$, occupancy $\agg{d}$, and offline distribution $\agg{\nu}$]
\label{def: aggregated transition} 
Consider the partition $\Phistar=\Phi(\fout, \fstar)$, where $\fout$ is the output of \pref{alg: BVFT} and $\fstar$ is defined in \pref{ass: approximate realizability}. Let $\phi, \phi'\in\Phistar$,  Define 
\begin{align*}
    \agg{T}(\phi'\mid{}\phi) = \frac{\sum_{\sx\in\phi}\sum_{\sx'\in\phi'} \nu(\sx)T(\sx' \mid{} \sx,\pieval)}{\sum_{\sx\in\phi} \nu(\sx)}. 
\end{align*}
Furthermore, let $\agg{d}(\phi)$ be the occupancy measure of $\pieval$ in the aggregated MDP. That is, $\agg{d}$ follows the recursive definition below: 
\begin{align*}
    \forall \phi'\in\Phistar_{h+1},\quad \agg{d}(\phi') &= \sum_{\phi\in\Phistar_h}\agg{d}(\phi)\agg{T}(\phi'\mid{}\phi), \quad \text{with}\  \agg{d}(\phi) = \frac{1}{H}\sum_{\sx\in\phi} \rho(\sx) \text{\ for\ } \phi\in \Phistar_1. 
\end{align*} 
Also, define $\agg{\nu}(\phi) = \sum_{\sx\in\phi} \nu(\sx)$.
\end{definition}
\begin{definition}[aggregated concentrability $\aggC_\epsilon^\star$]\label{def: cbar upper bound} 
The aggregated concentrability with respect to the aggregation $\Phistar=\Phi(\hat{f}, f^\star)$ (defined in \pref{def: aggregated transition}) is defined as
\begin{align}
     \aggCstar_\epsilon = \max_h \max\left\{\frac{\sum_{\phi\in\gI}\agg{d}(\phi)}{\sum_{\phi\in\gI}\agg{\nu}(\phi)}: \quad \gI\subset \Phistar_h, \quad \sum_{\phi\in\gI} \agg{d}(\phi) \geq \epsilon \right\}.     \label{eq: Cupper def}
\end{align}

   
\end{definition} 




\begin{definition}[projection operators $\gT_\gG f$ and $\hat{\gT}_\gG f$] \label{def: Bellman operator def}
    Let $f: \gX\rightarrow \mathbb{R}$, and let $\gG$ be any function set that consists of functions of the form $\gX \rightarrow\mathbb{R}$. Define 
    \begin{align*}
       \gT_{\gG} f &= \argmin_{g\in\gG} \sum_{s,\sx'} 
\nu(\sx)T(\sx'\mid{} \sx, \pieval) \left(g(\sx) - r(\sx,\pieval) -  f(\sx')\right)^2,  \\
       \That_\gG f &= \argmin_{g\in\gG}\frac{1}{n}\sum_{i=1}^{n'}\left(g(\sx_i) - \sr_i - f(\sx_i')\right)^2. 
    \end{align*}
    
\end{definition}

\begin{definition}[weighted norm $\|g\|_{w,h}$]\label{def: bar norm}
    Let $g\in \gG(f, f')$ and $\Phi=\Phi(f, f')$ for some $f,f'\in\gF$. Let $w: \Phi\to \mathbb{R}_{\geq 0}$ be arbitrary. With abuse of notation, define 
    $    \|g\|_{w, h} = \sqrt{\sum_{\phi\in\Phi_h} w(\phi)g(\phi)^2 } $, 
    where $g(\phi)$ is such that $g(\sx)=g(\phi)$ for all $\sx\in\phi$. 
\end{definition}

\begin{definition}[estimation error $\estat$]\label{def: estat}
    $\estat=\sqrt{\frac{\Phimax\log (n\Phimax|\gF|/\delta)}{n}}$, where $n$ is the number of offline samples, and $\Phimax$ is defined in \pref{def: partition induced by two f}. 
\end{definition}

We next establish a few properties of state aggregation. 


\begin{lemma}\label{lem: the form of TGf}
Let $g=\gT_{\gG(f,f')}f$ for some $f,f'\in\gF$. Fix a layer $h$. Let $g_{\alpha\beta}$ be the value of $g(\sx)$ for those $\sx\in\gX_h$'s
such that $f(\sx)=\alpha$ and $f'(\sx)=\beta$. Then $g_{\alpha\beta}$ has the following form: 
\begin{align*}
    g_{\alpha\beta} = \frac{\sum_{\sx\in\gX_h: f(\sx)=\alpha, f'(\sx)=\beta} \prn*{\nu(\sx) \gT f(\sx)}}{\sum_{\sx\in\gX_h: f(\sx)=\alpha, f'(\sx)=\beta} \prn*{\nu(\sx)}}. 
\end{align*} 

\end{lemma}
\begin{proofof}[\pref{lem: the form of TGf}]
      Recall from \pref{def: Bellman operator def} that $g$ is the minimizer of $\E_{\sx\sim \nu} \E_{\sx'\sim T(\cdot\mid{} \sx, \pieval)} 
 \left(g(\sx) - r(\sx,\pieval) -  f(\sx')\right)^2$ within $\gG(f,f')$. The derivative of this objective with respect to $g_{\alpha \beta}$ is 
\begin{align*}
    &\E_{\sx\sim\nu} \E_{\sx'\sim T(\cdot\mid{} \sx, \pieval)} 2(g_{\alpha \beta} - r(\sx,\pieval) - f(\sx')) \mathbb{I}\{ f(\sx)=\alpha, f'(\sx)=\beta, h(\sx)=h \} \\
    &= 2g_{\alpha\beta}\sum_{\sx\in\gX_h:~f(\sx)=\alpha, f'(\sx)=\beta} \nu(\sx) - 2\sum_{\sx\in \gX_h: f(\sx)=\alpha, f'(\sx)=\beta} \nu(\sx)\left(r(\sx,\pieval) + \sum_{\sx'} T(\sx'\mid{} \sx, \pieval)f(\sx')\right) \\
    &= 2g_{\alpha\beta}  \sum_{\sx\in\gX_h:~f(\sx)=\alpha, f'(\sx)=\beta} \nu(\sx) - 2\sum_{\sx\in \gX_h: f(\sx)=\alpha, f'(\sx)=\beta} \nu(\sx)\gT f(\sx). 
\end{align*}
Setting this to be zero gives the desired expression of $g_{\alpha\beta}$. 
\end{proofof}

\begin{lemma}\label{lem: property of G}
    For any $f\in\gF$, we have that $\max_{f'\in\gF}\|f-\gT_{\gG(f,f')}f \|_{\nu, h} \leq \|f-\gT f\|_{\nu, h}$. 
\end{lemma}
\begin{proofof}[\pref{lem: property of G}]
   Fix $f, f'\in\gF$ and fix $h\in[H]$. Let $g=\gT_{\gG(f,f')}f$ and let $g_{\alpha\beta}$ be the value of $g(\sx)$ for $\sx\in \gX_h$ such that $f(\sx)=\alpha$ and $f'(\sx)=\beta$. 
   
   Define $\nu_{\alpha\beta}=\sum_{\sx\in\gX_h}\nu(\sx)\mathbb{I}\{f(\sx)=\alpha \text{\ and\ } f'(\sx)=\beta\}$.   Then by definition, we have 
    \begin{align*}
        \|f - \gT_{\gG(f,f')}f\|_{\nu, h}^2 &= \sum_{\alpha,\beta} \nu_{\alpha\beta}(\alpha - g_{\alpha\beta})^2 \\ 
        &= \sum_{\alpha,\beta} \nu_{\alpha\beta}\left(\alpha - \frac{1}{\nu_{\alpha\beta}}\sum_{\sx\in\gX_h: f(\sx)=\alpha, f'(\sx)=\beta} \nu(\sx) \gT f(\sx)\right)^2   \tag{using \pref{lem: the form of TGf}} \\
        &\leq \sum_{\alpha,\beta} \nu_{\alpha\beta} \frac{1}{\nu_{\alpha\beta}}\sum_{\sx\in\gX_h: f(\sx)=\alpha, f'(\sx)=\beta}\nu(\sx) \left(\alpha-\gT f(\sx)\right)^2  \tag{Jensen's inequality} \\
        &= \sum_{\sx\in\gX_h} \nu(\sx)\left(f(\sx) - \gT f(\sx)\right)^2 \\
        &= \|f-\gT f\|_{\nu, h}^2. 
    \end{align*}
\end{proofof}

\subsection{Supporting Technical Results} 

\begin{lemma}\label{lem: empirical and true}
With probability at least $1-\delta$, for all $f, f'\in \gF$ and all $g\in \gG(f,f')$ such that $\sup_{\sx}|g(\sx)|\leq 1$, it holds that
\begin{align*}
    \|g\|_{\nu,h} &\leq \sqrt{2} \|g\|_{\hat{\nu},h} + \order\left(\estat\right), \\
    \|g\|_{\hat{\nu},h} &\leq \sqrt{2}\|g\|_{\nu,h} + \order\left(\estat\right). 
\end{align*}
(Recall the definition of $\estat$ in \pref{def: estat})
\end{lemma} 
\begin{proofof}[\pref{lem: empirical and true}]
    Fix $f, f'$ and $g$. 
    By Bernstein's inequality, with probability at least $1-\delta'$, 
    \begin{align*}
        \left|\|g\|_{\nu, h}^2 - \|g\|_{\hat{\nu}, h}^2\right| &= \left| 
\frac{1}{n} \sum_{i=1}^{n'} \left(g(\sx_i)^2 - \sum_{s} \nu(\sx) g(\sx)^2\right) \right| \\
&\leq \order\left(\sqrt{\frac{1}{n} \sum_{s} \nu(\sx)g(\sx)^4 \log(1/\delta')} + \frac{\log(1/\delta')}{n}\right) \\
&\leq  \sum_s \nu(\sx) g(\sx)^2 + \order\left(\frac{\log(1/\delta')}{ n}\right)   \tag{AM-GM} \\
&=  \|g\|_{\nu, h}^2 + \order\left(\frac{\log(1/\delta')}{n}\right). 
    \end{align*}
    Rearranging this gives 
    \begin{align*}
        \|g\|^2_{\nu, h} \leq 2\|g\|^2_{\hat{\nu}, h} + \order\left(\frac{\log(1/\delta')}{ n}\right)\qquad \text{and}\qquad \|g\|^2_{\hat{\nu}, h} \leq 2\|g\|^2_{\nu, h} + \order\left(\frac{\log(1/\delta')}{ n}\right). 
    \end{align*}
    Next, we take union bounds over $f, f'$ and $g$. Notice that for every pair of $(f, f')$, the value of $g\in\gG(f,f')$ on $x\in\gX_h$ is determined by the values of $\{g(\phi)\}_{\phi\in\Phi_h(f,f')}$, where $g(\phi)$ is the value of $g(\sx)$ for $x\in\phi$. Therefore, an $\epsilon$-net of $\gG(f,f')$ on layer $h$ can be constructed by discretizing each value of $\{g(\phi)\}_{\phi\in\Phi_h(f,f')}$, and its size is at most $(1/\epsilon)^{\order(|\Phi_h(f,f')|)} \leq (1/\epsilon)^{\order(\Phimax)}$. It suffices to pick $\epsilon=\frac{1}{n}$ and bound the discretization error by $\order(\frac{1}{n})$. Overall, the union bound is taken over $|\gF|^2 n^{\Phimax}$ instances. Therefore, we pick $\delta'=\frac{\delta}{|\gF|^2 n^{\Phimax}}$, which gives 
    \begin{align*}
        \|g\|^2_{\nu, h} &\leq 2|g\|^2_{\hat{\nu}, h} + \order\left(\frac{\Phimax\log(n|\gF|/\delta)}{ n}\right),  \quad \text{and}\\
        \|g\|^2_{\hat{\nu}, h} &\leq 2\|g\|^2_{\nu, h} + \order\left(\frac{\Phimax\log(n|\gF|/\delta)}{ n}\right). 
    \end{align*}
    Finally, taking square root on both sides and recalling that $\estat=\sqrt{\frac{\Phimax\log(n\Phimax|\gF|/\delta)}{n}}$ finishes the proof. 
    
\end{proofof}

\begin{lemma}\label{lem: BVFT operator closeness} 
    With probability at least $1-\delta$, for all $f, f'\in\gF$, and \(h \leq H\), 
    \begin{align*}
    \|\gT_{\gG(f,f')}f - \That_{\gG(f,f')} f\|_{\nu, h} \le \cO\left(\estat\right).
    \end{align*}
\end{lemma}
\begin{proofof}[\pref{lem: BVFT operator closeness} ]
    We first fix the function pair $f,f'$ and define additional notation. Let $\Phi=\Phi(f,f')$ and $\gG=\gG(f,f')$. For every $\phi\in\Phi_h$, define  
    \begin{align*}
        Y(\phi) &= \sum_{\sx\in \phi}\nu(\sx) \left(r(\sx,\pieval) + \sum_{\sx'} T(\sx'\mid{} \sx, \pieval) f(\sx')\right), \ \  && Z(\phi) = \sum_{\sx\in\phi} \nu(\sx),  \\
        \Yhat(\phi) &= \frac{1}{n} \sum_{i=1}^{n'}  \mathbb{I}\{\sx_i\in \phi\}  \left(r(\sx_i, \pieval) + f(\sx_{i}')\right), &&  \Zhat(\phi) = \frac{1}{n}\sum_{i=1}^{n'}  \mathbb{I}\{\sx_i\in \phi\}. 
    \end{align*} 
Additionally, let $\delta'=\frac{\delta}{|\gF|^2\Phimax}$,  $\epsilon_0 = \frac{\log(1/\delta')}{n}$, and $\Phi_h'=\{\phi\in\Phi_h:~ Z(\phi)\geq \epsilon_0\}$. 
        
Using \pref{lem: the form of TGf} together with definitions above, we get that for all $\sx\in \phi$, 
    \begin{align*}
        \gT_{\gG}f(\sx) = \frac{Y(\phi)}{Z(\phi) } \quad  \text{and} \quad   
        \That_{\gG}f(\sx) = \frac{\Yhat(\phi)}{\Zhat(\phi) }. 
    \end{align*}
    Thus, we have 
    \begin{align} 
         &\|\gT_\gG f - \That_\gG f\|_{\nu,h}^2 \\
         &=\sum_{\sx\in\gX_h} \nu(\sx) \left(\gT_\gG f(\sx) - \That_\gG f(\sx)\right)^2 \nonumber \\
         &=  \sum_{\phi\in\Phi_h} \sum_{\sx\in\phi} \nu(\sx) \left(\frac{Y(\phi)}{Z(\phi)} - \frac{\Yhat(\phi)}{ \Zhat(\phi) }\right)^2 \nonumber \\
         &\leq  \sum_{\phi\in\Phi_h'} \sum_{\sx\in\phi} \nu(\sx) \left(\frac{Y(\phi)}{Z(\phi)} - \frac{\Yhat(\phi)}{ \Zhat(\phi) }\right)^2 + 4\sum_{\phi\in\Phi_h \setminus \Phi_h'} \sum_{\sx\in\phi} \nu(\sx)H^2   \tag{$\frac{Y(\phi)}{Z(\phi)} - \frac{\Yhat(\phi)}{ \Zhat(\phi) } \in [-2, 2]$} \\ 
         &\leq 2\sum_{\phi\in\Phi_h'} Z(\phi) \left(\frac{Y(\phi)}{Z(\phi)} - \frac{\Yhat(\phi)}{Z(\phi) }\right)^2 + 2  \sum_{\phi\in\Phi_h'} Z(\phi)  \left(\frac{\Yhat(\phi)}{Z(\phi) } - \frac{\Yhat(\phi)}{\Zhat(\phi) }\right)^2  + 4\sum_{\phi\in\Phi_h\setminus \Phi_h'} Z(\phi)     \\
         &\leq 2\sum_{\phi\in\Phi_h'} \frac{(Y(\phi) - \Yhat(\phi))^2}{Z(\phi)} + 2\sum_{\phi\in\Phi_h'}  \frac{\Yhat(\phi)^2}{\Zhat(\phi)^2} \frac{(Z(\phi) - \Zhat(\phi))^2}{Z(\phi)} + 4|\Phi_h|\epsilon_0 \tag{because $Z(\phi)\leq \epsilon_0$ for all $\phi\in \Phi_h\setminus \Phi_h'$}  \nonumber    \\  
         &\leq 2\sum_{\phi\in\Phi_h'} \frac{(Y(\phi) - \Yhat(\phi))^2 +  (Z(\phi) - \Zhat(\phi))^2}{Z(\phi)} + \frac{4\Phimax\log(1/\delta')}{n}, 
         & \label{eq: bound T and That diff} 
    \end{align}
   where the last line follows by observing that  \(\nicefrac{\Yhat(\phi)}{\Zhat(\phi)} \leq 1\) for any \(\phi \in \Phi\). 

    Next, using Bernstein's inequality for any \(\phi \in \Phi_h\),  with probability at least $1-\delta'$,  we have 
    \begin{align*}
        |Y(\phi) - \Yhat(\phi)| &\leq \order\left(\sqrt{\frac{Z(\phi)}{n} \log(1/\delta')} + \frac{\log(1/\delta')}{n}\right), \\
        |Z(\phi) - \Zhat(\phi)| &\leq \order\left(\sqrt{\frac{Z(\phi)}{n} \log(1/\delta')} + \frac{\log(1/\delta')}{n}\right).
    \end{align*}
    Plugging the above in \pref{eq: bound T and That diff}, and using a union bound over $\phi\in\Phi_h'$, we get that with probability at least $1-\delta'\Phimax$, 
    \begin{align*}
        \|\gT_\gG f - \That_\gG f\|_{\nu,h}^2 
        &\leq \cO\left( \sum_{\phi\in\Phi_h'}  \frac{\log(1/\delta')}{n} + \frac{\log^2(1/\delta')}{n^2Z(\phi)} + \frac{\Phimax\log(1/\delta')}{n} \right) \\ 
        &\leq \cO\left(\frac{ \Phimax\log(1/\delta')}{n} + \frac{\Phimax\log^2(1/\delta')}{n^2\epsilon_0} + \frac{\Phimax\log(1/\delta')}{n}\right)  \tag{for $\phi\in \Phi_h'$, $Z(\phi)\geq \epsilon_0$} \\
        &=  \cO\left(\frac{ \Phimax\log(|\gF|\Phimax/\delta)}{n}\right). \tag{by the definition of $\epsilon_0$ and $\delta'$}
    \end{align*}
    Finally, using a union bound over $(f,f')\in\gF\times \gF$ and recalling that $\estat=\sqrt{\frac{\Phimax\log(n|\gF|\Phimax/\delta)}{n}}$ finishes the proof. 
\end{proofof}

\begin{lemma}\label{lem: BVFT term 1} Let $\fout$ be the output of \pref{alg: BVFT} and $\fstar$ be defined in \pref{ass: approximate realizability}. 
\begin{enumerate}[label=\((\alph*)\)] 
\item   $\|\fout - \gT_{\gG(\fout, \fstar)} \fout\|_{\nu, h} \le \cO\left(\eappr + \estat\right)$. 
\item 
    $\|\gT_{\gG(\fout, \fstar)} \fout - \gT_{\gG(\fout, \fstar)} \fstar\|_{\agg{d}, h} \leq \|\fout-\fstar\|_{\agg{d}, h+1}$. 
\end{enumerate}
(recall the definition of $\agg{d}$ in \pref{def: aggregated transition})

\end{lemma}
\begin{proofof}[\pref{lem: BVFT term 1}] 
We prove the two parts separately below. 
\begin{enumerate}[label=\((\alph*)\)] 

\item Note that 
    \begin{align*}
        \hspace{0.3in}&\hspace{-0.3in}\max_h \|\fout - \gT_{\gG(\fout, \fstar)} \fout\|_{\nu, h} \\
        &\leq \max_h 
 \|\fout - \That_{\gG(\fout, \fstar)}\fout\|_{\nu, h} + \|\That_{\gG(\fout, \fstar)}\fout - \gT_{\gG(\fout, \fstar)}\fout\|_{\nu, h}  \tag{triangle inequality}\\
        &\le \max_{f'\in\gF} \max_h \|\fout - \That_{\gG(\hat f, f')} \fout\|_{\nu, h} + \order\left(\estat\right)   \tag{by Lemma~\ref{lem: BVFT operator closeness}} \\
         &\le \sqrt{2}\max_{f'\in\gF}\max_h  \|\fout - \That_{\gG(\hat f, f')} \fout\|_{\hat{\nu}, h} + \order\left(\estat\right)  \tag{by Lemma~\ref{lem: empirical and true}} \\
        &\leq \sqrt{2}\max_{f'\in\gF} \max_h \|f^\star - \That_{\gG(f^\star, f')} f^\star\|_{\hat{\nu}, h} + \order\left(\estat\right) \tag{by the choice of $\fout$ in \pref{eq: BVFT}}\\
         &\leq 2\max_{f'\in\gF} \max_h \|f^\star - \That_{\gG(f^\star, f')} f^\star\|_{\nu, h} + \order\left(\estat\right) \tag{by Lemma~\ref{lem: empirical and true}}\\
        &\le 2\max_{f'\in\gF} \max_h \|f^\star - \gT_{\gG(f^\star, f')} f^\star\|_{\nu, h} + \order\left(\estat\right) \tag{by Lemma~\ref{lem: BVFT operator closeness}} \\
        &\le 2\max_h \|\fstar - \gT \fstar\|_{\nu,h} + \order\left(\estat\right).   \tag{by Lemma~\ref{lem: property of G}}
    \end{align*} 

 \item For the ease of notation, let $\gG=\gG(\fout, \fstar)$. Notice that by \pref{lem: the form of TGf}, 
    \begin{align*}
        \gT_{\gG}\fout(\phi) - \gT_{\gG}\fstar(\phi) &= \frac{\sum_{\sx \in \phi} \nu(\sx) \sum_{\sx'\in\gX_{h+1}}T(\sx'\mid{} \sx,  \pieval)\left(\fout(\sx') -  \fstar(\sx')\right)}{\sum_{\sx \in \phi} \nu(\sx)}. 
    \end{align*}
    Therefore, 
    \begin{align*}
        &\|\gT_{\gG} \fout - \gT_{\gG} \fstar\|^2_{\agg{d}, h}  \\
        &= \sum_{\phi\in\Phi_h }\agg{d}(\phi)\left( \gT_{\gG}\fout(\phi) - \gT_{\gG}\fstar(\phi) \right)^2 \\
        &= \sum_{\phi\in\Phi_h}\agg{d}(\phi)\left( \frac{\sum_{s \in \phi} \nu(\sx)\sum_{\sx'\in\gX_{h+1}}T(\sx'\mid{} \sx, \pieval)\left(\fout(\sx') -  \fstar(\sx')\right)}{\sum_{s \in \phi} \nu(\sx)} \right)^2  \\
        &\leq \sum_{\sx'\in\gX_{h+1}} 
  \sum_{\phi\in\Phi_h} \agg{d}(\phi)  \left( \frac{\sum_{s \in \phi} \nu(\sx)T(\sx'\mid{} \sx, \pieval) }{\sum_{s \in \phi} \nu(\sx)} \right)\left(\fout(\sx') -  \fstar(\sx')\right)^2 \tag{Jensen's inequality}  \\
        &= \sum_{\phi'\in\Phi_{h+1}} \sum_{\sx'\in\phi'} 
  \sum_{\phi\in\Phi_h} \agg{d}(\phi)  \left( \frac{\sum_{s \in \phi} \nu(\sx)T(\sx'\mid{} \sx, \pieval) }{\sum_{s \in \phi} \nu(\sx)} \right)\left(\fout(\phi') -  \fstar(\phi')\right)^2    \\
        &= \sum_{\phi'\in\Phi_{h+1}} 
  \sum_{\phi\in\Phi_h} \agg{d}(\phi)  \agg{T}(\phi'\mid{}\phi)\left(\fout(\phi') -  \fstar(\phi')\right)^2   \\
  &\leq \sum_{\phi'\in\Phi_{h+1}} 
  \agg{d}(\phi') \left(\fout(\phi') -  \fstar(\phi')\right)^2   \\
        &= \|\fout - \fstar\|_{\agg{d}, h+1}^2. 
    \end{align*}

 \end{enumerate}
\end{proofof}

\begin{lemma}\label{lem: conversion lemma}
    With the $\Phistar$ and $\aggCstar_\epsilon$ defined in \pref{def: aggregated transition} and \pref{def: cbar upper bound}, there exist $\gJ_1,\ldots, \gJ_H$ such that 
    \begin{itemize}
        \item $\gJ_h \subset \Phistar_h$ 
        \item $\sum_{\phi\in \gJ_h} \agg{d}(\phi) < \epsilon$
        \item $\max_h \max_{\phi \in \Phi_h\setminus \gJ_h}\frac{\agg{d}(\phi)}{\agg{\nu}(\phi)}\leq \aggCstar_\epsilon$
    \end{itemize}
\end{lemma} 
\begin{proofof}[\pref{lem: conversion lemma}]
    Define 
    \begin{align}
        \gI_h = \argmax\left\{\frac{\sum_{\phi\in\gI}\agg{d}(\phi)}{\sum_{\phi\in\gI}\agg{\nu}(\phi)}: \quad \gI\subset \Phistar_h, \quad \sum_{\phi\in\gI} \agg{d}(\phi) \geq \epsilon \right\}.   \label{eq: IH def}
    \end{align}
    If $\gI_h$ has more than one solution, we pick one such that $|\gI_h|$ is the smallest. 
    By the definition of $\aggCstar_\epsilon$, we know that $\frac{\sum_{\phi\in\gI_h}\agg{d}(\phi)}{\sum_{\phi\in\gI_h}\agg{\nu}(\phi)}\leq \aggCstar_\epsilon$ for all $h$.  

    Assume $\gI_h = \{\phi_{h,1}, \ldots, \phi_{h,N_h}\}$ where $N_h=|\gI_h|$, and assume without loss of generality that $$\frac{\agg{d}(\phi_{h,1})}{\agg{\nu}(\phi_{h,1})} \geq \frac{\agg{d}(\phi_{h,2})}{\agg{\nu}(\phi_{h,2})} \geq \cdots \geq \frac{\agg{d}(\phi_{h,N_h})}{\agg{\nu}(\phi_{h,N_h})}.$$ 
    If $N_{h}=1$, it is  easy to see that $\phi_{h,1} = \argmax_{\phi \in \Phistar_h}\frac{\agg{d}(\phi)}{\agg{\nu}(\phi)}$. This is because if not, then $$\left\{\phi_{h,1}, \  \argmax_{\phi \in \Phistar_h}\frac{\agg{d}(\phi)}{\agg{\nu}(\phi)}\right\}$$ will be a better solution for $\gI_h$ in \pref{eq: IH def} than $\{\phi_{h,1}\}$, contradicting that $\gI_h=\{\phi_{h,1}\}$. Thus, in this case, $\gI_h=\left\{\argmax_{\phi \in \Phistar_h}\frac{\agg{d}(\phi)}{\agg{\nu}(\phi)}\right\}$.  Then choosing $\gJ_h=\emptyset$ satisfies all conditions in the lemma.  
    
    If $N_h\geq 2$, we define $\gJ_h = \gI_h \setminus \{\phi_{h,N_h}\} = \{\phi_{h,1}, \ldots, \phi_{h,N_h-1}\}$. Below we verify that it satisfies the two inequalities in the lemma. 

    First, we prove $\sum_{\phi\in \gJ_h} \agg{d}(\phi) < \epsilon$ by contradiction. Suppose that $\sum_{\phi\in \gJ_h} \agg{d}(\phi) \geq \epsilon$. By the assumption that $$\frac{\agg{d}(\phi_{h,1})}{\agg{\nu}(\phi_{h,1})}  \geq \cdots \geq \frac{\agg{d}(\phi_{h,N_h})}{\agg{\nu}(\phi_{h,N_h})},$$ we have $$\frac{\sum_{\phi\in\gJ_h} \agg{d}(\phi)}{\sum_{\phi\in\gJ_h} \agg{\nu}(\phi)} \geq \frac{\sum_{\phi\in\gI_h} \agg{d}(\phi)}{\sum_{\phi\in\gI_h} \agg{\nu}(\phi)},$$ and thus $\gJ_h$ is also a solution of \pref{eq: IH def}. However, $|\gJ_h| < |\gI_h|$, contradicting the assumption that $|\gI_h|$ is the smallest. 

    Next, we prove $\max_h \max_{\phi \in \Phistar_h\setminus \gJ_h}\frac{\agg{d}(\phi)}{\agg{\nu}(\phi)}\leq \aggCstar_\epsilon$. Define $$\phi_h' := \argmax_{\phi\in\Phistar_h\setminus \gJ_h} \frac{\agg{d}(\phi)}{\agg{\nu}(\phi)}.$$ If $\phi_h'=\phi_{h,N_h}$, then we have $$\frac{\agg{d}(\phi_h')}{\agg{\nu}(\phi_h')}=\frac{\agg{d}(\phi_{h,N_h})}{\agg{\nu}(\phi_{h,N_h})}\leq \frac{\sum_{\phi\in\gI_h}\agg{d}(\phi)}{\sum_{\phi\in\gI_h}\agg{\nu}(\phi)} \leq \aggCstar_\epsilon.$$ If $\phi_h'\neq \phi_{h,N_h}$ and $\frac{\agg{d}(\phi_h')}{\agg{\nu}(\phi_h')} > \aggCstar_\epsilon$, then we have $$\frac{\sum_{\phi\in\gI_h} \agg{d}(\phi) + \agg{d}(\phi_h')}{\sum_{\phi\in\gI_h} \agg{\nu}(\phi) + \agg{\mu}(\phi_h')} >  \frac{\sum_{\phi\in\gI_h} \agg{d}(\phi)}{\sum_{\phi\in\gI_h} \agg{\nu}(\phi)}$$ because $\frac{\sum_{\phi\in\gI_h} \agg{d}(\phi)}{\sum_{\phi\in\gI_h} \agg{\nu}(\phi)}\leq \aggCstar_\epsilon$. This implies that $\gI_h\cup \{\phi_h'\}$ is a better solution than $\gI_h$ in \pref{eq: IH def}, which is a contradiction. This concludes the proof.   
\end{proofof}

\subsection{Proof of \pref{thm:main_BVFT_upper_bound}} \label{app: BVFT proof of main}

\begin{lemma}\label{lem: upper bound main lemma} Let $\aggCstar=\aggCstar_{\epsilon^2/H^2}$. With probability $\geq 1-\delta$, 
\begin{align*}
   \E_{\sx\sim \rho}[|\fout(\sx) - \fstar(\sx)|] \leq \cO\left(H\eappr\sqrt{\aggCstar} + H\sqrt{\frac{\aggCstar\Phimax \log(n\Phimax|\gF|/\delta)}{n}} + \epsilon \right). 
\end{align*} 
\end{lemma}

\begin{proofof}[\pref{lem: upper bound main lemma}]
    In this proof, we denote $\gG=\gG(\fout, \fstar)$. Recall the definitions of $\Phistar$, $\agg{d}$, and $\agg{\nu}$ in \pref{def: aggregated transition}. Notice that $\fout, \fstar, \gT_\gG \fout, \gT_\gG \fstar \in \gG$. For any $g\in\gG$ and $\phi\in\Phistar$,  we use $g(\phi)$ to represent the value of $g(\sx)$ for those $\sx\in\phi$. 
    Using \pref{def: bar norm}, we have 
    \begin{align}
        \| \fout-\fstar\|_{\agg{d}, h} &\leq \| \fstar - \gT_{ \gG} \fstar\|_{\agg{d}, h} + \| \fout - \gT_{ \gG}  \fout\|_{\agg{d}, h} + \|\gT_{ \gG} \fout -\gT_{\gG} \fstar\|_{\agg{d}, h}. 
 \label{eq: three main term BVFT} 
    \end{align}
    Before bounding the three terms above, we notice that by \pref{lem: conversion lemma}, there exist $\{\gJ_h\}_{h\in[H]}$ such that 
    \begin{itemize}
        \item $\gJ_h \subset \Phistar_h$ 
        \item $\sum_{\phi\in \gJ_h} \agg{d}(\phi) < \frac{\epsilon^2}{H^2} $
        \item $\max_h \max_{\phi\in\Phistar_h\setminus \gJ_h} \frac{\agg{d}(\phi)}{\agg{\nu}(\phi)}\leq \aggCstar_{\epsilon^2/H^2}$
    \end{itemize}

    Below we denote $\aggCstar=\aggCstar_{\epsilon^2/H^2}$. 
    Now we bound the three terms in \pref{eq: three main term BVFT}. The square of the first term can be upper bounded as 
    \begin{align*}
       \|\fstar - \gT_{\gG} \fstar  \|_{\agg{d}, h}^2 
       &=  \sum_{\phi\in\Phistar_h} \agg{d}(\phi)\left(\fstar(\phi) - \gT_\gG \fstar (\phi)\right)^2 \\
       &\leq \sum_{\phi\in\Phistar_h\setminus \gJ_h} \agg{d}(\phi)\left(\fstar(\phi) - \gT_\gG \fstar (\phi)\right)^2 + \frac{\epsilon^2}{H^2}  \tag{$\sum_{\phi\in\gJ_h}\agg{d}(\phi)\leq \frac{\epsilon^2}{H^2}$} \\
       &\leq \aggCstar\sum_{\phi\in\Phistar_h\setminus \gJ_h} \agg{\nu}(\phi)\left(\fstar(\phi) - \gT_\gG \fstar (\phi)\right)^2 + \frac{\epsilon^2}{H^2}  \tag{by the definition of $\aggC$}  \\
       &\leq \aggCstar\|\fstar - \gT_{\gG} \fstar  \|_{\agg{\nu}, h}^2 + \frac{\epsilon^2}{H^2} \\
       &\leq \aggCstar\|\fstar - \gT_{\gG} \fstar  \|^2_{\nu, h}   + \frac{\epsilon^2}{H^2}  \tag{because $\fstar, \gT_\gG \fstar\in \gG$}\\
       &\leq \aggC\|\fstar - \gT \fstar  \|_{\nu, h}^2 + \frac{\epsilon^2}{H^2} 
       \tag{by Lemma~\ref{lem: property of G}}\\ 
       &\leq \aggCstar\eappr^2 + \frac{\epsilon^2}{H^2}. \tag{by the definition of $\eappr$}
    \end{align*} 
    The square of the second term can be upper bounded as 
    \begin{align*}
        \| \fout - \gT_{ \gG}  \fout\|^2_{\agg{d}, h} 
        &= \sum_{\phi\in\Phistar_h}  \agg{d}(\phi) \left( \fout(\phi) - \gT_{ \gG}  \fout(\phi)\right)^2 \\
        &= \sum_{\phi\in\Phistar_h \setminus \gJ_h}  \agg{d}(\phi) \left( \fout(\phi) - \gT_{ \gG}  \fout(\phi)\right)^2 + \frac{\epsilon^2}{H^2}
        \\ 
        &\leq \aggCstar\sum_{\phi\in\Phistar_h \setminus \gJ_h}  \agg{\nu}(\phi) \left( \fout(\phi) - \gT_{ \gG}  \fout(\phi)\right)^2 + \frac{\epsilon^2}{H^2}\tag{by the definition of $\aggCstar$} \\
        &\leq \aggCstar \|\fout - \gT_{\gG} \fout\|_{\agg{\nu},h}^2 + \frac{\epsilon^2}{H^2} \\
        &= \aggCstar\|\fout - \gT_{\gG} \fout\|_{\nu,h}^2 + \frac{\epsilon^2}{H^2} \tag{because $\fout, \gT_\gG\fout\in\gG$} \\
        &\leq \cO\left(\aggCstar\eappr^2 + \aggCstar\estat^2 \right) + \frac{\epsilon^2}{H^2}. \tag{by \pref{lem: BVFT term 1}-(a)} 
    \end{align*}
    
    Next, again using \pref{lem: BVFT term 1}-(b), the last term can be upper bounded as 
    \begin{align*}
        \|\gT_{ \gG} \fout -\gT_{\gG} \fstar\|_{\agg{d}, h} \leq \|\fout - \fstar\|_{\agg{d}, h+1}. 
    \end{align*}

    Combining all above, we get 
    \begin{align*}
        \| \fout-\fstar\|_{\agg{d}, h} \leq \| \fout-\fstar\|_{\agg{d}, h+1} + \cO\left(\sqrt{\aggCstar}\eappr + \sqrt{\aggCstar}\estat  + \frac{\epsilon}{H}\right), 
    \end{align*}
    which gives $\|\fout - \fstar\|_{\agg{d}, 1}\leq \cO\left(H\sqrt{\aggCstar}\eappr + H\sqrt{\aggCstar}\estat + \epsilon\right)$ after recursively applying the inequality and using Cauchy-Schwarz inequality. The desired inequality follows by noticing that $\E_{\sx\sim \rho}[|\fout(\sx) - \fstar(\sx)|]\leq \|\fout-\fstar\|_{\agg{d}, 1}$ by Cauchy-Schwarz.  
\end{proofof}

\begin{proofof}[\pref{thm:BVFT_upper_bound} (\pref{thm:main_BVFT_upper_bound} in the main body)] 
    In the \emph{fully realizable} setting (i.e., $V^\pieval\in\gF$), in order to control the magnitude of $\Phimax$, we discretize the function set $\gF$, and make it an \emph{approximate realizable} case. 

    For each $f\in\gF$, we round the value of $f(s)$ to the nearest multiple of $\frac{\epsilon}{H\sqrt{\aggC}}$. This way, we have $\eappr=\frac{\epsilon}{H\sqrt{\aggC}}$ and $\Phimax = \order\left(\left(\frac{H}{\eappr}\right)^2\right) = \order\left(\frac{H^4\aggC}{\epsilon^2}\right)$. Thus, by \pref{lem: upper bound main lemma}, 
    \begin{align*}
        \|\fout - \fstar\|_{\agg{d}, 1} &\leq \cO\left(H\sqrt{\aggCstar}\eappr + H\sqrt{\frac{\aggCstar\Phimax \log(n|\gF|\Phimax/\delta)}{n}} + \epsilon\right) \\
        &\leq \cO\left(H \sqrt{\frac{\aggC^2 H^4  \log(n|\gF|H\aggC/\epsilon\delta)}{\epsilon^2 n}} + \epsilon\right).
    \end{align*}
    In order to make the last expression to be $\cO(\epsilon)$, we need 
    \begin{align*}
         n \geq \widetilde{\Omega}\left( 
 \frac{\aggC^2 H^6 \log(|\gF|/\delta)}{\epsilon^4} \right). 
    \end{align*} 
\end{proofof}

\subsection{Implications of the BVFT Upper Bound}\label{app: Cpush}

\begin{lemma}\label{lem: C and Cpush}
    For any $\epsilon$, $\aggC_\epsilon\leq \Cpush$, where $\Cpush$ is defined in \pref{def:strong-coverability}. 
\end{lemma}
\begin{proofof}[\pref{lem: C and Cpush}]
    By the definition of $\aggC_\epsilon$, 
    \begin{align*}
        \aggC_\epsilon &\leq \max_h \max_{\phi\in\Phi_h} \frac{\agg{d}(\phi)}{\agg{\nu}(\phi)} \\
        &= \max_h \max_{\phi\in \Phi_h} \frac{\sum_{\phi'\in\Phi_{h-1}}   \agg{d}(\phi') \agg{T}(\phi\mid \phi', \pieval) }{\sum_{\sx\in\phi} \mu(\sx, \pieval(\sx))}  \tag{by the definitions of $\agg{d}$ and $\agg{\nu}$ and $\nu$}  \\
        &= \max_h \max_{\phi\in \Phi_h} \frac{\sum_{\phi'\in\Phi_{h-1}}  \agg{d}(\phi') \agg{T}(\phi\mid \phi', \pieval) }{\sum_{\sx\in\phi}\mu(\sx)  \mu(\pieval(\sx)\mid\sx)}   \tag{by the definition of $\mu(\sa\mid \sx)$} \\
        &\leq \max_h \max_{\phi\in \Phi_h} \frac{\sum_{\phi'\in\Phi_{h-1}}  \agg{d}(\phi') \agg{T}(\phi\mid \phi', \pieval) }{\sum_{\sx\in\phi}\mu(\sx)} \cdot \CpushA \\
        &\leq \max_h \max_{\phi\in \Phi_h} \frac{ 1}{\sum_{\sx\in\phi} \mu(\sx)} \sum_{\phi'\in\Phi_{h-1}} \agg{d}(\phi') \frac{\sum_{\sx'\in\phi'} \nu(\sx')\sum_{\sx\in\phi} T(\sx\mid \sx', \pieval) }{\sum_{\sx'\in\phi'} \nu(\sx') } \cdot \CpushA \tag{by the definition of $\agg{T}$} \\
        &\leq \max_h \max_{\phi\in \Phi_h} \max_{\sx'\in\gX_{h-1}} \frac{\sum_{\sx\in \phi} T(\sx\mid \sx', \pieval)}{\sum_{\sx\in\phi}\mu(\sx)} \cdot \CpushA \\
        &\leq \max_h \max_{\sx\in \gX_h} \max_{\sx'\in\gX_{h-1}} \frac{T(\sx\mid \sx', \pieval)}{\mu(\sx)} \cdot \CpushA \\
        &\leq \CpushX \cdot \CpushA \\
        &= \Cpush. 
    \end{align*}
\end{proofof}

\begin{lemma}\label{lem: C and BH}
    Let the offline data distribution $\mu$ be admissible, i.e., $\mu(\sx, \sa) = d^{\pioff}(\sx)\pioff(\sa\mid \sx)$ for some $\pioff$\,. Suppose that $\frac{1}{\mu(\pieval(\sx)\mid \sx)}=\frac{1}{\pioff(\pieval(\sx)\mid \sx)}\leq \CpushA$ for all $\sx\in \gX$. Then for any $\epsilon$, $\aggC_\epsilon\leq (\CpushA)^H$. 
\end{lemma}

\begin{proofof}[\pref{lem: C and BH}]
    For a fixed $h$, we have  
    \begin{align*}
        &\max_{\phi\in\Phi_h} \frac{\agg{d}(\phi)}{\agg{\nu}(\phi)} \\
        &= \max_{\phi\in \Phi_h} \frac{\sum_{\phi'\in\Phi_{h-1}}   \agg{d}(\phi') \agg{T}(\phi\mid \phi', \pieval) }{\sum_{\sx\in\phi} \mu(\sx, \pieval(\sx))}  \tag{by the definitions of $\agg{d}$ and $\agg{\nu}$ and $\nu$}  \\
        &=  \max_{\phi\in \Phi_h} \frac{\sum_{\phi'\in\Phi_{h-1}}  \agg{d}(\phi') \agg{T}(\phi\mid \phi', \pieval) }{\sum_{\sx\in\phi}\mu(\sx)  \mu(\pieval(\sx)\mid\sx)}   \tag{by the definition of $\mu(\sa\mid \sx)$} \\
        &= \max_{\phi\in \Phi_h} \frac{\sum_{\phi'\in\Phi_{h-1}}  \agg{d}(\phi') \agg{T}(\phi\mid \phi', \pieval) }{\sum_{\sx\in\phi}\sum_{\sx'\in\gX_{h-1}}\sum_{\sa'\in \gA} \mu(\sx', \sa') T(\sx\mid \sx', \sa') \mu(\pieval(\sx)\mid\sx) } \tag{by the fact that $\mu$ is an occupancy measure} \\
        &\leq  \max_{\phi\in \Phi_h} \frac{\sum_{\phi'\in\Phi_{h-1}}  \agg{d}(\phi') \agg{T}(\phi\mid \phi', \pieval) }{\sum_{\sx\in\phi}\sum_{\sx'\in\gX_{h-1}} \mu(\sx', \pieval(\sx')) T(\sx\mid \sx', \pieval(\sx')) } \cdot \CpushA \\
        &= \max_{\phi\in \Phi_h} \frac{\sum_{\phi'\in\Phi_{h-1}}  \agg{d}(\phi') \agg{T}(\phi\mid \phi', \pieval) }{\sum_{\phi'\in\Phi_{h-1}}\sum_{\sx\in\phi}\sum_{\sx'\in\phi'} \mu(\sx', \pieval(\sx')) T(\sx\mid \sx', \pieval(\sx')) } \cdot \CpushA \\
        &=  \max_{\phi\in \Phi_h} \frac{\sum_{\phi'\in\Phi_{h-1}}  \agg{d}(\phi')  \agg{T}(\phi\mid \phi', \pieval) }{\sum_{\phi'\in\Phi_{h-1}} \left( \sum_{\sx'\in\phi'} \mu(\sx', \pieval(\sx')) \right)\agg{T}(\phi\mid \phi', \pieval)  } \cdot \CpushA   \tag{by the definition of $\agg{T}$} \\
        &=  \max_{\phi\in \Phi_h} \frac{\sum_{\phi'\in\Phi_{h-1}}  \agg{d}(\phi')  \agg{T}(\phi\mid \phi', \pieval) }{\sum_{\phi'\in\Phi_{h-1}}  \agg{\nu}(\phi') \agg{T}(\phi\mid \phi', \pieval)  } \cdot \CpushA  \tag{by the definition of $\agg{\nu}$}  \\
        &\leq \max_{\phi'\in\Phi_{h-1}} \frac{\agg{d}(\phi')}{\agg{\nu}(\phi')} \cdot \CpushA. 
    \end{align*}
    Recursively applying this we get $\max_h \max_{\phi\in\Phi_h} \frac{\agg{d}(\phi)}{\agg{\nu}(\phi)}\leq (\CpushA)^H$. Then by just noticing that $\aggC_\epsilon \leq \max_h \max_{\phi\in\Phi_h} \frac{\agg{d}(\phi)}{\agg{\nu}(\phi)}$ finishes the proof. 
\end{proofof} 

\clearpage

\section{Role of Realizable Value Function Class in Offline RL} \label{app:realizable_function_class}
So far, we have considered offline policy evaluation problems where in addition to an offline data distribution \(\mu\) (that satisfies concentrability), the learner is also given a function class \(\cF\) that contains \(Q^{\pieval}\)---the state-action value function corresponding to the policy \(\pieval\) that the learner wishes to evaluate. To understand the role of value function class in offline RL, in this section, we ask: 

\begin{center}
\begin{minipage}{0.8\textwidth} 
\begin{center}
\textit{Is statistically efficient offline policy evaluation feasible without access to a realizable value function class?} 
\end{center} 
\end{minipage}
\end{center}

We answer this question negatively for both admissible and trajectory data. Our first result in \pref{thm:realizable-H} (below) shows that given only admissible offline data, offline policy evaluation is intractable without a realizable value function class, even when we have bounded pushforward concentrability coefficient. 


\begin{theorem}\label{thm:realizable-H} 
For any positive integer $N$, there exists a class $\mathcal{M}$ of \MDP{}s with shared state space $\calX_N$, action space $\calA = \crl{\aone, \atwo}$ and horizon $H = 3$, a deterministic evaluation policy \(\pi_e\), and an exploration policy \(\piexp\) with $\Pr\prn*{\pieval(\sx) = \pioff(\sx)} \geq \nicefrac{1}{2}$ for all $\sx\in\calX$ such that any algorithm that estimate the value \(V^{\pi_e}(\rho, M)\) up to error $\nicefrac{1}{2}$ for all \MDP{}s \(M \in \cM\) must use $\Omega(N)$ many admissible samples in some \MDP in \(\cM\). 
\end{theorem}

\pref{thm:realizable-H} suggests the intractability of offline policy evaluation without a realizable function class since the result holds for any positive integer \(N\).

\begin{remark}
    Note that since $H = 3$ in the construction of \pref{thm:realizable-H}, the property $\Pr\prn*{\pieval(\sx) = \pioff(\sx)} \geq \nicefrac{1}{2}$ indicates that the pushforward concentrability coefficient  $\Cpush\le 8$ w.r.t.~the  admissible distribution $\mu_h(\sx, \sa; M) = d^\piexp(\sx, \sa; M)$  (see \pref{def:strong-coverability}). Thus, using \pref{lem: C and Cpush}, the aggregated concentrability coefficient \(\aggC \leq 8\)  for any aggregation scheme on the underlying MDPs. 
\end{remark} 

On the other hand, under access to a realizable state-action value function class  (\pref{ass: approximate realizability}), BVFT algorithm (\pref{thm:BVFT_upper_bound} + \pref{lem: C and BH})  obtains a sample complexity upper bound of \(O(\poly(2^H, \log(\abs{\cF})))\) which is tractable for \(H = 3\). This highlights the role of a realizable value function class in offline policy evaluation with admissible data. Our next result extends this to trajectory offline data. 


\begin{theorem}\label{thm:realizable-strong}
    There exists a class $\mathcal{M}$ of \MDP{}s with shared state space $\calX$, action space $\calA = \crl{\aone, \atwo}$, and horizon $H$, a deterministic evaluation policy \(\pi_e\) and an exploration policy \(\piexp\) such that the pushforward concentrability coefficient  $\Cpush\le 4$ for the offline distribution  $\mu_h(\sx, \sa; M) = d^{\piexp}_h(\sx, \sa; M)$  (see \pref{def:strong-coverability}). 

    Furthermore, any algorithm that estimate the value \(V^{\pi_e}(\rho, M)\)  up to precision $\nicefrac{1}{2}$ for all \MDP{s}  $M\in\mathcal{M}$ must use  $\Omega(2^H)$ many offline trajectories in some \MDP in \(\cM\). 
\end{theorem}

The above shows that agnostic offline policy evaluation is not statistically tractable even when given trajectory offline data. On the other hand, recall that under access to a realizable state-action value function class (\pref{ass: approximate realizability}) and bounded pushforward concentrability coefficient, the  BVFT algorithm in \citet{xie2021batch} enjoys a \(\poly(\Cpush, H, \log(\abs{\cF}), \nicefrac{1}{\epsilon})\) sample complexity (even without access to trajectory data). 

\begin{remark} For the lower bound \MDP construction in \pref{thm:realizable-strong}, recall that \(\pioff(\sx) = \mathrm{Uniform}(\cA)\) for any \(\sx \in \cX\). Thus, given trajectory data, the classical importance sampling algorithm from \citet{kearns1999approximate, agarwal2019reinforcement} can evaluate the value of \(\pieval\) upto precision \(\epsilon\) after collecting $\mathcal{O}\prn*{\frac{2^H}{\epsilon^2}}$ many offline trajectories from \(\pioff\).  
\end{remark}


\subsection{Proof of Lower Bounds}  

\begin{figure}[h!]
    \centering
    \includegraphics[width=0.9\textwidth]{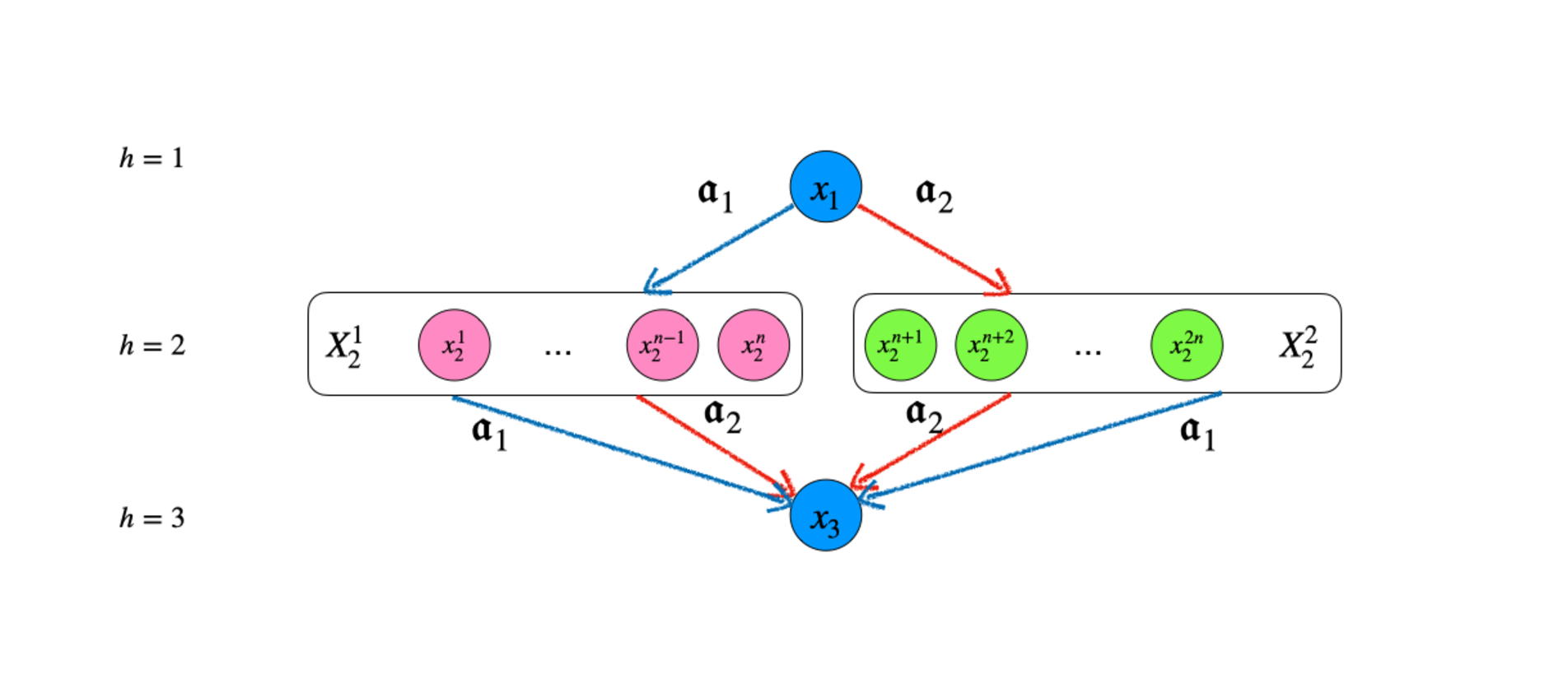} 
    \caption{Lower bound construction in  \pref{thm:realizable-H}. The {blue} arrows represent the transitions under the action $\aone$, and the {red} arrows represent the transitions under the action $\atwo$. In the middle layer, the arrows to the blocks \(\cX_2^1\) and \(\cX_2^2\) denote uniform transitions to the states within those blocks.} 
    \label{fig: 3layer}
\end{figure}

To avoid redundancy, we only provide an informal construction of the lower bound here. A formal lower bound construction can be obtained by following arguments similar to that in \pref{app:proof1}. 

\begin{proofof}[\pref{thm:realizable-H}] Let \(N\)  be a positive integer, and consider a state space \(\cX = \cX_1 \cup \cX_2\cup \cX_3\) where \(\cX_1 = \crl{\sx_1}\), $\cX_3 = \{\sx_3\}$ and \(\cX_2\) is of size \(4N^2\). Consider a partition \(\phi\) of \(\cX_2\) that divides it into two parts \(\cX_2^1\) and \(\cX_2^2\), each of size \(2N^2\). For any such partition \(\phi\), we define two MDPs \(M\ind{1}\) and \(M\ind{2}\), where the \(M_\phi\ind{i} = MDP(\cX, \cA, H, T\ind{i}, r\ind{i}, \rho)\) is defined such that (this construction can be viewed in \pref{fig: 3layer}):   
\begin{enumerate}[label=\(\bullet\)]
    \item Horizon \(H=3\), action space \(\cA = \crl{\aone, \atwo}\), and initial distribution \(\rho = \delta_{\sx_1}(\cdot)\). 
    \item Transition dynamics \(T\ind{i}\) is defined such that $T\ind{i}(\cdot \mid \sx_2) = \delta_{\sx_3}(\cdot)$ for any $\sx_2\in\calX_2$, and
    \begin{align*}
T\ind{i}(\cdot \mid \sx_1, \sa)  = \begin{cases}
    \mathrm{Uniform}(\cX_2^1) &\text{if} \quad \sa= \aone,\\
    \mathrm{Uniform}(\cX_2^2) &\text{if} \quad \sa= \atwo.
\end{cases}
\end{align*}

\item Reward function is defined for any $\sa\in \{\aone, \atwo\}$,
\begin{align*}
r\ind{i}(\sx, \sa) = \begin{cases} 
    0 &\text{if} \quad \sx = \sx_1,\\ 
    1 &\text{if} \quad \sx \in \cX\ind{i}_2,\\
    0 &\text{otherwise}.
\end{cases}
\end{align*}
\end{enumerate}

We thus define the class \(\cM\) as 
\begin{align*}
\cM = \bigcup_{\phi \in \Phi} (M\ind{1}_\phi, M\ind{2}_\phi),  
\end{align*}
where \(\Phi\) denotes the set of all feasible partitions for \(\cX_2\) into \(\cX_2^1\) and \(\cX_2^2\) of the same size, and satisfies \(\abs{\Phi} = 2^{O(N\log(N))}\). 

We further define the evaluation policy \(\pieval\) and \(\piexp\) such that for all \(\sx \in \cX\), 
\begin{align*}
\pieval(\sx) = \delta_{\aone}(\cdot), \qquad \text{and} \qquad \piexp(\sx) = \mathrm{Uniform}(\crl{\aone, \atwo}). 
\end{align*} 
Then for any $M\in\cup_{\phi\in\Phi}\mathcal{M}\ind{1}$, we have $V^\pieval(\rho; M) = 1$ and for any $M\in\cup_{\phi\in\Phi}\mathcal{M}\ind{2}$, we have $V^\pieval(\rho; M) = 0$. Hence if the algorithm cannot tell whether $M\in\cup_{\phi\in\Phi}\mathcal{M}\ind{2}$ or $M\in\cup_{\phi\in\Phi}\mathcal{M}\ind{2}$, then the algorithm must fail to output $\nicefrac{1}{2}$-accurate estimation in at least one case among $\mathcal{M}$.

The key intuition behind the proof is that given some dataset \(\cD_h = \crl{(\sx_h, \sa_h, r_h, \sx_{h+1})}\), the learner can not identify which states belong to \(\cX_2^1\) vs \(\cX_2^2\) in the second layer. Since, the reward model depends on whether \(\sx \in \cX_2^1\) or \(\sx \in \cX_2^2\), inability to identify the partition \(\phi\) which split \(\cX\) into \(\cX_2^1\) and \(\cX_2^2\), will lead to an error in evaluating \(V^{\pieval}(\rho; M)\) with probability \(\nicefrac{1}{2}\) since we consider \(\Phi\) to be the set of all possible partitions of equal size. 

To see the above, note that for any $M\in\mathcal{M}$, the marginal occupancy measure at layer $2$ is 
    $$d_2^\piexp(\cdot; M) = \mathrm{Uniform}(\cX_2 \times \{\aone, \atwo\}).$$
    Hence, samples $(\sx_2, \sa_2, r_2, \sx_3)$ where $\sx_2 \in \cX_2$ cannot provide useful any information unless we know whether $\sx_2 \in \cX_2^1$ or \(\cX_2^2\). However, while collecting admissible samples of the second layer, the distribution we samples from is $d^\piexp_2(\cdot; M)$, i.e. $\mathrm{Uniform}(\calX_2)$, which reveals no information on $\calX_2^1$ and $\calX_2^2$ unless some state $\sx_2$ appears both in samples $(\sx_1, a_1, r_1, \sx_2)$ and $(\sx_2, a_2, r_2)$. According to our choice of $N$, this happens with probability at most 
    $$1 - \prod_{i=1}^{2N}\left(1 - \frac{2N}{4N^2}\right)^{N}\le \frac{1}{2}.$$
    Hence any algorithm must fail to output $\nicefrac{1}{2}$-accurate estimation of $V^\pieval(\rho; M)$ in at least one $M\in\mathcal{M}$ with probability at least $\nicefrac{1}{2}$.
\end{proofof} 


\begin{figure}[h!]
    \centering
    \includegraphics[width=0.85\textwidth]{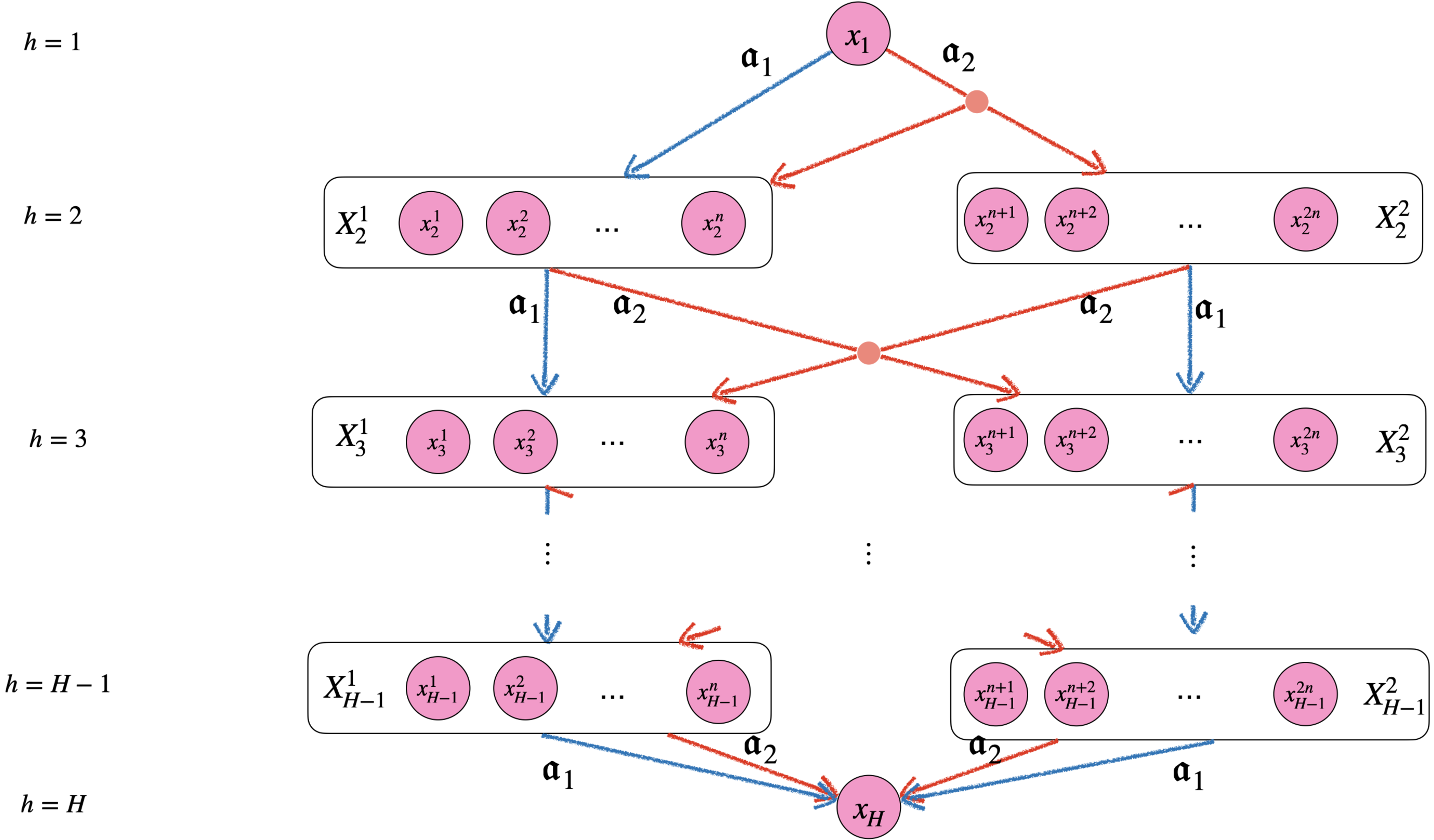} 
    \caption{Lower bound construction in  \pref{thm:realizable-strong}. The {blue} arrows represent the transitions under the action $\aone$, and the {red} arrows represent the transitions under the action $\atwo$. In layers for \(h = 2\) to \(H-1\), the arrows to the blocks  \(\cX_h^1\) and \(\cX_h^2\) 
 denote uniform transitions to the states within those blocks.} 
    \label{fig: Hlayer}
\end{figure} 

\begin{proofof}[\pref{thm:realizable-strong}] 
    Let $N = 2^H$ and consider a state space \(\cX = \cup_{h=1}^H \cX_h\) where \(\cX_1 = \crl{\sx_1}\) and \(\cX_2, \cdots, \cX_H\) are of size \(4N^2\) each. Consider $\phi$ to be a partition of \(\cX_2, \cX_3, \cdots, \cX_H\) that divides each $\cX_h$ ($2\le h\le H$) into two parts \(\cX_h^1\) and \(\cX_h^2\), each of size \(2N^2\). For any such \(\phi\), we define two MDPs \(M\ind{1}\) and \(M\ind{2}\), where the \(M_\phi\ind{i} = MDP(\cX, \cA, H, T\ind{i}, r\ind{i}, \rho)\) is defined such that (this construction can be viewed in \pref{fig: Hlayer}):   
    \begin{enumerate}[label=\(\bullet\)]
        \item Horizon \(H\), action space \(\cA = \crl{\aone, \atwo}\), and initial distribution \(\rho = \delta(\sx_1)\). 
        \item Transition dynamics \(T\ind{i}: \cX_h \times \cA \mapsto \cX_{h+1}\) is defined such that 
        \begin{align*}
            T\ind{i}(\cdot \mid \sx_1, \sa)  = \begin{cases}
                \mathrm{Uniform}(\cX_2^1) &\text{if} \quad \sa= \aone,\\
                \mathrm{Uniform}(\cX_2^1\cup\cX_2^2) &\text{if} \quad \sa= \atwo,
            \end{cases}
        \end{align*}
        and for $\sx\in \calX_h^1$ with $h\in [H]$,
        \begin{align*}
            T\ind{i}(\cdot \mid \sx, \sa)  = \begin{cases}
                \mathrm{Uniform}(\cX_2^1) &\text{if} \quad \sa= \aone,\\
                \mathrm{Uniform}(\cX_2^1\cup\cX_2^2) &\text{if} \quad \sa= \atwo,
            \end{cases} 
        \end{align*}
        for $\sx\in \calX_h^2$ with $h\in [H]$,
        \begin{align*}
            T\ind{i}(\cdot \mid \sx, \sa)  = \begin{cases}
                \mathrm{Uniform}(\cX_2^2) &\text{if} \quad \sa= \aone,\\
                \mathrm{Uniform}(\cX_2^1\cup\cX_2^2) &\text{if} \quad \sa= \atwo.
            \end{cases}
        \end{align*}
    
    \item Reward function: for any $\sa\in\{\aone, \atwo\}$, 
    \begin{align*}
    r\ind{i}(\sx, a) = \begin{cases} 
        1 &\text{if} \quad \sx \in \cX\ind{i}_H,\\ 
        0 &\text{otherwise}.
    \end{cases}
    \end{align*}
    \end{enumerate}
    We thus define the class \(\cM\) as 
    \begin{align*}
    \cM = \bigcup_{\phi \in \Phi} (M\ind{1}_\phi, M\ind{2}_\phi),  
    \end{align*}
    where \(\Phi\) denotes the set of all feasible partitions for \(\cX_2, \cdots, \cX_H\) into \(\cX_h^1\) and \(\cX_h^2\) of the same size. We further define the evaluation policy \(\pieval\) and \(\piexp\) such that for all \(\sx \in \cX\), 
    \begin{align*}
    \pieval(\sx) = \delta_{\aone}(\cdot), \qquad \text{and} \qquad \piexp(\sx) = \mathrm{Uniform}(\crl{\aone, \atwo}). 
    \end{align*} 

    For any \MDP $M\in\mathcal{M}$, we observe that the occupancy measure $d_h^\piexp(\cdot; M)$ is
    $$d_h^\piexp(\cdot; M) = \mathrm{Uniform}(\calX_h\times \mathcal{A}).$$
    Hence it is easy to verify that the strong coverability coefficient? of any instances in $\mathcal{M}$ is upper bounded by $4$. Additionally, we also have for any \MDP $M\in\bigcup_{\phi\in\Phi}M_\phi\ind{1}$, $V^\pieval(\rho; M) = 1$ and for any $M\in\bigcup_{\phi\in\Phi}M_\phi\ind{1}$, $V^\pieval(\rho; M) = 0$. 
        
    The only way to tell whether a case $M\in\bigcup_{\phi\in\Phi}M_\phi\ind{1}$ or $M\in\bigcup_{\phi\in\Phi}M_\phi\ind{2}$ is through the reward function in the last layer. However, if action $\atwo$ is taken in any step within the whole trajectory, the last layer distribution will be $\mathrm{Uniform}(\calX_H)$, which will induce the same reward distribution no matter the \MDP is $M_\phi\ind{1}$ or $M_\phi\ind{2}$. 
    
    Hence as long as none of the trajectory collected takes only action $1$ among the trajectory, the learner will fail to output $\nicefrac{1}{2}$-accurate estimation in at least one \MDP in $\mathcal{M}$ with probability at least $\nicefrac{1}{2}$. And using $o(2^H)$ trajectories, the learner will fail to output $\nicefrac{1}{2}$-accurate estimation in at least one \MDP in $\mathcal{M}$ with probability at least $\nicefrac{1}{2} - o(2^H)\cdot \nicefrac{1}{2^H}\ge \nicefrac{1}{4}$.
\end{proofof}

\end{document}